%% file: max-margin-icml.tex
\documentclass{article}

\usepackage{microtype}
\usepackage{graphicx}
\usepackage{booktabs} 
\usepackage{amsmath,amssymb}
\usepackage{bbm}
\usepackage{multirow}
\usepackage{subcaption}
\usepackage{amsthm} 
\newtheorem{theorem}{Theorem}
\newtheorem{corollary}{Corollary}
\newtheorem{lemma}{Lemma}
\newtheorem{definition}{Definition}
\newtheorem{assumption}{Assumption}
\newtheorem{proposition}{Proposition}
\newtheorem{remark}{Remark}
\newcommand{\bv}{\boldsymbol{v}}
\newcommand{\bw}{\boldsymbol{w}}
\newcommand{\bx}{\boldsymbol{x}}
\newcommand{\mL}{\mathcal{L}}
\newcommand{\bbeta}{\boldsymbol{\beta}}
\renewcommand{\labelenumii}{\Roman{enumii}}

\usepackage{hyperref}

\usepackage[accepted]{icml2021}

\icmltitlerunning{The Implicit Regularization for Adaptive Optimization Algorithms on Homogeneous Neural Networks}

\begin{document}

\twocolumn[
\icmltitle{The Implicit Regularization for Adaptive Optimization Algorithms on Homogeneous Neural Networks}



\icmlsetsymbol{equal}{*}

\begin{icmlauthorlist}
\icmlauthor{Bohan Wang}{MSRA}
\icmlauthor{Qi Meng}{MSRA}
\icmlauthor{Wei Chen}{MSRA}
\icmlauthor{Tie-Yan Liu}{MSRA}
\end{icmlauthorlist}

\icmlaffiliation{MSRA}{	Microsoft Research Asia,
	Beijing, China}

\icmlcorrespondingauthor{Wei Chen}{wche@microsoft.com}

\icmlkeywords{Adaptive Optimizer, Implicit Regularization, Margin}
\vskip 0.3in
]


\printAffiliationsAndNotice{}  

\begin{abstract}
Despite their overwhelming capacity to overfit, deep neural networks trained by specific optimization algorithms tend to generalize well to unseen data. Recently, researchers explained it by investigating the implicit regularization effect of optimization algorithms. A remarkable progress is the work \cite{lyu2019gradient}, which proves gradient descent (GD) maximizes the margin of homogeneous deep neural networks.  Except GD, adaptive algorithms such as AdaGrad, RMSProp and Adam are popular owing to their rapid training process. 
However, theoretical guarantee for the generalization of adaptive optimization algorithms is still lacking. In this paper, we study the implicit regularization of adaptive optimization algorithms when they are optimizing the logistic loss on homogeneous deep neural networks. We prove that adaptive algorithms that adopt exponential moving average strategy in conditioner (such as Adam and RMSProp)  
can maximize the margin of the neural network,  while  AdaGrad that directly sums historical squared gradients in conditioner can not.  
It indicates superiority on generalization of exponential moving average strategy in the design of the conditioner. Technically, we provide a unified framework to analyze convergent direction of adaptive optimization algorithms by constructing novel \emph{adaptive gradient flow} and   \emph{surrogate margin}.  Our experiments can well support the theoretical findings on convergent direction of adaptive optimization algorithms. 
\end{abstract}

\section{Introduction}
Deep learning techniques have been very successful in several domains, like computer vision \cite{voulodimos2018deep}, speech recognition \cite{deng2013new} and natural language processing \cite{young2018recent}.  In practice, deep neural networks (DNN) learned by optimization algorithms such as gradient descent (GD) and its variants can generalize well to unseen data \cite{witten2016data}.   However, deep neural networks are non-convex. The non-convex deep neural networks have been found to have large amount of global minima \cite{choromanska2015loss}, while only few of them can guarantee satisfactory generalization property \cite{brutzkus2017sgd}. Explaining why the highly non-convex model trained by a specific algorithm can generalize has become an important open question in deep learning.

Regarding the above question, one plausible explanation is that optimization algorithms implicitly regularize the training process \cite{neyshabur2015path}. That is, the optimization algorithm tends to drive parameters to certain kinds of global minima which generalize well, although no explicit regularization is enforced. Recently, exciting results have been shown for vanilla gradient descent. 
 A remarkable progress is the work \cite{lyu2019gradient}, which proves that GD maximizes the margin of homogeneous (non-linear) deep neural networks.

On the other hand, adaptive algorithms such as AdaGrad \cite{duchi2011adaptive}, RMSProp \cite{hinton2012neural}, and Adam \cite{kingma2014adam} have been in spotlight these years. These algorithms are proposed to improve the convergence rate of GD (or SGD) by
using second-order moments of historical gradients as conditioner and have been widely applied in deep learning \cite{ruder2016overview}.  Despite the rapid convergence of adaptive methods, numerous works have provided empirical evidence that adaptive methods may suffer from poor generalization performance \cite{wilson2017marginal, luo2019adaptive}. 
Several works try to improve the performance of adaptive optimization algorithms such as AdamW \cite{loshchilov2017decoupled}, AdaBound \cite{luo2019adaptive}, AdaBelief \cite{zhuang2020adabelief}. However, there is little theoretical analysis for generalization of adaptive algorithms. These observations and the research for GD motivate us to study the implicit regularization for adaptive algorithms.  

The key factor for the success of adaptive optimization algorithms is to design better conditioners of the gradient. Adagrad adopts the simple average of the squared values of the historical gradients in its conditioner, while RMSProp and Adam improve the simple average to exponential moving average strategy. In this paper, we aim to study the influence of different types of conditioners on convergent direction of parameters trained by adaptive optimization algorithms. Specifically, we work on the homogeneous neural networks (including fully connected or convolutional neural network with ReLU or leaky ReLU activations) with separable data under logistic loss (for binary classification) and cross-entropy (for multi-class classification). 
For logistic loss, we focus on characterizing the convergent direction of parameters (i.e., $\lim_{t\rightarrow \infty} \frac{w_t}{\|w_t\|_2}$) with respect to the training iteration $t$, which is a key target along this line of researches \cite{ soudry2018implicit,gunasekar2018implicit,lyu2019gradient}. 

\textbf{Our main result} is summarized in Theorem 1, which states that RMSProp and {Adam (w/m) (a variant of Adam without momentum acceleration)\footnote{How momentum influence the convergence of an optimization algorithm on non-convex deep neural network is still an open problem. Here, we only study a variant of Adam which sets the momentum parameter as $0$.}} maximize margin of the neural network (equivalent to the optimum of optimization problem in Eq.(2)) and AdaGrad does not converge to max-margin solution due to the anisotropic $\boldsymbol{h}_{\infty}$. 
\begin{theorem} (Informal)
We use $\Phi(\boldsymbol{w},\boldsymbol{x})$ to denote the homogeneous neural network model with parameter $\boldsymbol{w}$ and input $\boldsymbol{x}$. (1) For AdaGrad, any limit point of $w_t/\|w_t\|_2$ is a KKT point of the optimization problem 
\begin{align}
\min \|\boldsymbol{h}^{-1/2}_{\infty}\odot \boldsymbol{w}\|^2 \quad \textit{subject to } y_i\Phi(\boldsymbol{w}, \boldsymbol{x}_i)\geq 1, \forall i,
\end{align}  where $\boldsymbol{h}_{\infty}=\lim_{t\rightarrow\infty} \boldsymbol{h}(t)$ is the limit of the conditioner in AdaGrad.
(2) For Adam (w/m) and RMSProp, any limit point of $\boldsymbol{w}(t)/\|\boldsymbol{w}(t)\|_2$ is a KKT point of the optimization problem 
\begin{align}
\min\| \boldsymbol{w}\|^2 \quad \textit{subject to } y_i\Phi(\boldsymbol{w}, \boldsymbol{x}_i)\geq 1, \forall i.
\end{align} 
\end{theorem}
Theorem 1 indicates the importance of proper design on the conditioner, i.e., adaptive algorithms like Adam (w/m) and RMSProp that adopt exponential weighted average design on conditioner regularize the training to max-margin solution, which has low complexity. Therefore, we can expect good generalization performance for Adam (w/m) and RMSProp.  Furthermore, we illustrate that the convergence direction of AdaGrad is sensitive to initialization, which hurts its generalization.

{We establish Theorem 1 for both continuous flows of adaptive optimization algorithms and their discrete update rules.} The \textbf{technical contributions} to prove Theorem 1 are summarized as follows. (1) We propose \emph{adaptive gradient flow}, which is a unified framework to deal with adaptive gradients. With the adaptive gradient flow,  the analysis of convergent direction is transformed from original parameter space to a normalized parameter space.  (2) In the normalized parameter space, we construct \emph{surrogate margin} for the adaptive algorithms, and {with the surrogate margin, we show that the increasing rate of the parameter norm can be bounded by the decreasing rate of logarithmic loss and the loss converges to zero.} 
(3) We prove that any limit direction of the normalized parameter flow is a KKT point of the margin maximization problem in normalized parameter space. Moreover, we prove the convergent direction is unique if the neural network is definable \cite{kurdyka1998gradients}. The adaptive gradient flow and surrogate margin are designed for adaptive optimization algorithms, which makes the proof techniques different from that for vanilla GD in \cite{soudry2018implicit,lyu2019gradient}. {(4) We further prove the convergent direction for discrete update rules by characterizing the influence of the learning rate.}


Finally, we conduct experiments to observe the margin of homogeneous neural network during training of several adaptive optimization algorithms. For all experiments, the margins are increasing during training and the final margins of RMSProp and Adam (w/m) are larger than that of AdaGrad. We also observe the convergent direction of adaptive optimization algorithms under different realizations of initialization and results show that the convergent direction of AdaGrad is sensitive to initialization. These observations can well support our theoretical findings.

\section{Related Work}

\textbf{Implicit Regularization of First-order Optimization Methods.} \citet{soudry2018implicit} proved that gradient descent on linear logistic regression with separable data converges in the direction of the max $L^2$ margin solution of the corresponding hard-margin Support Vector Machine, and motivate a line of works on the implicit regularization of GD on linear model \cite{nacson2019convergence,ji2019implicit,li2019implicit,xu2018will}. 

Afterwards, researchers study the implicit regularization of GD on deep neural networks.  \citet{ji2018gradient,gunasekar2018implicit} studied the deep linear network and \citet{soudry2018implicit} studied the two-layer neural network with ReLU activation. \citet{nacson2019lexicographic} proved the asymptotic direction is along a KKT point of the $L^2$ max-margin problem for homogeneous deep neural networks. \citet{lyu2019gradient} independently proved similar result for homogeneous neural networks with simplified assumptions. Based on \cite{lyu2019gradient}, \citet{ji2020directional} further prove that parameters have only one asymptotic direction.

There are also works considering implicit regularization of other first-order optimization algorithms. \citet{nacson2019stochastic} worked on Stochastic Gradient Descent for linear logistic regression. \citet{gunasekar2018characterizing} studied mirror descent 
and steepest descent on linear model. \citet{arora2019exact} proved gradient descent on Neural Tangent Kernel  will converge 
to a global minimum near the initial point.

{However, there is little result on the implicit regularization of adaptive optimization methods.}

\textbf{Theoretical Evidence of Generalization of Adaptive Algorithms.}  {Adaptive algorithms have been in spotlight these years and many works empirically observe the generalization behavior of adaptive algorithms \cite{keskar2017improving,reddi2018adaptive,chen2018closing,luo2019adaptive}.  In comparison, there are few theoretical justifications.}  \citet{wilson2017marginal}  constructed a specific linear regression task where adaptive optimization algorithms converge to a solution that incorrectly classifies new data with probability arbitrarily close to half. 
\citet{zhou2020towards} modeled the distribution of stochastic noise in Adam, and showed that SGD tends to converge to flatter local minima. 
Another viewpoint is to study \textbf{the convergent direction of adaptive optimization algorithms}. To the best of our knowledge, the only work is \cite{qian2019implicit}, which proves the convergent direction of AdaGrad on linear logistic regression. In this paper, we study the convergent direction of adaptive optimization algorithms on deep neural networks which requires different techniques due to the non-convexity of deep networks.

Meanwhile, the \textbf{correlation between margin and generalization error} has also been extended to deep networks. \citet{bartlett2017spectrally} first bound the generalization error of deep neural networks using (spectrally) normalized margin by covering number. In parallel, \citet{neyshabur2018pac} adopt normalized margin into the PAC-Bayesian framework and derive generalization bound with different dependency on layer width from \cite{bartlett2017spectrally}. Empirically, \citet{jiang2019fantastic} present a large scale study of different generalization bounds in deep networks, and find there is a significant correlation between generalization error and normalized margin when optimizer is changed.  {These work support our study on generalization in deep learning through the margin theory.}


\section{Preliminaries}
In this paper, we study the logistic regression problem with homogeneous neural networks. Let training set $\boldsymbol{S}$ defined as $\boldsymbol{S}=\{(\boldsymbol{x}_i,y_i)\}_{i=1}^N$, where $\boldsymbol{x}_i\in\mathcal{X}$ ($i=1,2,\cdots,N$) are inputs, $y_i\in \mathbb{R}$ ($i=1,2,\cdots,N$) are labels, and $N$ is the size of $\boldsymbol{S}$. The empirical loss $\mathcal{L}$ with training set $\boldsymbol{S}$, neural network classifier $\Phi$, individual loss $\ell(x)=e^{-f(x)}$ and parameters $\boldsymbol{w}\in\mathbb{R}^p$ can be written as follows:
\begin{equation*}
    \mathcal{L}(\boldsymbol{w}, \boldsymbol{S})= \sum_{i=1}^N \ell\left(y_i\Phi(\boldsymbol{w},\boldsymbol{x}_i)\right).
\end{equation*}

In an optimization process, the training set $\boldsymbol{S}$ is fixed. Therefore, without loss of generality, we abbreviate $\mathcal{L}(\boldsymbol{w}, \boldsymbol{S})=\mathcal{L}(\boldsymbol{w})$, and $y_i\Phi(\boldsymbol{w},\boldsymbol{x}_i)=q_i(\boldsymbol{w})$. In this paper, we consider the exponential loss, i.e., $f(q_i(\boldsymbol{w}))=q_i(\boldsymbol{w})$, and the logistic loss, i.e., $f(q_i(\boldsymbol{w}))=-\log\log (1+e^{-q_i(\boldsymbol{w})})$. Both of $f$ are monotonously increasing and have an inverse. 

We will use Clarke's Subdifferential $\partial$ \cite{clarke1975generalized} in this paper as a natural extension of gradient $\nabla$ for locally Lipschitz functions. For any  locally  Lipschitz function $f:\mathbb{R}^p\rightarrow\mathbb{R}$, its Clarke's Subdifferential $\partial f$ at point $\boldsymbol{w}_0$ is defined as \begin{equation*}
    \operatorname{conv}\{\lim _{k \rightarrow \infty} \nabla f\left(\boldsymbol{w}_{k}\right): \boldsymbol{w}_{k} \rightarrow \boldsymbol{w}_0, \nabla f(\boldsymbol{w}_k) \text { exists }\}.
\end{equation*} Following \cite{davis2020stochastic}, we also define $f$  admits a chain rule if for any arc \footnote{A arc $\boldsymbol{z}:\mathbb{R}^{+}\rightarrow\mathbb{R}^p$ satisfies for any compact set $I\subset\mathbb{R}^{+}$, $\boldsymbol{z}$ is absolute continuous on $I$.  } $\boldsymbol{z}: \mathbb{R}^{+}\rightarrow\mathbb{R}^p$, $\forall \boldsymbol{h}\in \partial f(\boldsymbol{z}(t))$, $\frac{\mathrm{d}f(z(t))}{\mathrm{d}t}=\langle \boldsymbol{h}, \frac{\mathrm{d} z}{\mathrm{d}t}\rangle$, a.e. for $t>0$.
\subsection{Continuous Flow for Adaptive Algorithms}
Adaptive optimization algorithms including AdaGrad, RMSProp, Adam are widely used to optimize the loss function in deep learning. The update rules for these adaptive optimization algorithms can be written as \footnote{In this paper, we only consider no-momentum versions of the algorithms, i.e., the algorithms without momentum acceleration. }
\begin{equation}\label{eq:update}
    \boldsymbol{w}(k+1)-\boldsymbol{w}(k)=-\eta\boldsymbol{h}(k)\odot \partial^s \mathcal{L}(\boldsymbol{w}(k)), 
\end{equation}where $k=1,2,\cdots$ denotes the iteration index, $\partial^s \mathcal{L}(\boldsymbol{w}(k))\in \partial \mathcal{L}(\boldsymbol{w}(k))$, $\eta$ denotes a constant learning rate, $\boldsymbol{h}(k)$ is called the conditioner which adaptively assigns different learning rates for different coordinates. For AdaGrad, {\small$\boldsymbol{h}(k)^{-1}=\sqrt{\varepsilon\mathbf{1}_p+\sum_{\tau=0}^{k}\partial^s \mathcal{L}(\boldsymbol{w}(\tau))^2 }$} where $\epsilon$ is a positive constant, and $\mathbf{1}_p$ is a length-$p$ vector with all components to be $1$. Here, $\partial^s \mathcal{L}(\boldsymbol{w}(\tau))^2=\partial^s\mathcal{L}(\boldsymbol{w}(\tau))\odot\partial^s\mathcal{L}(\boldsymbol{w}(\tau))$ and $\odot$ denotes the element-wise product of a vector.  Different from AdaGrad, RMSProp adopts exponential weighted average strategy in $\boldsymbol{h}(k)$, i.e., {\small$ \boldsymbol{h}(k)^{-1}=\sqrt{\varepsilon\mathbf{1}_p+\sum_{\tau=0}^{k}(1-b) b^{k-\tau}\partial^s \mathcal{L}(\boldsymbol{w}(\tau))^2 }$}. {Adam further introduces a bias-correction coefficient $\frac{1}{1-b^k}$ and $\boldsymbol{h}(k)^{-1}=\sqrt{\varepsilon\mathbf{1}_p+\frac{\sum_{\tau=0}^{k}(1-b) b^{k-\tau}\partial^s \mathcal{L}(\boldsymbol{w}(\tau))^2}{1-b^k} }$.} In this paper, we use $\boldsymbol{h}^{A}(k)$, $\boldsymbol{h}^{R}(k)$ and {$\boldsymbol{h}^{M}(k)$} to distinguish  the term $\boldsymbol{h}(k)$ in AdaGrad, RMSProp and Adam  respectively.

Taking $\eta\rightarrow 0$, the continuous time limits (i.e., continuous flow) of the three optimization algorithms are
\begin{equation}\label{eq:flow}
    \frac{d\boldsymbol{w}(t)}{\mathrm{d}t}=-\boldsymbol{h}(t)\odot\partial^s \mathcal{L}(\boldsymbol{w}(t)),
\end{equation}
{\small$\boldsymbol{h}^A(t)^{-1}=\sqrt{\varepsilon\mathbf{1}_p+\int_{0}^{t} \partial^s \mathcal{L}(\boldsymbol{w}(\tau))^2 \mathrm{d} \tau}$, $\boldsymbol{h}^{R}(t)^{-1}=\sqrt{\varepsilon\mathbf{1}_p+\int_{0}^{t}(1-b)e^{-(1-b)(t-\tau)} \partial^s \mathcal{L}(\boldsymbol{w}(\tau))^2 \mathrm{d} \tau}$} and $\boldsymbol{h}^{M}(t)^{-1}=\sqrt{\varepsilon\mathbf{1}_p+\frac{\int_{0}^{t}(1-b)e^{-(1-b)(t-\tau)} \partial^s \mathcal{L}(\boldsymbol{w}(\tau))^2 \mathrm{d} \tau}{1-b^t}}$. 

Our study will start with the continuous version of the two algorithms. Specifically, for the continuous case, we focus on the following scenario.
\begin{assumption}
\label{assum: continuous}
The empirical loss is defined as $\mathcal{L}(\boldsymbol{w})=\sum_{i=1}^N e^{-f(q_i(\boldsymbol{w}))}$. The following propositions hold:

\renewcommand{\labelenumi}{\Roman{enumi}}
\renewcommand{\labelenumii}{\roman{enumii}}
\begin{enumerate}
    \item (Regularity). For any  $i$, $\Phi(\boldsymbol{w},\boldsymbol{x}_i)$ is locally Lipschitz and admits a chain rule with respect to $\boldsymbol{w}$;
    \item (Homogeneity). There exists $L>0$ such that $\forall \alpha>0$ and $i$, $\Phi(\alpha\boldsymbol{w},\boldsymbol{x}_i)=\alpha^{L} \Phi(\boldsymbol{w},\boldsymbol{x}_i)$;
    \item (Separability). There exists a time $t_0$ such that $f^{-1}(\log \frac{1}{ \mathcal{L}(t_0)})>0$.
\end{enumerate}
\end{assumption}

 I, II in Assumption \ref{assum: continuous} holds for a board class of networks allowing for ReLU, max pooling, and convolutional layers;
Assumption \ref{assum: continuous}.III holds generally for over-parameterized neural networks, which can achieve complete correct classification in training set.

\subsection{KKT point}
{
We give a brief introduction to KKT conditions and KKT points. For a constrained optimization problem defined as 
\begin{equation*}
    \min f(\boldsymbol{w}) \text{ subject to: }g_i(\boldsymbol{w})\le 0,\text{ }\forall i\in[N],
\end{equation*}
KKT conditions are necessary conditions for a point $\boldsymbol{w}_0$ to be optimal in above problem, which  require that there exists non-negative reals $\lambda_i$, such that
\begin{gather}
\label{eq:KKT_varepsilon}  0 \in  \partial f(\boldsymbol{w}_0)+\sum_{i=1}^N\lambda_i \partial g_i (\boldsymbol{w}_0); \quad
    \sum_{i=1}^N\lambda_i g_i (\boldsymbol{w}_0) = 0.
\end{gather}
}A weaker notion of KKT condition is $(\varepsilon,\delta)$ KKT condition, which requires left sides of eq. (\ref{eq:KKT_varepsilon}) 
to be respectively smaller than $\varepsilon$ and $\delta$. We will formally define $(\varepsilon,\delta)$ KKT points and give some of their properties in Appendix \ref{sec:appen_KKT}.  

\textbf{Notations.} {In this paper, we use $\boldsymbol{o}$, $\mathcal{O}$, $\Theta$, and $\Omega$ to hide the absolute multiplicative factors. Concretely,
$f(t)=\boldsymbol{o}(g(t))$ if $\overline{\lim}_{t\rightarrow\infty} \frac{f(t)}{g(t)}=0$; $f(t)=\mathcal{O}(g(t))$ if $\overline{\lim}_{t\rightarrow\infty} \frac{f(t)}{g(t)}<\infty$; $f(t)=\Omega(g(t))$ if $\underline{\lim}_{t\rightarrow\infty} \frac{f(t)}{g(t)}>0$; $f(t)=\Theta(g(t))$ if $f(t)=\Omega(g(t))$ and $f(t)=\mathcal{O}(g(t))$.}  

\section{Main Results}
In this section, we introduce the main results on convergent direction of adaptive optimization algorithms. In Section \ref{sec4.1}, we propose a unified adaptive gradient flow and prove that it converges to KKT point of max-margin problem. In Section \ref{sec4.2}, we apply results for adaptive gradient flow to AdaGrad, RMSProp and Adam (w/m) to get the convergent directions of their continuous flow. In Section \ref{sec4.3}, we prove the convergent directions of the discrete update rules of adaptive optimization algorithms.
\subsection{Adaptive Gradient Flow: Definition and Results}\label{sec4.1}
Adaptive optimizers such as AdaGrad, RMSProp and Adam can be viewed as adding component-wise conditioner to gradient updates and the limit of the component-wise conditioner may be anisotropic for different components. We first define adaptive gradient flow whose limit of component-wise conditioner is isotropic. 

\begin{definition}
\label{def:approximate_GF}
A function $\boldsymbol{v}(t)$ is called to obey an adaptive gradient flow $\mathcal{F}$ with loss $\mathcal{L}$ and component learning rate $\boldsymbol{\beta}(t)$, if it can be written as the following form
\begin{equation*}
    \frac{d \boldsymbol{v}(t)}{\mathrm{d}t}= -\boldsymbol{\beta}(t)\odot \partial^s \tilde{\mathcal{L}}(\boldsymbol{v}(t)),
\end{equation*}
where $\partial^s \tilde{\mathcal{L}} (\bv(t))\in \partial \tilde{\mathcal{L}} (\bv(t))$,  $\boldsymbol{\beta}(t)$ satisfies that $\lim_{t\rightarrow\infty}$ $ \boldsymbol{\beta}(t)=\mathbf{1}_p$, and $\frac{\mathrm{d}\log \boldsymbol{\beta}(t)}{\mathrm{d}t}$ is Lebesgue Integrable.
\end{definition}

We make some explanations for Definition \ref{def:approximate_GF}: 
Conditions $\lim_{t\rightarrow\infty} \boldsymbol{\beta}(t)=\mathbf{1}_p$ and $\frac{\mathrm{d}\log \boldsymbol{\beta}(t)}{\mathrm{d}t}$ being Lebesgue Integrable ensures $\boldsymbol{\beta}(t)$ converges to $\mathbf{1}_p$ without large fluctuation. These constraints are common, in the sense that AdaGrad, RMSProp and Adam (w/m) can be transferred into such flows by simple reparameterization (see Section \ref{sec4.2}); but are also vital, which guarantee adaptive gradient flows converge to KKT point of max-margin problem as follows:

\begin{theorem}
\label{thm:approximate_flow}
Let $\boldsymbol{v}$ obey an adaptive gradient flow $\mathcal{F}$ which satisfies Assumption \ref{assum: continuous}. Let $\bar{\boldsymbol{v}}$ be any limit point of $\{\hat{\boldsymbol{v}}(t)\}_{t=0}^{\infty}$ (where $\hat{\boldsymbol{v}}(t)=\frac{\boldsymbol{v}(t)}{\Vert \boldsymbol{v}(t)\Vert}$ ). Then $\bar{\boldsymbol{v}}$ is along the direction of a KKT point of the following $L^2$ max-margin problem $(P)$:
\begin{gather*}
\min \frac{1}{2}\Vert \boldsymbol{v} \Vert^2\\
\text{subject to } \tilde{q}_i(\boldsymbol{v}) \ge 1, \forall i \in[N].
\end{gather*}
\end{theorem}

$(P)$ is equivalent to the $L^2$ max-margin problem: suppose $\boldsymbol{v}_0$ is an optimal point of $(P)$. Then there exists an $i \in [N]$, such that, $\tilde{q}_i(\boldsymbol{v}_0)=1$ (otherwise, we can let $\boldsymbol{v_0}'=\boldsymbol{v}_0/\tilde{q}_{\min}(\boldsymbol{v}_0)^{\frac{1}{L}}$. Then $\boldsymbol{v_0}'$ is also a fixed point of $(P)$ and have a smaller $L^2$ norm than $\boldsymbol{v_0}'$, which leads to contradictory). Therefore, $\tilde{q}_{\min}(\boldsymbol{v}_0)=1$ ($\tilde{q}_{\min}(\boldsymbol{v})\overset{\triangle}{=}\min_i\{\tilde{q}_i(\boldsymbol{v})\}$), and maximizing the normalized margin $\frac{\tilde{q}_{\min}(\boldsymbol{v})}{\Vert \boldsymbol{v}\Vert^L}$ is equivalent to minimize $\Vert \boldsymbol{v}\Vert^2$.


 Theorem \ref{thm:approximate_flow} shows that the adaptive gradient flow actually drives the parameters to solutions of $L^2$ max-margin problem. We will give the proof skeleton of Theorem \ref{thm:approximate_flow} in Section \ref{sec:main_result_proof}.

\begin{remark}
\label{remark: multi_class}
Our result can be 
extended to the multi-class classification with logistic loss and same assumption as Assumption \ref{assum: continuous} except that $\Phi(\boldsymbol{w},\boldsymbol{x}_i)$ is a $C$-dimension vector in multi-class case with $C$ number of classes. The corresponding $L^2$ max-margin classification problem is then
\begin{gather*}
\min \frac{1}{2}\Vert \boldsymbol{v} \Vert^2\\
\text{subject to } (\Phi(\boldsymbol{w},\boldsymbol{x}_i))_{y_i}-(\Phi(\boldsymbol{w},\boldsymbol{x}_i))_{j} \ge 1,\\
\forall i \in[N], j\in[C]/\{y_i\}.
\end{gather*}
We defer the proof to Appendix \ref{appen: multi-class-classification}.
\end{remark}

While Theorem \ref{thm:approximate_flow}
 does NOT guarantee direction of parameters converges as $t\rightarrow \infty$, we present a theorem in the end of this section which provides such a guarantee when neural network $\Phi$ is definable with respect to parameters $\boldsymbol{w}$.

\begin{theorem}
\label{thm:approximate_flow_definable}
Let all assumptions in Theorem \ref{thm:approximate_flow} hold. Assume further $\Phi(\boldsymbol{w},\boldsymbol{x}_i)$ is definable with respect to parameter $\boldsymbol{w}$ for any $i\in [N]$. Then direction of parameters $\{\hat{\boldsymbol{v}}(t)\}_{t=0}^{\infty}$ converges.
\end{theorem}

We defer the formal definition of definable to Appendix \ref{sec:proof_definable}, but point out here that definability allows for linear, ReLU, polynomial activations, max pooling and convolutional layers, and skip connections. Furthermore, for locally Lipschitz definable function, chain rule holds almost everywhere (Lemma \ref{lem:chain_rule}).

The proof can be derived by bounding the curve length of $\hat{\boldsymbol{v}}(t)$ using $\tilde{\gamma}(t)$ and Kurdyka-Lojasiewicz inequalities developed in \cite{ji2020directional}, and we defer the details to Appendix \ref{sec:proof_definable}.

\subsection{Results for Adaptive Algorithms: Continuous Case}\label{sec4.2}
In this section, we will prove gradient flow of AdaGrad, RMSProp, and Adam (w/m) can be transferred into adaptive gradient flow. We start from proving convergence of conditioner in AdaGrad and further shows AdaGrad can be reparameterized as an adaptive gradient flow. 


\begin{theorem}
\label{lem:adagrad_approximate}
For AdaGrad flow defined as eq. (\ref{eq:flow}) with $\boldsymbol{h}(t)=\boldsymbol{h}^{A}(t)$, we have that
\begin{itemize}
    \item $\boldsymbol{h}^{A}(t)$ converges as $t\rightarrow\infty$. Furthermore, $\boldsymbol{h}_{\infty}=\lim_{t\rightarrow \infty} \boldsymbol{h}^{A}(t)$ has no zero component.
    \item 
    $\frac{\mathrm{d} \boldsymbol{v}^A(t)}{\mathrm{d}t}\in -\boldsymbol{\beta}^A(t)\odot \partial \tilde{\mathcal{L}}^A(\boldsymbol{v}^A(t))$
satisfies definition of adaptive gradient flow, where
\begin{align*}
    &\boldsymbol{v}^{A}(t)=\boldsymbol{h}_{\infty}^{-1 / 2} \odot \boldsymbol{w}(t),\boldsymbol{\beta}^A(t)=\boldsymbol{h}_{\infty}^{-1} \odot \boldsymbol{h}^A(t),\\
    &\tilde{\mathcal{L}}^A(\boldsymbol{v}^{A})=\mathcal{L}(\boldsymbol{h}_{\infty}^{\frac{1}{2}}\odot \boldsymbol{v}^{A}).
\end{align*}

\end{itemize}

\end{theorem}
We provide some intuitions for proof of Theorem \ref{lem:adagrad_approximate}. The former part of the first property is because $\boldsymbol{h}^A(t)$ is non-increasing with respect to $t$. However, the latter part yields that integration of square of the gradient converges to a positive real, which is non-trivial; the second property is obtained by component-wisely scaling $\boldsymbol{w}$ and direct verification; the last property can be obtained by Newton-Leibniz formula for absolutely continuous function since $\frac{\mathrm{d}\boldsymbol{\beta}^A(t)}{\mathrm{d}t}$ is non-negative. We defer the detailed proof to Appendix \ref{sec:appen_approximate}.  

Similar properties also hold for RMSProp and Adam (w/m) as the following Theorem .

\begin{theorem}
\label{lem:rms_approximate}
For RMSProp and Adam flow defined as eq. (\ref{eq:flow}) respectively with $\boldsymbol{h}(t)=\boldsymbol{h}^{R}(t)$ and $\boldsymbol{h}(t)=\boldsymbol{h}^{M}(t)$, we have that, for $I\in\{R,M\}$,
\begin{itemize}
    \item $\boldsymbol{h}^{I}(t)$ converges as $t\rightarrow\infty$. Furthermore, $\lim_{t\rightarrow \infty} \boldsymbol{h}^{I}(t)=\varepsilon^{-\frac{1}{2}}\mathbf{1}_p^T$.
    \item 
    $\frac{\mathrm{d} \boldsymbol{v}^I(t)}{\mathrm{d}t}\in -\boldsymbol{\beta}^I(t)\odot \bar{\partial} \tilde{\mathcal{L}}^I(\boldsymbol{v}^A(t))$
satisfies definition of adaptive gradient flow, where
{\small\begin{gather*}
    \boldsymbol{v}^{I}(t)=\varepsilon^{\frac{1}{4}} \boldsymbol{w}(t), \boldsymbol{\beta}^I(t)=\varepsilon^{\frac{1}{2}} \boldsymbol{h}(t),
    \tilde{\mathcal{L}}^I(\boldsymbol{v}^{I})=\mathcal{L}(\varepsilon^{-\frac{1}{4}} \boldsymbol{v}^{I}).
\end{gather*}}
\end{itemize}
\end{theorem}
Both conditioners $\boldsymbol{h}^R$ and $\boldsymbol{h}^M$ have an exponential decay term $e^{-(t-\tau)(1-b)}$, which drives $\int_{\tau=0}^t (1-b)e^{-(1-b)(t-\tau)}\bar{\partial}\mathcal{L}(\boldsymbol{w}(\tau))^2d \tau$ to zero, and conditioners to isotropy. The detailed proof requires a more careful analysis in measure than the AdaGrad flow. We defer them to Section \ref{sec:appen_approximate}.  

By Theorems \ref{lem:adagrad_approximate} and \ref{lem:rms_approximate}, gradient flow of AdaGrad, RMSProp and Adam (w/m) can both be transferred into adaptive gradient flows: $\boldsymbol{v}^I$ obeys an adaptive gradient flow with loss $\tilde{\mathcal{L}}^I$ and conditioner $\boldsymbol{\beta}^I$ ($I\in{A,R,M}$). Furthermore, Assumption \ref{assum: continuous} also holds for $\tilde{\mathcal{L}}^I$: for AdaGrad, RMSProp, and Adam , we can uniformly represent $\tilde{\mathcal{L}}^A(\boldsymbol{v}^A)$, $\tilde{\mathcal{L}}^R(\boldsymbol{v}^R)$, and $\tilde{\mathcal{L}}^M(\boldsymbol{v}^M)$ as $\mathcal{L}(\tilde{\boldsymbol{h}}^{\frac{1}{2}}\odot \boldsymbol{v})$, where  $\tilde{\boldsymbol{h}}$ is a component-wisely positive constant vector. By Assumption \ref{assum: continuous},  $\tilde{\mathcal{L}}(\boldsymbol{v})$ can be further written as  
\begin{align*}
    \tilde{\mathcal{L}}(\boldsymbol{v})=&\mathcal{L}(\tilde{\boldsymbol{h}}^{\frac{1}{2}}\odot \boldsymbol{v})=\sum_{i=1}^N e^{-f(q_i(\tilde{\boldsymbol{h}}^{\frac{1}{2}}\odot \boldsymbol{v}))}.
\end{align*}
If we denote $\tilde{q}_i(\boldsymbol{v})=q_i(\tilde{\boldsymbol{h}}^{\frac{1}{2}}\odot \boldsymbol{v})$, we have $\tilde{q}_i$ is also an $L$ homogeneous function, and $\tilde{\mathcal{L}}(\boldsymbol{v})=\sum_{i=1}^N e^{-f( \tilde{q}_i(\boldsymbol{v}))}$. 

Combining Theorem \ref{thm:approximate_flow} with Theorems \ref{lem:adagrad_approximate} and \ref{lem:rms_approximate}, one can obtain convergent directions of AdaGrad flow and RMSProp flow by simple parameter substitution of $(P)$.

\begin{theorem}
\label{thm:AdaGrad_flow}
Let $\boldsymbol{w}$ satisfy AdaGrad flow defined as eq. (\ref{eq:flow}) with $\boldsymbol{h}(t)=\boldsymbol{h}^A(t)$. Then, any limit point of $\{\hat{\boldsymbol{w}}(t)\}_{t=0}^{\infty}$ (where $\hat{\boldsymbol{w}}(t)=\frac{\boldsymbol{w}(t)}{\Vert \boldsymbol{w}(t)\Vert}$ is normalized parameter) is along the direction of a KKT point of the following optimization problem $(P^A)$:
\begin{gather*}
    \min \frac{1}{2}\Vert \boldsymbol{h}_{\infty}^{-\frac{1}{2}}\odot \boldsymbol{w} \Vert^2\\
    \text{Subject to: } q_i(\boldsymbol{w})\ge 1.
\end{gather*}
\end{theorem}

 \begin{theorem}
\label{thm:RMS_flow}
Let $\boldsymbol{w}$ satisfy RMSProp or Adam flow defined as eq. (\ref{eq:flow}) respectively with $\boldsymbol{h}(t)=\boldsymbol{h}^R(t),\boldsymbol{h}^M(t)$. Then, any limit point of $\{\hat{\boldsymbol{w}}(t)\}_{t=0}^{\infty}$ (where $\hat{\boldsymbol{w}}(t)=\frac{\boldsymbol{w}(t)}{\Vert \boldsymbol{w}(t)\Vert}$ is normalized parameter) is along the direction of a KKT point of the following optimization problem $(P^R)$:
\begin{gather*}
    \min\frac{1}{2}\Vert  \boldsymbol{w} \Vert^2\\
    \text{Subject to: } q_i(\boldsymbol{w})\ge 1.
\end{gather*}
\end{theorem}

Intuitively, $(P^R)$ is the $L^2$ max-margin problem, which means RMSProp flow biases parameters to a local minimum with good generalization property; on the other hand, the target of $(P^A)$ has a reliance of $\boldsymbol{h}_{\infty}$, which is a constant vector in $(P^A)$ but can be influenced by the optimization process and initialization, and may further lead to worse generalization. We will discuss the difference between convergent directions of AdaGrad and RMSProp in detail in Section \ref{sec:discussion_generalization}.

\subsection{Results for Adaptive Algorithms: Discrete Case}\label{sec4.3}
In practice, gradient descent methods are employed since calculating exact gradient flow requires huge efforts. In this section, we show same results hold in Theorems \ref{thm:AdaGrad_flow} and \ref{thm:RMS_flow} for discrete update rules of adaptive algorithms with slightly different assumptions.

As for the discrete case, two additional assumptions are needed as follows (For brevity, we put the complete assumption to the appendix):
\begin{assumption}
\label{assum:discrete}
\renewcommand{\labelenumi}{\Roman{enumi}}
\renewcommand{\labelenumii}{\roman{enumii}}
\begin{enumerate}
    \item (smooth). For any fixed $x$, $\Phi(\cdot;x)$ is $M$ smooth (i.e., $\Phi$ is twice continuously differentiable with respect to $x$ and all the eigenvalues of the Hessian are within $[-M,M]$);
    \item (Learning Rate). For $k>k_0$, $\eta_t\le C(t)$, where $C(t)$ is a non-decreasing function (defined in Appendix \ref{sec:discrete case}). Also, $\eta_t$ is lower bounded by a positive real, that is, there exists a constant $\tilde{\eta}>0$, such that, for any $k>k_0$, $\eta_k\ge \tilde{\eta}$.
\end{enumerate}
\end{assumption}
 
We make the following explanations for Assumption \ref{assum:discrete}. Assumption \ref{assum:discrete}(I) is needed technically because we need to consider second order Taylor expansion around each point along the training $\{\boldsymbol{w}(k); k=1,\cdots,\}$. Results based on this assumption are the state-of-art in the existing literature of the implicit bias of GD (e.g. [3]). We put loosening this assumption to future works. Assumption \ref{assum:discrete}(II) guarantees that the second order Taylor expansion is upper bounded and the step size is not too small. With Assumption \ref{assum:discrete}, we have the following theorem:
\begin{theorem}
\label{thm:discrete_adaptive}
With Assumptions \ref{assum: continuous} and Assumption \ref{assum:discrete}, Theorems \ref{thm:AdaGrad_flow} and Theorems \ref{thm:RMS_flow} hold respectively for discrete update of AdaGrad and discrete updates of RMSProp and Adam. 
\end{theorem}
We put the proof for Theorem \ref{thm:discrete_adaptive} to Appendix \ref{sec:discrete case}.
\subsection{Discussions}\label{sec:discussion_generalization}
We make some discussions on the results derived in Section 4.2 and 4.3. 
First, as shown in \cite{li2019implicit}, the optimization problem $P^R$ is equivalent to $L_2$ margin maximization problem. 
Theorems \ref{thm:RMS_flow} and \ref{thm:AdaGrad_flow} show that RMSProp and Adam (w/m) converge to max-margin solution,  while AdaGrad may drive the parameters to a different direction. The corresponding optimization problem of AdaGrad has a reliance on $\boldsymbol{h}_{\infty}$, which is shown to be sensitive to the optimization path before convergence (shown in Section 6.2), and makes the convergent direction sensitive (we will discuss this in detail in Appendix \ref{appen: h_infty}). Because the normalized margin is used as a complexity norm in generalization literature (i.e., larger normalized margin indicating better generalization performance) \cite{bartlett1999generalization}, {our results indicate the superiority on generalization of exponential moving average strategy in the design of the conditioner.}



Second, two key factors that guarantee generalization of RMSProp and Adam are exponential weighted average design on the conditioner and the added constant $\epsilon$ in $\boldsymbol{h}(t)$. Our results show the benefit of the two factors: it accelerates the training process at early stage of optimization by adaptively adjusting the learning rate, but it still converges to max-margin solution because the denominator of conditioner tends to constant $\epsilon$ at later stage.  Most of previous works explain $\epsilon$ to ensure positivity of $\boldsymbol{h}(t)$. Our results show that $\epsilon$ is important for the convergent direction of the parameters and the generalization ability. 

\section{Proof Sketch of Theorem \ref{thm:approximate_flow}}
\label{sec:main_result_proof}
In this section, we present the proof sketch of Theorem \ref{thm:approximate_flow}.  The proof can be divided into three stages: (I) we define surrogate margin and prove that it is lower bounded and equivalent to normalized margin as time tends to infinity; (II) We use surrogate margin to lower bound the decreasing rate of empirical loss $\mathcal{L}$, and prove $\lim_{t\rightarrow \infty} \tilde{\mathcal{L}}(\boldsymbol{w}(t))=0$; (III) For every convergent direction $\bar{\boldsymbol{v}}$, a series of $(\varepsilon_i,\delta_i)$ KKT point which converges to $\bar{\boldsymbol{v}}$ with $\lim_{i\rightarrow\infty} \varepsilon_i=\lim_{i\rightarrow\infty} \delta_i=0$ is constructed. We then show every convergent direction is a KKT point of optimization problem $(P)$. 


\subsection{surrogate margin on adaptive gradient flow}
\label{sec:smoothed_margin}
For adaptive gradient flow $-\boldsymbol{\beta}(t)\odot\partial^s \tilde{\mathcal{L}}(\boldsymbol{v}(t))= \frac{\mathrm{d} \bv (t)}{\mathrm{d}t }$, we first deal with the change of $\Vert \boldsymbol{v} \Vert$. To derive change of $\Vert \boldsymbol{v} \Vert$, we study the surrogate norm $\rho(t)\overset{\triangle}{=}\Vert\boldsymbol{\beta}(t)^{-\frac{1}{2}}\odot \boldsymbol{v}(t)\Vert$ because   $\rho(t)=\Theta(\Vert\boldsymbol{v}(t)\Vert)$ based on $\lim_{t\rightarrow\infty}\boldsymbol{\beta}(t)=\mathbf{1}$.

The normalized margin $\gamma(t)=\frac{\tilde{q}_{\min}(t)}{\Vert \boldsymbol{v}(t) \Vert^L}$ connects margin $\tilde{q}_{\min}(t)$ with parameter norm  $\Vert \boldsymbol{v}(t) \Vert$. 
The next lemma admits us to define an surrogate margin using $\tilde{\mathcal{L}}$.

\begin{lemma}
\label{lem:relationship_between_normalized_smoothed}
If  $\lim_{t\rightarrow\infty}\tilde{\mathcal{L}}=0$, we have $\frac{f^{-1}(\log\frac{1}{\tilde{\mathcal{L}}(\boldsymbol{v}(t))})}{\rho(t)^L}=\Theta(\gamma(t))$.

\end{lemma}

Based on Lemma \ref{lem:relationship_between_normalized_smoothed}, we define surrogate margin $\tilde{\gamma}(t)$ as 
\begin{equation*}
    \tilde{\gamma}(t)=\frac{f^{-1}(\log \frac{1}{\tilde{\mathcal{L}}(\boldsymbol{v}(t))})}{\rho(t)^L}.
\end{equation*}

Since $\rho(t)=\Theta(\Vert\boldsymbol{v}(t)\Vert)$, $\tilde{\gamma}$ actually bridge the norm of parameters with empirical loss. A desired property for  $\tilde{\gamma}$ is to have a positive lower bound, since with this property, one can further bound parameter norm using empirical loss. The following lemma shows that $\tilde{\gamma}$ is lower bounded for adaptive gradient flow $\mathcal{F}$ with empirical loss $\tilde{\mathcal{L}}$ satisfying Assumption \ref{assum: continuous}.

\begin{lemma}
\label{lem:lower_bound_margin}
Let a function $\boldsymbol{v}(t)$ obey an adaptive gradient flow $\mathcal{F}$ with loss $\tilde{\mathcal{L}}$ and component learning rate $\boldsymbol{\beta}(t)$, where $\tilde{\mathcal{L}}$ satisfies Assumption \ref{assum: continuous}. Then there exists a time $t_1\ge t_0$, such that, for any time $t\ge t_1$, $\tilde{\gamma}(t)\ge e^{-\frac{1}{2}}\tilde{\gamma}(t_1)$.
\end{lemma}

\begin{remark}
Our surrogate margin can be obtained by replacing $\Vert \boldsymbol{v}(t)\Vert$  by $\rho(t)=\Vert \boldsymbol{\beta}(t)^{-\frac{1}{2}}\odot\boldsymbol{v}(t)\Vert$ in the smoothed margin in $\text{[3]}$. This allows us to lower bound the derivative of surrogate margin and further lower bound the surrogate margin as Lemma \ref{lem:lower_bound_margin}, while the derivative of smoothed margin for adaptive gradient flow can not be bounded easily.
\end{remark}

Here we briefly give a road map of the proof. The derivative of norm $\rho(t)$ can be split into two parts: one is the increasing of parameter $\boldsymbol{v}$, and another is the change of component learning rate $\boldsymbol{\beta}^{-\frac{1}{2}}(t)$.  Applying homogeneity of $\tilde{q}_i$ and
Cauchy–Schwarz inequality, we bound the first term using the derivative of $f^{-1}\left(\log \left(\frac{1}{\tilde{\mathcal{L}}(t)} \right)\right)$; the second term can be lower bounded by {\small$\sum_{i=1}^{p} \left(\frac{d \log \boldsymbol{\beta}_{i}^{-\frac{1}{2}}(t)}{\mathrm{d} t}\right)_{+}$}, whose integration is bounded by the definition of adaptive gradient flow. The proof is completed by putting two parts together.

By the discussion above, one can conclude that  derivative of $\tilde{\gamma}(t)$ can be calculated by subtracting a small enough term  {\small$\sum_{i=1}^{p} \left(\frac{d \log \boldsymbol{\beta}_{i}^{-\frac{1}{2}}(t)}{\mathrm{d} t}\right)_{+}$} from a non-negative term. This fact leads to the the convergence of $\tilde{\gamma}(t)$. 

\begin{lemma}
\label{lem:conver_gamma}
Suppose a function $\boldsymbol{v}$ obey an adaptive gradient flow $\mathcal{F}$, which satisfies Assumption \ref{assum: continuous}. Then the surrogate margin $\tilde{\gamma}(t)$ converges.
\end{lemma}

\subsection{Convergence of Empirical Loss and Parameters}
\label{sec:convergence_erm}
By Lemma \ref{lem:lower_bound_margin}, we have that for an adaptive gradient flow $\mathcal{F}$ with Assumption \ref{assum: continuous}, the norm $\rho(t)$ can be bounded as $\rho(t)= \mathcal{O}(  f^{-1}(\log(\frac{1}{\tilde{\mathcal{L}}(\boldsymbol{v}(t))}))^{\frac{1}{L}})$. On the other hand, by chain rule, the derivative of empirical loss with respect to time can be calculated as
{\small\begin{align*}
    &\frac{\mathrm{d} \tilde{\mathcal{L}}(\boldsymbol{v}(t))}{\mathrm{d}t}=\langle\partial^s \tilde{\mathcal{L}}(\boldsymbol{v}(t)) ,\frac{\mathrm{d} \boldsymbol{v}(t)}{\mathrm{d}t}\rangle
=-\Vert\boldsymbol{\beta}^{\frac{1}{2}}(t)\odot \partial^s \tilde{\mathcal{L}}(\boldsymbol{v}(t))\Vert^2\\
    &\le-\frac{\langle\partial^s \tilde{\mathcal{L}}(\boldsymbol{v}(t)), \boldsymbol{v}(t) \rangle^2}{\rho(t)^2},
\end{align*}}
where the last inequality is derived by the Cauchy inequality applying to $\langle\boldsymbol{\beta}^{\frac{1}{2}}\odot{\partial^s} \tilde{\mathcal{L}}, \boldsymbol{\beta}^{-\frac{1}{2}}\odot\boldsymbol{v} \rangle$. By the homogeneity of $\tilde{q}_i$, we can further lower bound $\frac{\langle{\partial^s} \tilde{\mathcal{L}}, \boldsymbol{v} \rangle^2}{\rho^2}$ using $\tilde{\mathcal{L}}$.
In other words, Lemma \ref{lem:lower_bound_margin} ensures that the decreasing rate of the empirical loss $\tilde{\mathcal{L}}$ can be lower bounded by a function of itself. Based on the above methodology, we can prove  that empirical loss will decrease to zero, while parameter norm will converge to infinity as the following lemma.   

\begin{lemma}
\label{lem:assymptotic_loss}

Let a function $\boldsymbol{v}(t)$ obey an adaptive gradient flow $\mathcal{F}$ with loss $\tilde{\mathcal{L}}$ and component learning rate $\boldsymbol{\beta}(t)$, where $\tilde{\mathcal{L}}$ satisfies Assumption \ref{assum: continuous}. Then, $\lim_{t\rightarrow\infty} \tilde{\mathcal{L}}(\boldsymbol{v}(t))=0$, and consequently,
$\lim_{t\rightarrow\infty} \Vert \boldsymbol{v}(t)\Vert=\infty$.
\end{lemma}


\subsection{Convergence to KKT point}
\label{sec:convergent_KKT}

 We start by proving for any $t\ge t_1$, $\hat{\boldsymbol{v}}(t)=\frac{\boldsymbol{v}(t)}{\Vert \boldsymbol{v}(t)\Vert}$ is an approximate KKT point.   Based on the surrogate margin that we construct in Section \ref{sec:smoothed_margin}, we can further show for normalized $\boldsymbol{v}$ is an approximate KKT point as the following Lemma :
\begin{lemma}
\label{lem:construction_kkt}
Let $\hat{\boldsymbol{v}}$ and $\widehat{{\partial^s} \tilde{\mathcal{L}}(\boldsymbol{v})}$ be  $\boldsymbol{v}$ and ${\partial^s} \tilde{\mathcal{L}}(\boldsymbol{v})$ respectively normalized by their $L^2$ norms. Then $\hat{\boldsymbol{v}}(t)$ is a $\left(\mathcal{O}(1-\langle\hat{\boldsymbol{v}}(t),\widehat{-{\partial}^s \tilde{\mathcal{L}}(t)}\rangle),\mathcal{O}\left(\frac{1}{\log \frac{1}{\tilde{\mathcal{L}}(\boldsymbol{t})}}\right)\right)$ KKT point of optimization problem $(P)$ in Theorem \ref{thm:approximate_flow}. 
\end{lemma}

We made some explanations to Lemma \ref{lem:construction_kkt}: by the results in Section \ref{sec:convergence_erm}, we have $\lim_{t\rightarrow\infty}\mathcal{O}\left(\frac{1}{\log \frac{1}{\tilde{\mathcal{L}}(t)}}\right)=0$. Therefore, we only need to find a convergent  series $\tilde{\boldsymbol{v}}(t)$ with $1+\langle\hat{\boldsymbol{v}}(t),\widehat{{\partial}^s \tilde{\mathcal{L}}(t)}\rangle$ goes to zero. 

For this purpose, we construct an approximate norm
$\tilde{\rho}(t)$ as
    {\small$\sqrt{\rho(t)^2-2\int_{t_{1}}^{t}\left\langle\boldsymbol{v}(\tau), \boldsymbol{\beta}^{-\frac{1}{2}}(\tau) \odot \frac{d \boldsymbol{\beta}^{-\frac{1}{2}}}{\mathrm{d} t}(\tau) \odot \boldsymbol{v}(\tau)\right\rangle d \tau},
$}
 which measures the increasing of $\boldsymbol{v}(t)$. 
 $1-\langle\hat{\boldsymbol{v}}(t),-\widehat{{\partial}^s \tilde{\mathcal{L}}(t)}\rangle$ can then be bound by the next lemma:
 
 \begin{figure*}[t!]
\centering
\begin{subfigure}
[b]{0.5\columnwidth}        \includegraphics[trim=0cm 0.5cm 0cm 0cm, width=\textwidth]{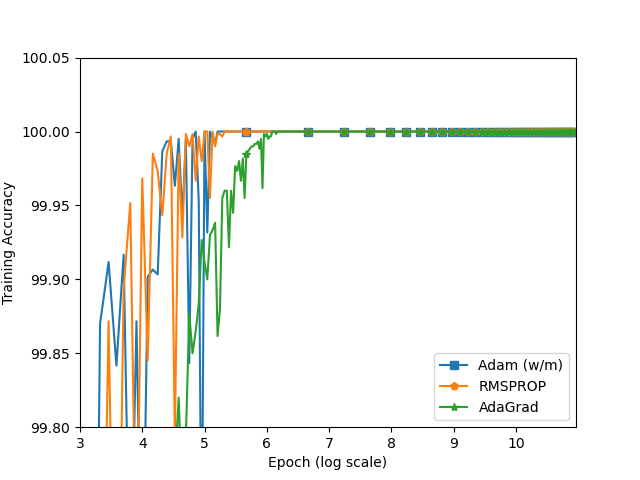}
\caption{ Training Accuracy }
\label{fig: training accuracy}
\end{subfigure}
\begin{subfigure}
   [b]{0.5\columnwidth}        \includegraphics[trim=0cm 0.5cm 0cm 0cm, width=\textwidth]{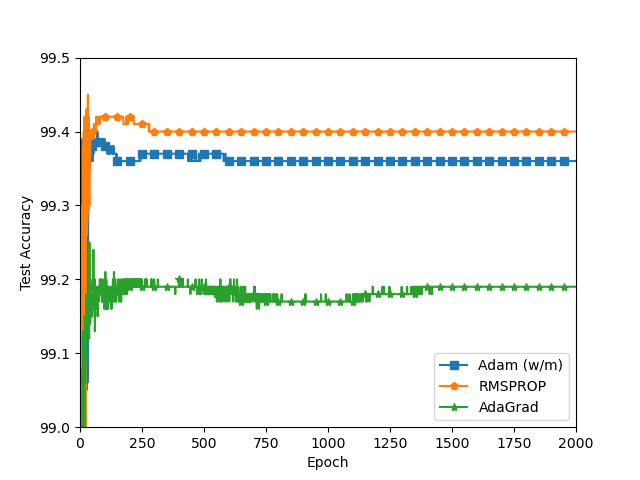}
\caption{ Test Accuracy}
\label{fig:test accuracy}
\end{subfigure}
\begin{subfigure}
[b]{0.5\columnwidth}        \includegraphics[trim=0cm 0.5cm 0cm 0cm, width=\textwidth]{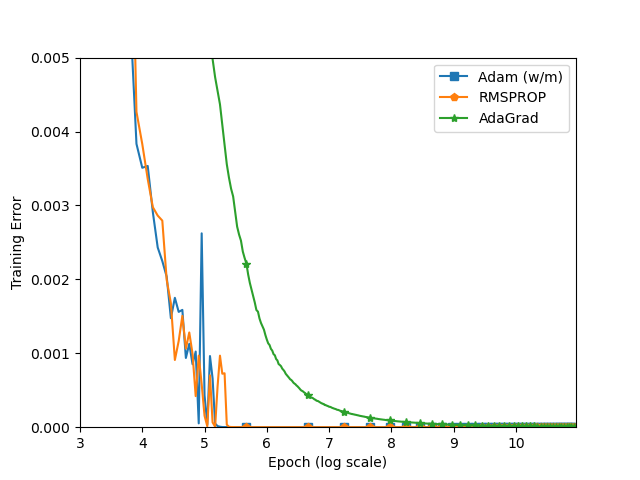}
\caption{ Training Loss}
\label{fig: loss}
\end{subfigure}
\begin{subfigure}
   [b]{0.5\columnwidth}        \includegraphics[trim=0cm 0.5cm 0cm 0cm, width=\textwidth]{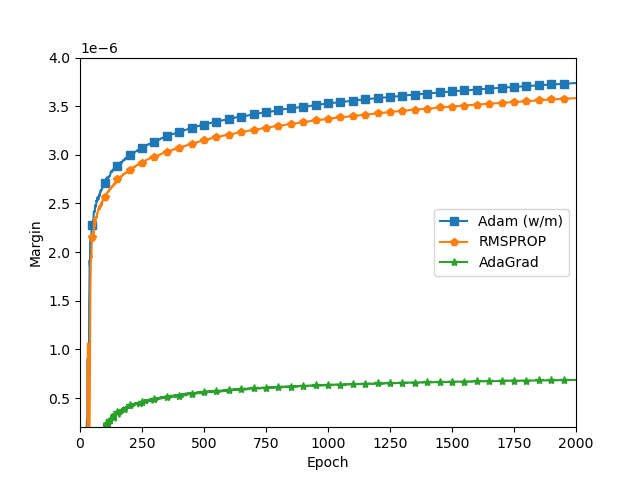}
\caption{ Normalized Margin}
\label{fig: margin}
\end{subfigure}
\vskip -0.3cm
\caption{Observation of normalized margin and generalization performance of different optimizers on MNIST. 
While all optimizers end with training accuracy $100\%$ in (a), $1-\textit{test accuracy}$ can reflect the generalization error.  }
\vskip -0.3cm
\end{figure*}
\begin{figure*}[htbp]
    \centering
    \begin{subfigure}[b]{0.5\columnwidth}   
    \includegraphics[trim=0cm 0.5cm 0cm 0cm, width=\textwidth]{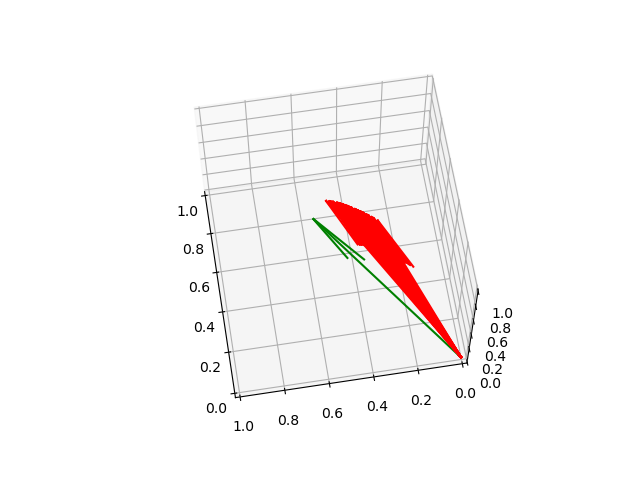}
        \caption{ $h^{-\frac{1}{2}}_{\infty}$ in AdaGrad}
        \label{fig1:direction_h}
    \end{subfigure}
    \begin{subfigure}[b]{0.5\columnwidth}     
    \includegraphics[trim=0cm 0.5cm 0cm 0cm, width=\textwidth]{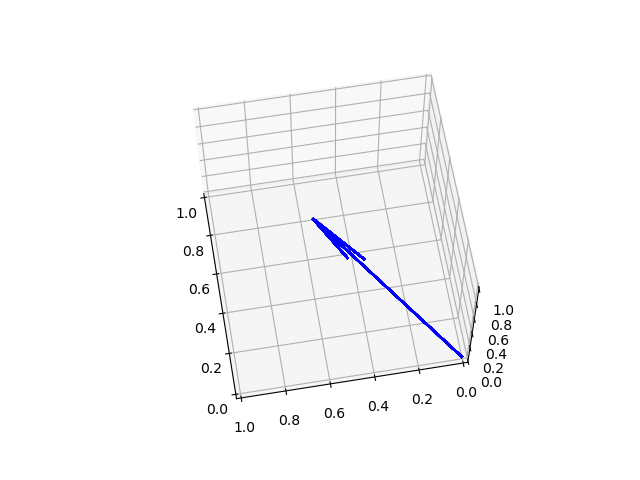}
        \caption{$h^{-\frac{1}{2}}_{\infty}$ in RMSProp }
        \label{fig2:direction_h}
    \end{subfigure}
   \begin{subfigure}[b]{0.5\columnwidth}   
    \includegraphics[trim=0cm 0.5cm 0cm 0cm, width=\textwidth]{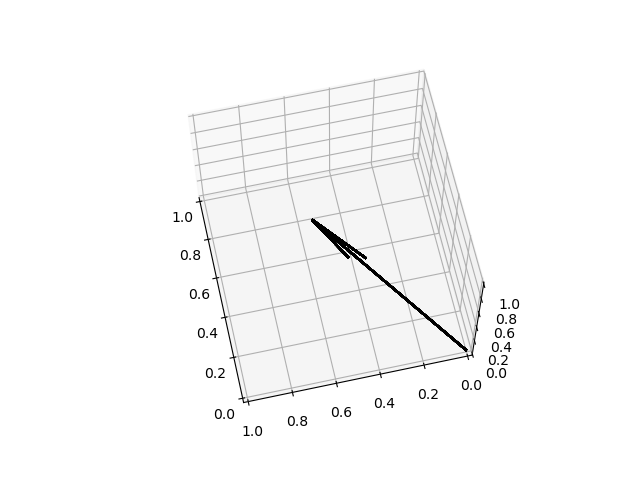}
        \caption{$h^{-\frac{1}{2}}_{\infty}$ in Adam}
        \label{fig3:direction_h}
    \end{subfigure}
    \begin{subfigure}[b]{0.5\columnwidth}   
    \includegraphics[trim=0cm 0.5cm 0cm 0cm, width=\textwidth]{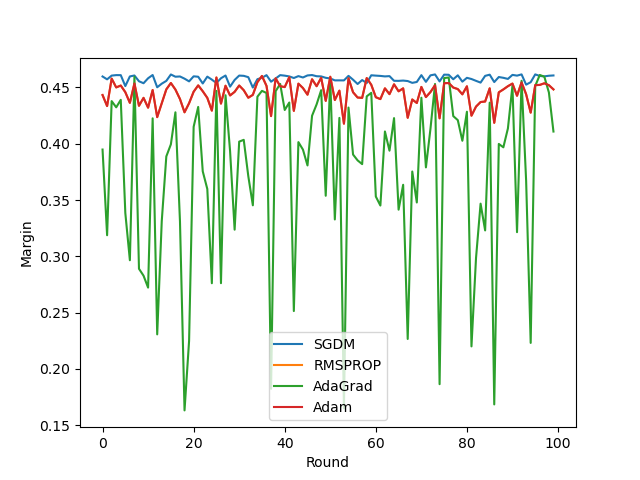}
        \caption{Margin}
        \label{fig:toy-example-margin}
    \end{subfigure}
    \vskip -0.3cm
    \caption{Direction of limit of conditioner in AdaGrad, RMSProp, and Adam (w/m) with different realizations of random initialization. In (a)-(c), the green vector stands for the isotropic direction $(\frac{1}{\sqrt{3}},\frac{1}{\sqrt{3}},\frac{1}{\sqrt{3}})$. One red vector in (a) stands  for direction of $(\boldsymbol{h}_{\infty}^A)^{-\frac{1}{2}}$ in one experiment. 
    Blue vector in (b) stands for  direction of $(\boldsymbol{h}_{\infty}^R)^{-\frac{1}{2}}$ under different initialization. Black vector in (c) stands for direction of $(\boldsymbol{h}_{\infty}^M)^{-\frac{1}{2}}$ under different initialization. In (d), final values of the margin for the four algorithms are plotted (Adam (w/m) coincides with RMSProp). }
    \vskip -0.3cm
    \label{fig:direction_h}
\end{figure*}
 
\begin{lemma}
\label{lem: bound_b}
For any $t_3>t_2\ge t_1$, there exists a $\xi \in[t_2,t_3]$, such that
{\small\begin{align*}
\left(\langle\hat{v}(\xi),-\widehat{{\partial}^s \tilde{\mathcal{L}}(\xi)}\rangle^{-2}-1\right) \le  \boldsymbol{o}\left(\frac{1}{\log \tilde{\rho}(t_3) -\log \tilde{\rho}(t_2) }\right),
\end{align*}}
and 
\begin{equation*}
   \Vert \hat{\boldsymbol{v}}(\xi)-\hat{\boldsymbol{v}}(t_2)\Vert \le  \mathcal{O}\left(\log \tilde{\rho}(t_3) -\log \tilde{\rho}(t_2) \right).
\end{equation*}
\end{lemma} 

Therefore, given a sequence of parameter direction $\{\hat{\boldsymbol{v}}(t_i)\}_{i=1}^{\infty}$ with limit $\hat{\boldsymbol{v}}$, we can always construct another sequence $\{t_i'\}_{i=1}^{\infty}$ with $\left(1-\langle\hat{v}(t_i'),-\widehat{{\partial}^s \tilde{\mathcal{L}}(t_i')}\rangle\right)$ and $\hat{\boldsymbol{v}}(t_i)-\hat{\boldsymbol{v}}(t_i')$ converging to zero.

Combining Lemma \ref{lem:construction_kkt} and \ref{lem: bound_b}, for any convergent direction $\bar{\boldsymbol{v}}$, we can construct a series of $\{t_i\}_{i=1}^{\infty}$, such that $\hat{\boldsymbol{v}}(t_i)$ is $(\varepsilon_i,\delta_i) $ KKT point, with $\lim_{i\rightarrow\infty} \hat{\boldsymbol{v}}(t_i)=\bar{\boldsymbol{v}}$, and $\lim_{i\rightarrow\infty}\varepsilon_i=\lim_{i\rightarrow\infty}\delta_i=0$. On the other hand,  constraints of $(P)$ satisfies Mangasarian-Fromovitz constraint qualification (see Appendix \ref{sec:appen_KKT}), which ensures that  $\bar{\boldsymbol{v}}$ is a KKT point of $(P)$, and completes the proof.

\section{Experiments}
\label{sec:experiment}
\subsection{Observations on Normalized Margin and Generalization Performance}
In this section, we conduct experiments to
 verify the theoretical results. We train a homogeneous neural networks using AdaGrad, RMSProp and Adam (w/m) respectively. {We adopt the homogeneous 4-layer convolutional neural network used in \cite{madry2017towards} as our model and use MNIST \cite{lecun1998mnist} as the dataset.  We use default learning rate on PyTorch platform for all the algorithms and Adam (w/m) adopts the same learning rate as Adam. Because our theory is established for full batch gradient without randomness, we set minibatch size to be $1024$  which is relatively large to mimic the full batch gradient. We put more details on the network structure and the settings of hyper-parameters in Appendix \ref{sec:expe_mnist}, where we also add standard SGD (with momentum) and Adam to observe influence of momentum.} 

We plot training accuracy, testing accuracy and training loss in Figure \ref{fig: training accuracy}, \ref{fig:test accuracy}, and \ref{fig: loss}. We also plot the value of the normalized margin during training in Figure \ref{fig: margin}.
We have the following observations:  (1) The normalized margins of AdaGrad, RMSProp and Adam (w/m) are lower bounded and the final normalized margin of AdaGrad is the lowest. It is consistent with our theoretical results. (2) The training loss of AdaGrad, RMSProp and Adam (w/m) goes to zero and AdaGrad achieves the lower test accuracy (the worse generalization), which shows the superiority of conditioners in RMSProp and Adam (w/m) on generalization. { (3) Although our theory does not include momentum version of the algorithms, the normalized margin of SGD and Adam are also lower bounded
, which shows potential on extension of our theory to momentum version. }


\subsection{Observations on Convergent Direction}
\label{sec:exper_toy_main}
In this section, we observe the direction of $\boldsymbol{h}_{\infty}$ on a simple case to illustrate that $\boldsymbol{h}_{\infty}$ of AdaGrad is anistropic and sensitive to initialization.
The model we use is expressed as  $\Phi(\boldsymbol{x},\boldsymbol{w},v)=v\sigma(\langle\boldsymbol{w},\boldsymbol{x}\rangle)$, where $x\in\mathbb{R}^2, \boldsymbol{w}\in\mathbb{R}^2$ and $v\in\mathbb{R}$ and $\sigma(x)$ is the Leaky ReLU activation function, i.e., $\sigma(x)=x$ for $x\geq 0$ and $\sigma(x)=\frac{x}{2}$ for $x<0$.

{We repeat AdaGrad, RMSProp and Adam (w/m) for 100 rounds with different random seeds of initialization. We plot  $\boldsymbol{h}^{-\frac{1}{2}}_{\infty}$ for AdaGrad, RMSProp and Adam in Figure 2 (a), (b) and (c), respectively. We can observe that the $\boldsymbol{h}^{-\frac{1}{2}}_{\infty}$ in AdaGrad are different for 100 runs and $\boldsymbol{h}^{-\frac{1}{2}}_{\infty}$ in RMSProp and Adam (w/m) are coincide. It indicates that $\boldsymbol{h}^{-\frac{1}{2}}_{\infty}$ in AdaGrad is sensitive to initialization. We also plot the value of the margin for the three algorithms under different initialization in Figure 2(d). We can observe that the margin of AdaGrad fluctuates under different initialization, while that for RMSProp and Adam (w/m) are smoother. We further show the relation between $\boldsymbol{h}^{-\frac{1}{2}}_{\infty}$ and the convergent direction of parameters in Appendix \ref{sec:expe_toy}. These results indicate that the convergent direction of AdaGrad is sensitive to initialization, which may hurt its generalization. }

\section{Conclusion}
In this paper, we study the convergent direction of both continuous and discrete cases of adaptive optimization algorithms on homogeneous deep neural networks. We prove that RMSProp and Adam (w/m) will converge to the KKT points of the $L^2$ max-margin problem, while AdaGrad does not. The main technical contribution of this paper is to propose a general framework for analyses of adaptive optimization algorithms' convergent direction. In future, we will study how optimization techniques such as momentum, weight decay and stochastic noise in optimization algorithm influence the convergent direction.

\bibliography{references}
\bibliographystyle{icml2021}
    
    



\clearpage

\onecolumn

\input{Appendix}
\end{document}

%% file: Appendix.tex
\appendix
\section{Preliminaries}
In this section, we provide some definitions and basic lemmas which will be used in the proof. The section is organized as follows: in Subsection \ref{sec:property_exponential}, we show general properties which exponential loss and logistic loss share; in Subsection \ref{sec:appen_KKT}, (approximate) KKT conditions is defined and sufficient conditions of being an approximate KKT point
is given; in Subsection \ref{sec:appen_explanation_conditoner}, we show how conditioners of AdaGrad, RMSProp, and Adam in continuous flow is formulated; in Subsection \ref{sec:appen_o_minimal}, we introduce o-Minimal structure, definable set and definable functions, and show two  Kurdyka-Lojasiewicz inequalities; in Subsection \ref{sec;appen_measure}, we show some basic definitions from Measure Theory, including
measurable set and Lebesgue Integrability.

\subsection{Property of Exponential and Logistic Loss}
\label{sec:property_exponential}
In this subsection, we provide several properties which both exponential and logistic loss possess. The properties of exponential and logistic loss can be described as the following proposition:
\begin{proposition}\label{lem:property_loss} 
For $\ell\in \{\ell_{exp},\ell_{\log}\}$:
\begin{itemize}
    \item There exists a $C^1$ function $f$, such that $\ell=e^{-f}$;
    \item For any $x\in \mathbb{R}$, $f'>0$. Therefore, $f$ is reversible, and $f^{-1}\in C^1$;
     \item $f'(x)x$ is non-decreasing for $x\in(0,\infty)$, and $\lim_{x\rightarrow\infty} f'(x)x=\infty$;
    \item There exists a large enough $x_f$ and a constant  $K\ge 1$ ,such that, 
    \begin{itemize}
        \item $\forall \theta \in[\frac{1}{2},1),$ $\forall x\in (x_f,\infty)$, and $\forall y\in f^{-1}(x_f,\infty)$: $(f^{-1})'(x)\le K (f^{-1})'(\theta x)$ and $f'(y)\le Kf'(\theta y)$;
        \item For all $y\in [x_f,\infty)$, $\frac{f(x)}{f'(x)}\in[\frac{1}{2K}x,2Kx]$;
        \item For all $x\in [f^{-1}(x_f),\infty)$, $\frac{f^{-1}(x)}{(f^{-1})'(x)}\in[\frac{1}{2K}x,2Kx]$.
        \item $f(x)=\Theta(x)$ as $x\rightarrow\infty$.
    \end{itemize}

\end{itemize}
\end{proposition}

All properties are easy to verify in Proposition \ref{lem:property_loss} and we omit it here. For brevity, we will use $g(x)=f^{-1}(x)$ in the following proofs.

\subsection{KKT Condition}
\label{sec:appen_KKT}
Being a KKT point is a first order necessary condition for being an optimal point.  We first give the definition of approximate KKT point for general optimization problem $(Q)$.

\begin{definition}
\label{def:general_KKT}
Consider the following optimization problem $(Q)$ for $\boldsymbol{x}\in \mathbb{R}^d$:
\begin{gather*}
\min f(\boldsymbol{x})\\
\text{subject to } g_i(x) \le 0, \forall i \in[N],  
\end{gather*}
where $f, g_i$ ($i=1,\cdots,N$): $\mathbb{R}^d
\rightarrow \mathbb{R}$ are locally Lipschitz functions. We say that $x\in \mathbb{R}^d$
is a feasible
point of $(P)$ if $x$ satisfies $g_i(x) \le 0$ for all $i \in [N]$.

For any $\varepsilon,\delta>0$, a feasible point of $(Q)$ is an $(\varepsilon,\delta)$-KKT point if there exists $\lambda_i\ge 0$, $\mathbf{k}\in {\partial} f(\boldsymbol{x})$, and $\boldsymbol{h}_{i} \in {\partial} g_{i}(\boldsymbol{x})$ for all $i\in [N]$ (we will slightly abuse ${\partial} f (\boldsymbol{x})$ to respresent a element in ${\partial} f (\boldsymbol{x})$) such that

1. $\left\|\boldsymbol{k}+\sum_{i \in[N]} \lambda_{i} \boldsymbol{h}_{i}(\boldsymbol{x})\right\|_{2} \leq \varepsilon$;

2. $\forall i \in[N]: \lambda_{i} g_{i}(\boldsymbol{x}) \geq-\delta$.

Specifically, when $\varepsilon=\delta=0$, we call $\boldsymbol{x}$ a KKT point of $(Q)$.
\end{definition}

The following Mangasarian-Fromovitz constraint qualification (MFCQ) bridges $(\varepsilon,\delta)$ KKT points with KKT points.

\begin{definition}
 A feasible point $\boldsymbol{x}$ of $(Q)$ is said to satisfy MFCQ  if there exists $\boldsymbol{a}\in \mathbb{R}^d$ such that for every $i\in[N]$ with $g_i(\boldsymbol{x})=0$,
\begin{equation*}
    \forall \boldsymbol{h} \in {\partial} g_{i}(\boldsymbol{x}):\langle\boldsymbol{h}, \boldsymbol{a}\rangle>0.
\end{equation*}
\end{definition}

MFCQ guarantees that the limit of approximate KKT point with convergent $\varepsilon$ and $\delta$ is a KKT point.

\begin{lemma}
\label{lem: appro_KKT_to_KKT}
Suppose for any $k\in\mathbb{N}$, $\boldsymbol{x}_k$ is a $(\varepsilon_k,\delta_k)$-KKT point of $(Q)$ defined in Definition \ref{def:general_KKT}. If $\lim_{k\rightarrow\infty}\varepsilon_k=0$,  $\lim_{k\rightarrow\infty}\delta_k=0$, and $\lim_{k\rightarrow\infty}\boldsymbol{x}_k=\boldsymbol{x}$, where the limit point $\boldsymbol{x}$ satisfies MFCQ, then $\boldsymbol{x}$ is a KKT point of $(Q)$. 
\end{lemma}

\subsection{How is the Continuous Form of Conditioner Formulated?}
\label{sec:appen_explanation_conditoner}
In this subsection, we show how conditioners of the continuous case for AdaGrad, RMSProp, Adam (w/m) are derived. Both discrete updates of these optimizers can be written as 
\begin{gather}
\label{eq:discrte_1}
    \boldsymbol{w}(t+1)-\boldsymbol{w}(t)\in -\eta\frac{1}{\sqrt{\varepsilon+\psi(\boldsymbol{m}(t),t)}} \odot \partial^s 
    \mathcal{L}(\boldsymbol{w}(t)),
    \\
    \label{eq:discrete_2}
    \boldsymbol{m}(t+1)-\boldsymbol{m}(t)=\phi(\boldsymbol{m}(t),\partial^s  \mathcal{L}(\boldsymbol{w}(t)) ),
    \\
    \nonumber
    \boldsymbol{m}(0)=\boldsymbol{0}.
\end{gather}
where $\partial^s \mathcal{L}(\boldsymbol{w}(t))\in \partial \mathcal{L}(\boldsymbol{w}(t)) $. For AdaGrad, $\phi(\boldsymbol{m}(t),\partial^s  \mathcal{L}(\boldsymbol{w}(t)) )=\partial^s  \mathcal{L}(\boldsymbol{w}(t))^2$, $\psi(\boldsymbol{m}(t),t)=\boldsymbol{m}(t)$; for RMSProp, $\phi(\boldsymbol{m}(t),\partial^s  \mathcal{L}(\boldsymbol{w}(t)) )=(1-b)(\partial^s\mathcal{L}(\boldsymbol{w}(t))^2-\boldsymbol{m}(t))$, $\psi(\boldsymbol{m}(t),t)=\boldsymbol{m}(t)$; for Adam (w/m), $\phi(\boldsymbol{m}(t),\partial^s  \mathcal{L}(\boldsymbol{w}(t)) )=(1-b)(\partial^s \mathcal{L}(\boldsymbol{w}(t))^2-\boldsymbol{m}(t))$, $\psi(\boldsymbol{m}(t),t)=\frac{\boldsymbol{m}(t)}{1-b^t}$.

One can easily observe that eqs. (\ref{eq:discrte_1}) and (\ref{eq:discrete_2}) is a discretization of the following equations:
\begin{gather}
\label{eq:continuous_1}
    \frac{\mathrm{d} \boldsymbol{w}(t)}{\mathrm{d} t}=-\frac{1}{\sqrt{\varepsilon+\psi(\boldsymbol{m}(t),t)}} \odot \partial^s 
    \mathcal{L}(\boldsymbol{w}(t)),
    \\
    \label{eq:continuous_2}
   \frac{\mathrm{d} \boldsymbol{m}(t)}{\mathrm{d} t}=\phi(\boldsymbol{m}(t),\partial^s  \mathcal{L}(\boldsymbol{w}(t)) ),
   \\
   \nonumber
   \boldsymbol{m}(0)=\boldsymbol{0}.
\end{gather}

As for AdaGrad, 
\begin{equation*}
    \frac{\mathrm{d} \boldsymbol{m}(t)}{\mathrm{d} t}=\partial^s\mathcal{L}(\boldsymbol{w}(t))^2,
\end{equation*}
which leads to 
\begin{equation*}
    \boldsymbol{m}(t)=\int_{0}^t \partial^s\mathcal{L}(\boldsymbol{w}(\tau))^2 \mathrm{d} \tau,
\end{equation*}
and 
\begin{equation*}
      \frac{\mathrm{d} \boldsymbol{w}(t)}{\mathrm{d} t}=-\frac{1}{\sqrt{\varepsilon+\int_{0}^t \partial^s\mathcal{L}(\boldsymbol{w}(\tau))^2 \mathrm{d} \tau}} \odot \partial^s
    \mathcal{L}(\boldsymbol{w}(t))
\end{equation*}
As for RMSProp and Adam
\begin{equation*}
     \frac{\mathrm{d}\boldsymbol{m}(t)}{\mathrm{d}t}=(1-b)(\partial^s\mathcal{L}(\boldsymbol{w}(t))^2- \boldsymbol{m}(t)).
\end{equation*}

By solving the above differential equation, we have 
\begin{equation*}
    \frac{\mathrm{d}e^{(1-b)t}\boldsymbol{m}(t)}{\mathrm{d}t}=e^{(1-b)t}(1-b)\partial^s\mathcal{L}(\boldsymbol{w}(t))^2,
\end{equation*}
which by integration implies 
\begin{equation*}
    \boldsymbol{m}(t)=\int_{0}^t e^{-(1-b)(t-\tau)}(1-b)\partial^s\mathcal{L}(\boldsymbol{w}(\tau))^2\mathrm{d} \tau.
\end{equation*}
 
 Therefore, for RMSProp, the continuous flow is 
\begin{equation*}
      \frac{\mathrm{d} \boldsymbol{w}(t)}{\mathrm{d} t}=-\frac{1}{\sqrt{\varepsilon+\int_{0}^t e^{-(1-b)(t-\tau)}(1-b)\partial^s\mathcal{L}(\boldsymbol{w}(\tau))^2\mathrm{d} \tau}} \odot \partial^s
    \mathcal{L}(\boldsymbol{w}(t));
\end{equation*}
while for Adam  (w/m), the continuous flow is 
\begin{equation*}
      \frac{\mathrm{d} \boldsymbol{w}(t)}{\mathrm{d} t}=-\frac{1}{\sqrt{\varepsilon+\frac{\int_{0}^t e^{-(1-b)(t-\tau)}(1-b)\partial^s\mathcal{L}(\boldsymbol{w}(\tau))^2\mathrm{d} \tau}{1-b^t}}} \odot \partial^s
    \mathcal{L}(\boldsymbol{w}(t)).
\end{equation*}

\subsection{o-Minimal Structure and Definable functions}
\label{sec:appen_o_minimal}
Here we define o-Minimal structure and definable functions which we omit in Theorem \ref{thm:approximate_flow_definable}. 

\begin{definition}[Appendix B, \citet{ji2020directional}]
 An o-minimal structure is a collection 
$\mathcal{S}=\{\mathcal{S}_n\}_{n=1}^{\infty}$, where each $\mathcal{S}_n$ is a set of subsets of $\mathbb{R}_n$ satisfying the
following conditions:
1. $\mathcal{S}_1$ is the collection of all finite unions of open intervals and points;\\
2. $\mathcal{S}_n$ includes the zero sets of all polynomials on $\mathbb{R}_n$;\\
3. $\mathcal{S}_n$ is closed under finite union, finite intersection, and complement;\\
4. $\mathcal{S}$ is closed under Cartesian products: if $A\in \mathcal{S}_m$  and $B \in \mathcal{S}_n$, then $A \times B \in \mathcal{S}_{m+n}$;\\
5. $\mathcal{S}$ is closed under projection $\Pi_n$ onto the first $n$ coordinates: if $A \in \mathcal{S}_{n+1}$, then $\Pi_n(A) \in \mathcal{S}_n$.
\end{definition}

A definable function on above o-Minimal Structure can be defined as follows:

\begin{definition}[Appendix B, \citet{ji2020directional}]
 A function $f:D\rightarrow\mathbb{R}^m$ with $D\subset \mathbb{R}^n$ is definable if the graph of $f$ is in $\mathcal{S}_{n+m}$.
\end{definition}

A natural question is: which function is definable? The next Lemma helps to solve this question.

\begin{lemma}[Lemma B.2, \citet{ji2020directional}]
\quad
\begin{itemize}
    \item All polynomials are definable, therefore, linear or other polynomial activation is definable;
    \item If both $f(x)$ and $g(x)$ are definable, $\min{f(x),g(x)}$ and $\max{f(x),g(x)}$ are definable, therefore, ReLU activation is definable;
    \item If $f_i(x):D\rightarrow\mathbb{R}$ ($i=1,2,\cdots,n$) is definable, then $\boldsymbol{f}(x)=(f_1(x),f_2(x),\cdots,f_n(x))$ is definable.
    \item Suppose there exists $k,d_0,d_1,\cdots,d_L>0$, and $L$ definable functions $(g_1,g_2,\cdots,g_L)$, where $g_j:\rightarrow\mathbb{R}^{d_{0}} \times \cdots \times \mathbb{R}^{d_{j-1}} \times \mathbb{R}^{k} \rightarrow \mathbb{R}^{d_{j}}$. Let $h_1(x,W)\overset{\triangle}{=}g_1(x,W)$, and for $2\le j\le L$,
    \begin{equation*}
        h_{j}(x, W):=g_{j}\left(x, h_{1}(x, W), \ldots, h_{j-1}(x, W), W\right),
    \end{equation*} then all $h_j$ are definable. Therefore, neural networks with polynomial and ReLU activation, convolutional and max-pooling layers, and skip connections are definable.
\end{itemize}
\end{lemma}

An important property for definable function is Kurdyka-Lojasiewicz inequality, which can bound gradient of definable function in a small region. Here we present two Kurdyka-Lojasiewicz inequalities  given by \cite{ji2020directional}:
\begin{lemma}[Lemma 3.6, \citet{ji2020directional}]
\label{lem:construction_of_psi_conditional}
Given a locally Lipschitz definable function $f$ with an open domain $D\in\{x|\Vert x\Vert>1\}$, for any $c$, $\eta>0$, there exists $a>0$ and a definable desingularizing function $\Psi$ on $[0,a)$ (that is, $\Psi(x)\in C^1((0,a))\cap C^0([0,a))$ with $\Psi(0)=0$), such that,
\begin{equation*}
    \Psi^{\prime}(f(x))\|x\|\|\bar{\partial} f(x)\| \geq 1, \text{ if } f(x)\in(0,a), \text{ and } \left\|\bar{\partial}_{\perp} f(x)\right\| \geq c\|x\|^{\eta}\left\|\bar{\partial}_{ \backslash\backslash} f(x)\right\|,
\end{equation*}
where $\bar{\partial} f(x)$ is the unique one with the smallest norm in $\partial f(x)$,  $\bar{\partial}_{\backslash\backslash} f(x)$ is the projection of $\bar{\partial} f(x)$ to $x$ and $\bar{\partial}_{ \perp} f(x)=\bar{\partial} f(x)-\bar{\partial}_{ \backslash\backslash} f(x)$ is the remaining term 

\end{lemma}

\begin{lemma}[Lemma 3.7, \citet{ji2020directional}]
\label{lem:construction_of_psi_no_condition}
Given a locally Lipschitz definable function $f$ with an open domain $D\subset \{x|\Vert x\Vert>1\}$, for any $\lambda>0$, there exists $a>0$, and a definable desingularizing function $\Psi$ on $[0,a)$ such that 
\begin{equation*}
    \max \left\{1, \frac{2}{\lambda}\right\} \Psi^{\prime}(f(x))\|x\|^{1+\lambda}\|\bar{\partial} f(x)\| \geq 1, \text{ if } f(x)\in(0,a).
\end{equation*}
\end{lemma}

At the end of this subsection, we show that definability actually guarantees that $\Phi$ admits a chain rule, which is formally stated as following:
\begin{lemma}[Lemma B.9, \citet{ji2020directional}]
\label{lem:chain_rule}
Given a locally Lipschitz definable $f:D\rightarrow\mathbb{R}$ with an open domain $D$, for any interval $I$ and any arc $z:I\rightarrow D$, it holds for a.e. $t\in I$ that
\begin{equation*}
    \frac{\mathrm{d} f\left(z_{t}\right)}{\mathrm{d} t}=\left\langle z_{t}^{*}, \frac{\mathrm{d} z_{t}}{\mathrm{d} t}\right\rangle, \quad \text { for all } z_{t}^{*} \in \partial f\left(z_{t}\right).
\end{equation*}
\end{lemma}

Therefore, if we are deal with definable neural networks $\Phi$ as in Theorem \ref{thm:approximate_flow_definable}, we no longer need to assume $\Phi$ admits a chain rule which is already guaranteed by lemma \ref{lem:chain_rule}.

\subsection{Discussion of the influence of initialization on the solution of $P^{A}$}
\label{appen: h_infty}
For AdaGrad, $\boldsymbol{h}^{-2}_{\infty}$ is defined as $\varepsilon \mathbf{1}_p+\sum_{t=0}^{\infty} \nabla \mathcal{L}(\boldsymbol{w}(t))^2$, which is the sum of squared gradients along the trajectory. Intuitively, as the initialization changes, the trajectory changes respectively, and so does the direction of $\boldsymbol{h}_{\infty}$. This intuition can be further verified by Experiment in Section 6.2, where we plot the direction of $\boldsymbol{h}^{-\frac{1}{2}}_{\infty}$ as the initialization changes.
\begin{figure}
\centering
\includegraphics[scale=0.3]{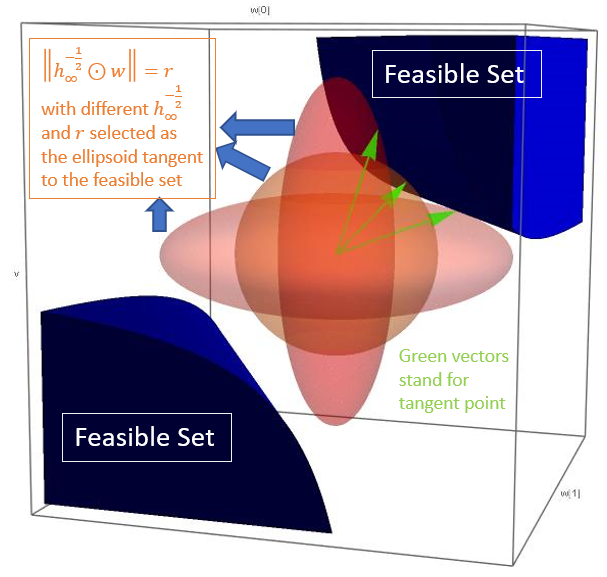}
\caption{How $\boldsymbol{h}_{\infty}$ influence convergence direction of $\boldsymbol{w}$}
\label{fig:figure}
\end{figure}
Furthermore, how $\boldsymbol{h}_{\infty}$ influence the max-margin problem can be interpreted as follows: optimizing $\Vert \boldsymbol{h}_{\infty}^{-\frac{1}{2}} \odot \boldsymbol{w}\Vert^2$ with constraints is equivalent to find the radius $r$ of ellipsoid $\Vert \boldsymbol{h}_{\infty}^{-\frac{1}{2}} \odot \boldsymbol{w}\Vert^2=r^2$ when the   ellipsoid is tangent to the feasible set. This intuition is visualized in Figure \ref{fig:figure}. One can easily observe that as the direction of $\boldsymbol{h}_{\infty}$ changes, the direction of the tangent point changes.
\subsection{Basic knowledge from Measure Theory}
\label{sec;appen_measure}
In this section, we present basic definitions of measurable set, measurable functions and Lebesgue Integrability. These definition involves use of exterior measure and Borel set in Euclidean space, which we omit them here. Readers interested in measure theory can refer to \cite{stein2009real} for details.
\begin{definition}[\citet{stein2009real}, Chapter 1, page 16] 
A subset $E$ of $\mathbb{R}^d$is Lesbesgue measurable, or simply measurable, if for any $\varepsilon>0$, there exists an open set $G$, with $E\subset G$, and
\begin{equation*}
    m_{*}(G/E)\le \varepsilon,
\end{equation*}
where $m_{*}$ is the exterior measure on $\mathbb{R}^d$.

\begin{definition}[\citet{stein2009real}, Chapter 1, page 28] 
A function $f$ on a measurable subset $E$ of $\mathbb{R}^d$ is measurable if for all $a\in \mathbb{R}$, the set 
\begin{equation*}
    f^{-1}([-\infty,a))=\{x\in E: f(x)<a\}
\end{equation*}
is measurable.
\end{definition}

\begin{definition}[\citet{stein2009real}, Chapter 1, page 64]
A measurable function $f$ defined on a measurable subset of $\mathbb{R}^d$ is Lesbesgue integrable if 
\begin{equation*}
    \int_{E}\vert f(x) \vert\mathrm{d} m(x)< \infty.
\end{equation*}
\end{definition}
 
\end{definition}

\section{Proof of Results for Adaptive Algorithms in Continuous Case}
\label{section:appen_continuous}

This section collects proof of Theorem \ref{thm:approximate_flow}, Theorem \ref{lem:adagrad_approximate},  Theorem \ref{lem:rms_approximate}, and also contains proof of Theorem \ref{thm:AdaGrad_flow} and Theorem \ref{thm:RMS_flow}. Organization of this section is as follows: In Subsection \ref{sec:appen_approximate}, we present proof of Theorems \ref{lem:adagrad_approximate} and Theorem \ref{lem:rms_approximate}; in Subsection \ref{sec:appen_proof_approximate}, we present proof of Theorem \ref{thm:approximate_flow} based on the proof skeleton in Section \ref{sec:main_result_proof}; in Subsection \ref{sec:appen_direction_adagrad_rmsp}, we prove  Theorem \ref{thm:AdaGrad_flow} and Theorem \ref{thm:RMS_flow} based on \ref{thm:approximate_flow}, Theorem \ref{lem:adagrad_approximate},  Theorem \ref{lem:rms_approximate}; finally, in Subsection \ref{sec:appen_convergence_rate}, we provide tight convergence rate of loss and parameter norm in adaptive gradient flows.

\subsection{Proof of Theorem  \ref{lem:adagrad_approximate} and Theorem \ref{lem:rms_approximate}: Transition from Continuous Adaptive Algorithms to Adaptive Gradient Flow}
\label{sec:appen_approximate}

\subsubsection{Proof of Theorem \ref{lem:adagrad_approximate}}

The proof of Theorem \ref{lem:adagrad_approximate} is divided into two stages: we first prove convergence of $\boldsymbol{h}^A(t)$ and $\boldsymbol{\beta}^A(t)$; then we show $\frac{\mathrm{d} \boldsymbol{v}^A(t)}{\mathrm{d}t}=-\boldsymbol{\beta}^A(t)\odot \partial^s \tilde{\mathcal{L}}^A(\boldsymbol{v}^A(t))$ satisfies adaptive gradient flow, and is equivalent to AdaGrad flow.

We first show $\int_{0}^{\infty} \partial^s \mathcal{L}(\boldsymbol{w}(t))^2  \mathrm{d} \tau$ is bounded.

\begin{lemma}
\label{lem:adagrad_flow_rate_conver}
For AdaGrad flow defined as eq. (\ref{eq:flow}) with $\boldsymbol{h}=\boldsymbol{h}^A$,  
\begin{equation*}
    \int_{0}^{\infty} (\partial^s \mathcal{L}(\boldsymbol{w}(\tau)))_i^2  \mathrm{d} \tau< \infty, i=1, \cdots, p.
\end{equation*}
\end{lemma}
\begin{proof}
We use reduction of absurdity. If there exists an $i$, such that, $\int_{0}^{\infty} (\partial^s \mathcal{L}(\boldsymbol{w}(\tau)))_i^2  \mathrm{d} \tau$ diverges, by equivalence of integral convergence, then
\begin{equation*}
    \int_{0}^{\infty} \frac{(\partial^s \mathcal{L}(\boldsymbol{w}(t)))_i^2 }{\varepsilon+\int_{0}^{t} (\partial^s \mathcal{L}(\boldsymbol{w}(\tau)))_i^2   \mathrm{d} \tau}  \mathrm{d}t = \infty.
\end{equation*}

Since $\int_{0}^{\infty}(\partial^s \mathcal{L}(\boldsymbol{w}(\tau)))_i^2   \mathrm{d} \tau=\infty$, when $t$ is large enough, 
\begin{equation*}
    \varepsilon+\int_{0}^{t} (\partial^s \mathcal{L}(\boldsymbol{w}(\tau)))_i^2  \mathrm{d} \tau>\sqrt{\varepsilon+\int_{0}^{t} (\partial^s \mathcal{L}(\boldsymbol{w}(\tau)))_i^2 \mathrm{d} \tau}  .
\end{equation*}

Therefore,
\begin{equation*}
\int_{0}^{\infty} \frac{(\partial^s \mathcal{L}(\boldsymbol{w}(t)))_i^2 }{\sqrt{\varepsilon+\int_{0}^{t} (\partial^s \mathcal{L}(\boldsymbol{w}(\tau)))_i^2  \mathrm{d} \tau}}  \mathrm{d}t = \infty.
\end{equation*}

By integrating $\frac{\mathrm{d} \mathcal{L}(\boldsymbol{w}(t))}{\mathrm{d} t}$,
\begin{align*}
     \mathcal{L}(\boldsymbol{w}(0))-\mathcal{L}(\boldsymbol{w}(t))&=-\int_{0}^{t}  \frac{\mathrm{d} \mathcal{L}(\boldsymbol{w}(\tau))}{\mathrm{d} \tau} \mathrm{d}\tau
     \\
     &=-\int_{0}^{t} \left\langle \partial^s \mathcal{L}(\boldsymbol{w}(\tau)) ,\frac{\mathrm{d} \boldsymbol{w}(\tau)}{\mathrm{d}\tau} \right\rangle \mathrm{d}\tau
     \\
     &=\int_{0}^{t} \langle \partial^s \mathcal{L}(\boldsymbol{w}(\tau)) ,\boldsymbol{h}^A(\tau)\odot\partial^s \mathcal{L}(\boldsymbol{w}(\tau)) \rangle \mathrm{d}\tau
     \\
     &=\sum_{i=1}^p\int_{0}^{t}  \frac{(\partial^s \mathcal{L}(\boldsymbol{w}(s)))_i^2 }{\sqrt{\varepsilon+\int_{0}^{t}(\partial^s \mathcal{L}(\boldsymbol{w}(\tau)))_i^2   \mathrm{d} \tau}} \mathrm{d}s
     \\
     &=\infty,
\end{align*}

 which leads to a contradictory, since $\mathcal{L}(\boldsymbol{\omega}(0))-\mathcal{L}(\boldsymbol{\omega}(t))$ is upper  bounded by $\mathcal{L}(\boldsymbol{\omega}(0))$.

The proof is completed.
\end{proof}

Now we are ready to prove Theorem \ref{lem:adagrad_approximate}.

\begin{theorem}[Theorem \ref{lem:adagrad_approximate}, restated]
\label{lem:construct_v^A}
Define $\boldsymbol{h}_{\infty}=\lim_{t\rightarrow \infty} \boldsymbol{h}^A(t)$. Then $\boldsymbol{h}_{\infty}$ has no zero elements. Let
\begin{gather*}
    \boldsymbol{v}^A(t)=\boldsymbol{h}_{\infty}^{-1 / 2} \odot \boldsymbol{w}(t),\\
    \boldsymbol{\beta}^A(t)=\boldsymbol{h}_{\infty}^{-1} \odot \boldsymbol{h}^A(t),\\
    \tilde{\mathcal{L}}^A(\boldsymbol{v})=\mathcal{L}(\boldsymbol{h}_{\infty}^{\frac{1}{2}}\odot \boldsymbol{v}). 
\end{gather*}
We have 
\begin{equation}
\label{eq:adagrad_normalized_flow}
    \frac{\mathrm{d}\boldsymbol{v}^A(t)}{\mathrm{d} t}=- \boldsymbol{\beta}^A(t) \odot \partial^s \tilde{\mathcal{L}}^A(\boldsymbol{v}^A(t)),
\end{equation}
while $\lim_{t\rightarrow\infty} \boldsymbol{\beta}^A(t)=1$, and $\frac{\mathrm{d}\log \boldsymbol{\beta}^A(t)}{\mathrm{d}t}$ is Lebesgue integrable.

\end{theorem}
\begin{proof}
Since 
\begin{equation*}
    \boldsymbol{h}_{\infty}=\frac{1}{\sqrt{\varepsilon\mathbf{1}_p+\int_{0}^{\infty}(\partial^s \mathcal{L}(\boldsymbol{w}(\tau)))_i^2   \mathrm{d} \tau}},
\end{equation*}
by Lemma \ref{lem:adagrad_flow_rate_conver}, $\boldsymbol{h}_{\infty}$ has no zero elements.

We then prove eq. (\ref{eq:adagrad_normalized_flow}) by direct calculation.
\begin{align*}
    \frac{\mathrm{d}\boldsymbol{v}^A(t)}{\mathrm{d} t}&=\boldsymbol{h}_{\infty}^{-1 / 2} \odot\frac{d\boldsymbol{w}(t)}{\mathrm{d} t}
    \\
    &=-\boldsymbol{h}_{\infty}^{-1 / 2}\odot\boldsymbol{h}^A(t)\odot \partial^s \mathcal{L}(\boldsymbol{w}(t)) 
    \\
    &=-\boldsymbol{h}_{\infty}^{-1 / 2}\odot\boldsymbol{h}^A(t)\odot\boldsymbol{h}_{\infty}^{-1 / 2}\odot  \boldsymbol{h}_{\infty}^{1 / 2}\odot\partial^s_{\boldsymbol{w}} \tilde{\mathcal{L}}^A( \boldsymbol{h}_{\infty}^{-1 / 2}\odot\boldsymbol{w}(t)) 
    \\
    &=-\boldsymbol{\beta}^A(t)\odot  \partial^s_{\boldsymbol{h}_{\infty}^{-1 / 2}\odot\boldsymbol{w}(t)} \tilde{\mathcal{L}}^A( \boldsymbol{h}_{\infty}^{-1 / 2}\odot\boldsymbol{w}(t))
    \\
    &=-\boldsymbol{\beta}^A(t)\odot  \partial^s_{\boldsymbol{v}(t)} \tilde{\mathcal{L}}^A( \boldsymbol{v}(t)),
\end{align*}
which completes the proof of eq. ( \ref{eq:adagrad_normalized_flow}).

$\lim_{t\rightarrow\infty} \boldsymbol{\beta}^A(t)=1$ can be derived directly by convergence of $\boldsymbol{h}^A(t)$, while since 
\begin{equation*}
    \frac{\mathrm{d}\log \boldsymbol{\beta}^A(t)}{\mathrm{d}t}= \frac{\mathrm{d}\log \boldsymbol{h}^A(t)}{\mathrm{d}t}\overset{(*)}{\le} \mathbf{0},
\end{equation*}
where inequality $(*)$ is due to $\boldsymbol{h}^A$ is non-increasing.

Therefore, 
\begin{equation*}
    \int_{0}^{\infty}\left\vert\frac{\mathrm{d}\log \boldsymbol{\beta}^A(t)}{\mathrm{d}t}\right\vert \mathrm{d}t= -\int_{0}^{\infty}\frac{\mathrm{d}\log \boldsymbol{\beta}^A(t)}{\mathrm{d}t} \mathrm{d}t=\log \boldsymbol{\beta}^A(0)<\infty.
\end{equation*}

The proof is completed.
\end{proof}

\subsubsection{Proof of Theorem \ref{lem:rms_approximate}}
We first prove Theorem \ref{lem:rms_approximate} for RMSProp, and then extend the proof for Adam (w/m) and other Adam-like optimizers. The proof strategy is similar with AdaGrad:  we first prove convergence of $\boldsymbol{h}^R(t)$ and integrability $\frac{\mathrm{d}\log \boldsymbol{\beta}^R(t)}{\mathrm{d} t}$; then we show $\frac{\mathrm{d} \boldsymbol{v}^R(t)}{\mathrm{d}t}=-\boldsymbol{\beta}^R(t)\odot \partial^s \tilde{\mathcal{L}}^R(\boldsymbol{v}^R(t))$ is equivalent to RMSProp flow, and satisfies adaptive gradient flow.

However, for RMSProp flow, the convergence of $\boldsymbol{h}^R$ requires more effort. We start from the following lemma, which bounds $F_i(t)\overset{\triangle}{=}\int_{0}^{t} (1-b) b^{t-\tau} (\partial^s \mathcal{L}(\boldsymbol{w}(\tau)))_i^2  \mathrm{d} \tau$ ($i\in [p]$).
\begin{lemma}
\label{lem:boundedness_for_F}
For RMSProp flow defined as eq. (\ref{eq:flow}) with $\boldsymbol{h}=\boldsymbol{h}^R$ , $F_i(t)$ is bounded ($i=1,2,\cdots,p$), that is,
\begin{equation*}
    \overline{\lim}_{t\rightarrow \infty} F_i(t)< \infty.
\end{equation*}
\end{lemma}
\begin{proof}
When $b=1$, $F_i(t)=0$ for all $t$ and $i$, which trivially yields the claim. When $b\ne 1$, we use reduction of absurdity. If there exists an $i$, such that, $\overline{\lim}_{t\rightarrow \infty} F_i(t)= \infty$, then $t_k\overset{\triangle}{=}\inf \{t: F_i(t)\ge k\}<\infty$ holds. Furthermore, since $\mathcal{L}$ is locally Lipschitz with respect to $\boldsymbol{w}$, $g_i(t)$ is locally bounded for any $t$, which leads to the absolute continuity of $F_i$. Therefore, since $F_i(0)=0$, $t_k$ monotonously increases.

Therefore, we have that
\begin{align*}
    \int_{t_{k}}^{t_{k+1}} \frac{(\partial^s \mathcal{L}(\boldsymbol{w}(t)))_i^2}{\sqrt{\varepsilon+F_i(t)}} \mathrm{d}t&\ge\int_{t_{k}}^{t_{k+1}} \frac{(\partial^s \mathcal{L}(\boldsymbol{w}(t)))_i^2}{\sqrt{\varepsilon+k+1}} \mathrm{d}t
    \\
    &=\frac{1}{\sqrt{\varepsilon+k+1}}\int_{t_{k}}^{t_{k+1}} (\partial^s \mathcal{L}(\boldsymbol{w}(t)))_i^2 \mathrm{d}t
    \\
    &\ge \frac{1}{\sqrt{\varepsilon+k+1}}\int_{t_{k}}^{t_{k+1}} (1-b)e^{-(1-b)(t_{k+1}-t)} (\partial^s \mathcal{L}(\boldsymbol{w}(t)))_i^2 \mathrm{d}t
    \\
    &= \frac{1-b}{\sqrt{\varepsilon+k+1}} (F_i(t_{k+1})-e^{-(1-b)(t_{k+1}-t_k)} F_i(t_{k}))
    \\
    &\ge \frac{1-b}{\sqrt{\varepsilon+k+1}} (k+1-e^{-(1-b)(t_{k+1}-t_k)}k)
    \\
    &\ge \frac{1-b}{\sqrt{\varepsilon+k+1}}.
\end{align*}

Adding all $k\ge 1$, we then have 
 \begin{equation*}
     \mathcal{L}(t_1)\ge\int_{t_{1}}^{\infty} \frac{(\partial^s \mathcal{L}(\boldsymbol{w}(t)))_i^2}{\sqrt{\varepsilon+F_i(t)}} \mathrm{d}t =\infty,
 \end{equation*}
 which leads to a contradictory.
 
 The proof is completed.
\end{proof}

The next lemma shows that $\int_{0}^{\infty}(\partial^s \mathcal{L}(\boldsymbol{w}(t)))_i^2\mathrm{d}t$ converges, which indicates that $\lim_{t\rightarrow\infty} F_i(t)=0$.

\begin{lemma}
\label{lem:F_go_zero}
For RMSProp flow defined as eq. (\ref{eq:flow}) with $\boldsymbol{h}=\boldsymbol{h}^R$ , $\int_{0}^{\infty}(\partial^s \mathcal{L}(\boldsymbol{w}(t)))_i^2\mathrm{d}t$ converges, which indicates that $\lim_{t\rightarrow\infty} F_i(t)=0$ ($i=1,2,\cdots,p$). Consequently,  $\lim_{t\rightarrow\infty} F_i(t)=0$ and $\lim_{t\rightarrow\infty} \boldsymbol{h}^R(t)=\frac{1}{\sqrt{\varepsilon}} \mathbf{1}_p$.
\end{lemma}
\begin{proof}
Similar to Lemma \ref{lem:boundedness_for_F}, when $b=1$, the claim trivially holds. When $b\ne1$, by Lemma \ref{lem:boundedness_for_F}, there exist $M_i>0$ $(i=1,2,\cdots, p)$, such that,
$F_i(t)\le M_i$ for any $t>0$. Therefore,
\begin{equation*}
    \mathcal{L}(0)\ge\sum_{i=1}^p\int_{0}^{\infty} \frac{(\partial^s \mathcal{L}(\boldsymbol{w}(t)))_i^2}{\sqrt{\varepsilon+F_i(t)}} \mathrm{d}t\ge \sum_{i=1}^p\int_{0}^{\infty} \frac{(\partial^s \mathcal{L}(\boldsymbol{w}(t)))_i^2}{\sqrt{M_i+\varepsilon}} \mathrm{d}t,
\end{equation*}
which proves $\int_{0}^{\infty} (\partial^s \mathcal{L}(\boldsymbol{w}(t)))_i^2 \mathrm{d}t<\infty$.

Therefore, for any positive real $\varepsilon>0$ and a fixed index $i\in[p]$, there exists a time $T$, such that,
\begin{equation*}
    \int_{T}^{\infty} (\partial^s \mathcal{L}(\boldsymbol{w}(t)))_i^2 \mathrm{d}t \le \varepsilon.
\end{equation*}

Thus, for any $t\ge T$,
\begin{equation*}
    F_i(t)= e^{-(t-T)(1-b)} F_i(T)+ \int_{T}^{\infty} (\partial^s \mathcal{L}(\boldsymbol{w}(t)))_i^2 \mathrm{d}t\le e^{-(t-T)(1-b)} F_i(T)+ \varepsilon,
\end{equation*}
which leads to
\begin{equation*}
    \overline{\lim}_{t\rightarrow\infty} F_i(t)\le \varepsilon.
\end{equation*}
Since $\varepsilon$ and $i$ can be picked arbitrarily, the proof is completed.
\end{proof}

Similar to the proof of Lemma \ref{lem:construct_v^A},  we can rewrite the RMSProp flow as
\begin{equation*}
    \frac{\mathrm{d}\boldsymbol{v}^R(t)}{\mathrm{d} t}=- \boldsymbol{\beta}^R(t) \odot \partial^s \tilde{\mathcal{L}}^R(\boldsymbol{v}^R(t)),
\end{equation*}
where 
\begin{gather*}
    \boldsymbol{v}^R(t)=\sqrt[4]{\varepsilon}\boldsymbol{w}(t),\\
    \boldsymbol{\beta}^R(t)=\sqrt{\varepsilon}  \boldsymbol{h}(t),\\
   \tilde{ \mathcal{L}}^R(\boldsymbol{v})=\mathcal{L}(\sqrt[4]{\varepsilon^{-1}} \boldsymbol{v}),
\end{gather*}
and $\lim_{t\rightarrow\infty} \boldsymbol{\beta}^R(t)=\mathbf{1}_p$. We only need to prove $\frac{d \log \boldsymbol{\beta}^R(t)}{\mathrm{d}t}$ is  Lebesgue integrable to complete the proof of Theorem \ref{lem:rms_approximate}.

\begin{proof}[Proof of Theorem \ref{lem:rms_approximate} for RMSProp]
For any fixed $i=1,2,\cdots,p$, by Lemma \ref{lem:F_go_zero},
\begin{equation*}
     \int_{0}^{\infty} \frac{\mathrm{d} \log \boldsymbol{\beta}^R (t)}{\mathrm{d}t} \mathrm{d}t=-\int_{0}^{\infty} \frac{\mathrm{d} \log \sqrt{\frac{\varepsilon+F_i(t)}{\varepsilon}}}{\mathrm{d}t} \mathrm{d}t=0.
\end{equation*}
Therefore,
\begin{equation*}
     0=\int_{0}^{\infty} \frac{d \log \sqrt{\frac{\varepsilon+F_i(t)}{\varepsilon}}}{\mathrm{d}t} \mathrm{d}t=\int_{0}^{\infty} \left(\frac{\mathrm{d} \log \boldsymbol{\beta}^R (t)}{\mathrm{d}t}\right)_{+} \mathrm{d}t+\int_{0}^{\infty} \left(\frac{\mathrm{d} \log \boldsymbol{\beta}^R (t)}{\mathrm{d}t}\right)_{-} \mathrm{d}t,
\end{equation*}
we only need to prove the convergence of $\int_{0}^{\infty} \left(\frac{d \log \sqrt{\frac{\varepsilon+F_i(t)}{\varepsilon}}}{\mathrm{d}t}\right)_{-} \mathrm{d}t$, is equivalent to convergence of $\int_{0}^{\infty} (\frac{\mathrm{d} \log \sqrt{\frac{\varepsilon+F_i(t)}{\varepsilon}}}{\mathrm{d}t})_{+} \mathrm{d}t$.

For any $k\in \mathcal{Z}^{+}$, denote $B_k=(k,k+1)\cap \{t:\frac{d \log \sqrt{\frac{\varepsilon+F_i(t)}{\varepsilon}}}{\mathrm{d}t}>0\}$. Since $\frac{d \log \sqrt{\frac{\varepsilon+F_i(t)}{\varepsilon}}}{\mathrm{d}t}$ is measurable, we have that $B_k$ is a measurable set. Denote $M_k$ as an upper bound for $\frac{d \log \sqrt{\frac{\varepsilon+F_i(t)}{\varepsilon}}}{\mathrm{d}t}$ on $[k,k+1)$ (which is guaranteed since $\mathcal{L}$ is locally Lipschitz). Since $B_k$ is a measurable set, there exists an open set $\tilde{B}_k\subset(k,k+1)$, such that $m(\tilde{B}_k/B_k)\le \frac{1}{k^2 M_k}$, where $m$ is Lebesgue measure on $\mathbb{R}$.

Therefore, we have
\begin{align*}
    \int_{t\in [k,k+1)} \left(\frac{d \log \sqrt{\frac{\varepsilon+F_i(t)}{\varepsilon}}}{\mathrm{d}t}\right)_{+} \mathrm{d}t
    = & \int_{t\in B_k} \left(\frac{d \log \sqrt{\frac{\varepsilon+F_i(t)}{\varepsilon}}}{\mathrm{d}t}\right) \mathrm{d}t
    \\
    \le &\int_{t\in \tilde{B}_k} \left(\frac{d \log \sqrt{\frac{\varepsilon+F_i(t)}{\varepsilon}}}{\mathrm{d}t}\right) \mathrm{d}t +M_k\cdot \frac{1}{k^2 M_k}
    \\
    \\
    = &\int_{t\in \tilde{B}_k} \left(\frac{d \log \sqrt{\frac{\varepsilon+F_i(t)}{\varepsilon}}}{\mathrm{d}t}\right) \mathrm{d}t + \frac{1}{k^2 }.
\end{align*}

Furthermore, let $\tilde{B}_k=\cup_{j=1}^{\infty} (t_{k,j},t'_{k,j})$. Then we have
\begin{align*}
    &\int_{t\in \tilde{B}_k}\left(\frac{d \log \sqrt{\frac{\varepsilon+F_i(t)}{\varepsilon}}}{\mathrm{d}t}\right) \mathrm{d}t
    \\
    =&\sum_{j=1}^\infty  \int_{t_{k,j}}^{t'_{k,j}} \frac{d \log \sqrt{\frac{\varepsilon+lF_i(t)}{\varepsilon}}}{\mathrm{d}t} \mathrm{d}t
    \\
    =&\sum_{j=1}^\infty  \left( \log \sqrt{\frac{\varepsilon+F_i(t_{k,j}')}{\varepsilon}}-\log \sqrt{\frac{\varepsilon+F_i(t_{k,j})}{\varepsilon}}\right)
    \\
    =&\frac{1}{2} \sum_{j=1}^\infty   \log \frac{\varepsilon+F_i(t_{k,j}')}{\varepsilon+F_i(t_{k,j})}
    \\
    = &\frac{1}{2} \sum_{j=1}^\infty   \log \frac{\varepsilon+e^{-(t_{k,j}'-t_{k,j})(1-b)}F_i(t_{k,j})+\int_{t_{k,j}}^{t_{k,j}'} (1-b)e^{-(t_{k,j}'-t)(1-b)}(\partial^s \mathcal{L}(\boldsymbol{w}(t)))_i^2 \mathrm{d}t}{\varepsilon+F_i(t_{k,j})}
    \\
    \le&\frac{1}{2} \sum_{j=1}^\infty   \log \frac{\varepsilon+F_i(t_{k,j})+\int_{t_{k,j}}^{t_{k,j}'} (1-b)e^{-(t_{k,j}'-t)(1-b)}(\partial^s \mathcal{L}(\boldsymbol{w}(t)))_i^2 \mathrm{d}t}{\varepsilon+F_i(t_{k,j})}
    \\
    \le &\frac{1}{2} \sum_{j=1}^\infty \frac{\int_{t_{k,j}}^{t_{k,j}'} (1-b)e^{-(t_{k,j}'-t)(1-b)}(\partial^s \mathcal{L}(\boldsymbol{w}(t)))_i^2 \mathrm{d}t}{\varepsilon+F_i(t_{k,j})}
    \\
    \le &\frac{1}{2} \sum_{j=1}^\infty \frac{\int_{t_{k,j}}^{t_{k,j}'} (\partial^s \mathcal{L}(\boldsymbol{w}(t)))_i^2 \mathrm{d}t}{\varepsilon}
    \\
    \le &\frac{1}{2}\frac{\int_{t\in\tilde{B}_k} (\partial^s \mathcal{L}(\boldsymbol{w}(t)))_i^2 \mathrm{d}t}{\varepsilon} <\infty.
\end{align*}

The proof is completed.
\end{proof}

In the rest of the section, we extend the proof of Theorem \ref{lem:rms_approximate} from RMSProp to Adam.

Conditioner for Adam is the same as RMSProp, except that Adam will divide a bias-corrected term $1-b^t$ for conditioner each step, that is,
\begin{equation*}
    \frac{\mathrm{d} \boldsymbol{w}(t)}{\mathrm{d} t}=-\frac{\partial^s \mathcal{L}(\boldsymbol{w}(t))}{\sqrt{\varepsilon\mathbf{1}_d+\frac{\int_{0}^t e^{-(1-b)(t-\tau)}(1-b)\partial^s\mathcal{L}(\boldsymbol{w}(\tau))^2\mathrm{d} \tau}{(1-b^t)}}}.
\end{equation*}

Generally, updates of RMSProp and Adam (w/m) can be both expressed as 
\begin{equation*}
    \frac{\mathrm{d} \boldsymbol{w}(t)}{\mathrm{d} t}=-\frac{\partial^s \mathcal{L}(\boldsymbol{w}(t))}{\sqrt{\varepsilon\mathbf{1}_d+\frac{\int_{0}^t e^{-(1-b)(t-\tau)}(1-b)\partial^s\mathcal{L}(\boldsymbol{w}(\tau))^2\mathrm{d} \tau}{(1-a^t)}}} ,
\end{equation*}
where for RMSProp $a=0$, and for Adam $a=b$. For any  $0\le a<1$, $0\le b\le1$, define $F_i(t)\overset{\triangle}{=}\int_{0}^{t} (1-b) b^{t-\tau} (\partial^s \mathcal{L}(\boldsymbol{w}(t)))_i^2  \mathrm{d} \tau$ as in RMSProp case. We will show Lemmas \ref{lem:boundedness_for_F} and  \ref{lem:F_go_zero}.

\begin{lemma}
\label{lem:boundedness_for_F_adam}
For adaptive gradient flow defined as eq. (\ref{eq:flow}) with $\boldsymbol{h}^{-1}(t)=\sqrt{\varepsilon\mathbf{1}_d+\frac{\int_{0}^t e^{-(1-b)(t-\tau)}(1-b)\partial^s\mathcal{L}(\boldsymbol{w}(\tau))^2\mathrm{d} \tau}{(1-a^t)}}$ with  $0\le a<1$ and $0\le b\le1$, $F_i(t)$ is bounded ($i=1,2,\cdots,p$), that is,
\begin{equation*}
    \overline{\lim}_{t\rightarrow \infty} F_i(t)< \infty.
\end{equation*}
\end{lemma}
\begin{proof}
 The proof follows the same routine as proof of Lemma \ref{lem:boundedness_for_F}, except in this case we have
 \begin{align*}
     \int_{t_{k}}^{t_{k+1}} \frac{(\partial^s \mathcal{L}(\boldsymbol{w}(t)))_i^2}{\sqrt{\varepsilon+F_i(t)/(1-a^\tau)}} \mathrm{d}t \ge\int_{t_{k}}^{t_{k+1}} \frac{(\partial^s \mathcal{L}(\boldsymbol{w}(t)))_i^2}{\sqrt{\varepsilon+(k+1)/(1-a^{t_k})}} \mathrm{d}t.
 \end{align*}
 For $k \ge 1$, we further have
 \begin{equation*}
     \int_{t_{k}}^{t_{k+1}} \frac{(\partial^s \mathcal{L}(\boldsymbol{w}(t)))_i^2}{\sqrt{\varepsilon+(k+1)/(1-a^{t_k})}} \mathrm{d}t \ge \int_{t_{k}}^{t_{k+1}} \frac{(\partial^s \mathcal{L}(\boldsymbol{w}(t)))_i^2}{\sqrt{\varepsilon+(k+1)/(1-a^{t_1})}} \mathrm{d}t\ge  \frac{1-b}{\sqrt{\varepsilon+(k+1)/(1-a^{t_1})}},
 \end{equation*}
sum of which diverges.
 
 The proof is completed.
\end{proof}

\begin{lemma}
\label{lem:F_go_zero_adam}
For adaptive gradient flow, $\int_{0}^{\infty} g_i(t)^2 \mathrm{d}t$ converges ($i=1,2,\cdots,p$). Consequently,  $\lim_{t\rightarrow\infty} F_i(t)=0$ and $\lim_{t\rightarrow\infty} \boldsymbol{h}(t)=\frac{1}{\sqrt{\varepsilon}} \mathbf{1}_p$.
\end{lemma}
\begin{proof}
The proof is the same as proof of Lemma \ref{lem:F_go_zero}, except that
\begin{equation*}
    \mathcal{L}(0)\ge\sum_{i=1}^p\int_{1}^{\infty} \frac{(\partial^s \mathcal{L}(\boldsymbol{w}(t)))_i^2}{\sqrt{\varepsilon+F_i(t)/(1-a^t)}} \mathrm{d}t\ge \sum_{i=1}^p\int_{1}^{\infty} \frac{(\partial^s \mathcal{L}(\boldsymbol{w}(t)))_i^2}{\sqrt{M_i/(1-a^{1        })+\varepsilon}} \mathrm{d}t.
\end{equation*}

The proof is completed.
\end{proof}

\begin{proof}[Proof of Theorem \ref{lem:rms_approximate} for Adam flow]
We only need to prove Lebesgue integrability of $\frac{\mathrm{d}\log \boldsymbol{\beta}(t)}{\mathrm{d} t}$.

 The proof is the same as proof of Theorem \ref{lem:rms_approximate} for RMSProp flow, except that
 \begin{align*}
     &F_i(t'_{k,j})/(1-a^{t'_{k,j}})-F_{i}(t_{k,j})/(1-a^{t_{k,j}})
     \\
     \le&  F_i(t'_{k,j})/(1-a^{t'_{k,j}})-F_{i}(t_{k,j})/(1-a^{t'_{k,j}})
     \\
     \le& \frac{1}{1-a}(F_i(t'_{k,j})-F_i(t_{k,j})).
 \end{align*}
 
 The proof is completed.
\end{proof}

It is worth noting that the current framework of adaptive gradient flow can not cover Adam with a decaying $\varepsilon$ or without $\varepsilon$. It will be interesting to see if the framework can be modified to analyze these optimizers, and we leave this as a future work.

\subsection{Proof of Theorem \ref{thm:approximate_flow}}
\label{sec:appen_proof_approximate}
\subsubsection{Proof of surrogate margin Lemmas: Lemma \ref{lem:relationship_between_normalized_smoothed}, Lemma \ref{lem:lower_bound_margin}, and Lemma \ref{lem:conver_gamma}}
In the beginning, we first prove a basic lemma for normalized margin, i.e., the normalized margin $\gamma$ and normalized gradients are upper bounded:
\begin{lemma}
\label{lem:appen_bound_gamma}
For any $\boldsymbol{v} \in \mathbb{R}^p/\{\boldsymbol{0}\}$, the normalized margin $\gamma=\frac{\tilde{q}_{\min}(\boldsymbol{v})}{\Vert\boldsymbol{v}\Vert^L}$ and normalized gradients $\Vert\frac{\partial^s\tilde{q}_i(\boldsymbol{v})}{\Vert\boldsymbol{v}\Vert^{L-1}}\Vert$ ($i=1,2,\cdots,p$) are upper bounded universally.
\end{lemma}

\begin{proof}
By homogeneity of  $\tilde{q}_i$ ($i=1,2,\cdots,p$), only parameters with unit norm needed to be considered. That is,
\begin{align*}
   & \left\{\gamma(\boldsymbol{v}):\boldsymbol{v}\in  \mathbb{R}^p/\{\boldsymbol{0}\right\}
    \\
    =&\left\{\gamma(\boldsymbol{v}):\Vert\boldsymbol{v}\Vert=1\right\}.
\end{align*}

Since $\tilde{q}_i$ is continuous and $\{\boldsymbol{v}:\Vert\boldsymbol{v}\Vert=1\}$ is a compact set, normalized margin is upper bounded.

 Normalized gradients $\Vert\frac{\partial^s\tilde{q}_i(\boldsymbol{v})}{\Vert\boldsymbol{v}\Vert^{L-1}}\Vert$ are also bounded following similar routine since $\tilde{q}_i$ is locally Lipschitz.
\end{proof}

We then formally define surrogate norm $\rho(t)$ and surrogate margin $\tilde{\gamma}$ as follows:
\begin{definition}[Surrogate norm and surrogate margin]
 Let  $\boldsymbol{v}(t)$ obey an adaptive gradient flow $\mathcal{F}$ which satisfies Assumption \ref{assum: continuous}, with loss $\tilde{\mathcal{L}}$ and component learning rate $\boldsymbol{\beta}(t)$. The surrogate margin $\rho(t)$ along $\mathcal{F}$ is defined as 
\begin{equation*}
  \rho(t)=\Vert \boldsymbol{\beta}(t)^{-\frac{1}{2}}\odot\boldsymbol{v}(t)\Vert,
\end{equation*}
and surrogate margin $\tilde{\gamma}(t)$ is defined as 
\begin{equation*}
    \tilde{\gamma}(t)=\frac{f^{-1}(\log \frac{1}{\tilde{\mathcal{L}}(\boldsymbol{v}(t))})}{\rho(t)^L}.
\end{equation*}
\end{definition}

We can now restate Lemma \ref{lem:relationship_between_normalized_smoothed} as follows:
\begin{lemma}[Lemma \ref{lem:relationship_between_normalized_smoothed} restated]
\label{lem:appen_smoothed_normal}
 Let  $\boldsymbol{v}(t)$ obey an adaptive gradient flow $\mathcal{F}$ which satisfies Assumption \ref{assum: continuous}, with loss $\tilde{\mathcal{L}}$ and component learning rate $\boldsymbol{\beta}(t)$. Then we have $\lim_{t\rightarrow\infty}\frac{\rho(t)}{\Vert\boldsymbol{v}(t)\Vert}=1$. Furthermore, if further $\lim_{t\rightarrow \infty }\tilde{\mathcal{L}}(t)=\infty$, we have  $\lim_{t\rightarrow\infty}\frac{\gamma(t)}{\tilde{\gamma}(t)}=1$
\end{lemma}

\begin{proof}
By the definition of approximate norm $\rho(t)=\Vert\boldsymbol{\beta}(t)^{-\frac{1}{2}}\odot \boldsymbol{v}(t) \Vert$ and $\lim_{t\rightarrow\infty}\boldsymbol{\beta}(t)=\mathbf{1}_p$, we have that
\begin{equation*}
    1=\underline{\lim}_{t\rightarrow\infty}\min_i \boldsymbol{\beta}^{-\frac{1}{2}}_i(t)\le\underline{\lim}_{t\rightarrow\infty}\frac{\rho(t)}{\Vert\boldsymbol{v}(t)\Vert}\le\overline{\lim}_{t\rightarrow\infty}\frac{\rho(t)}{\Vert\boldsymbol{v}(t)\Vert}\le \overline{\lim}_{t\rightarrow\infty}\max_i \boldsymbol{\beta}^{-\frac{1}{2}}_i(t)=1,
\end{equation*}
which leads to $\lim_{t\rightarrow\infty}\frac{\rho(t)}{\Vert\boldsymbol{v}(t)\Vert}=1$.

 By the definition of $\tilde{\mathcal{L}}(\boldsymbol{v})$, we have 
\begin{equation*}
   e^{-f(\tilde{q}_{\min}(\boldsymbol{v}))}\le\tilde{\mathcal{L}}(\boldsymbol{v})=\sum_{i=1}^N e^{-f(\tilde{q}_i(\boldsymbol{v}))}\le N e^{-f(\tilde{q}_{\min}(\boldsymbol{v}))}.
\end{equation*}
Rearranging the above equation, we have
\begin{align*}
    &\tilde{q}_{\min}(\boldsymbol{v})\ge g\left(\log\frac{1}{\tilde{\mathcal{L}}(\boldsymbol{v})}\right)\ge  g(f(\tilde{q}_{\min}(\boldsymbol{v}))-\log N)
    \\
    =&g(f(\tilde{q}_{\min}(\boldsymbol{v})))-\log N g'(\xi)=\tilde{q}_{\min}(\boldsymbol{v})-\log N g'(\xi),
\end{align*}
where $\xi\in (f(\tilde{q}_{\min}(\boldsymbol{v}))-\log N,f(\tilde{q}_{\min}(\boldsymbol{v})))$.

Therefore, the surrogate margin can be bounded as 
\begin{align}
\nonumber
   &\frac{\tilde{q}_{\min}(\boldsymbol{v}(t))-\log N g'(\xi)}{\Vert\boldsymbol{v}(t)\Vert^L} \frac{\Vert\boldsymbol{v}(t)\Vert^L}{\rho(t)^L}=\frac{\tilde{q}_{\min}(\boldsymbol{v}(t))-\log N g'(\xi)}{\rho(t)^L} 
   \\\label{eq:equivalence_margin}
   \le& \tilde{\gamma}(t)\le \frac{\tilde{q}_{\min}}{\rho(t)^L}=\gamma(t) \left(\frac{\Vert\boldsymbol{v}(t)\Vert}{\rho(t)}\right)^L.
\end{align}

By the assumption that $\lim_{t\rightarrow\infty}\tilde{\mathcal{L}}=0$, there exists a large enough time $\tilde{t}$, such that, any time $t\ge \tilde{t}$, $g(\xi)>0$. The left side of the above equation \ref{eq:equivalence_margin} can be further rearranged as 
\begin{align*}
    \frac{\tilde{q}_{\min}(\boldsymbol{v}(t))-\log N g'(\xi)}{\Vert\boldsymbol{v}(t)\Vert^L} \frac{\Vert\boldsymbol{v}(t)\Vert^L}{\rho(t)^L}=&\left(\gamma(t)-\frac{\log N g'(\xi)}{\Vert \boldsymbol{v}(t)\Vert^L}\right)\frac{\Vert\boldsymbol{v}(t)\Vert^L}{\rho(t)^L}
    \\
    =&\left(\gamma(t)-\frac{\log N g'(\xi)g(\xi)}{\Vert \boldsymbol{v}(t)\Vert^Lg(\xi)}\right)\frac{\Vert\boldsymbol{v}(t)\Vert^L}{\rho(t)^L}
    \\
    \ge &\left(\gamma(t)-\frac{\log N g'(\xi)\gamma(t)}{g(\xi)}\right)\frac{\Vert\boldsymbol{v}(t)\Vert^L}{\rho(t)^L}
    \\
    =&\left(\gamma(t)-\frac{\log N \gamma(t)}{g(\xi)f'(g(\xi))}\right)\frac{\Vert\boldsymbol{v}(t)\Vert^L}{\rho(t)^L}.
\end{align*}

By the assumption that $\lim_{t\rightarrow\infty} \tilde{\mathcal{L}}(t)=0$, $\underline{\lim}_{t\rightarrow\infty} \xi\ge \underline{\lim}_{t\rightarrow\infty}  \log \frac{1}{\tilde{\mathcal{L}}(t)}-\log N=\infty$, which further indicates $\lim_{t\rightarrow\infty}g(\xi)f'(g(\xi))=\infty$ by the third item of Proposition \ref{lem:property_loss}. 
The proof is completed by taking $t$ to infinity of eq. (\ref{eq:equivalence_margin}).

\end{proof}

By Lemmas \ref{lem:appen_bound_gamma} and \ref{lem:appen_smoothed_normal}, we have that if $\lim_{t\rightarrow\infty}\tilde{\mathcal{L}}(t)\rightarrow0$, $\tilde{\gamma}(t)$ is upper bounded. We then lower bound $\tilde{\gamma}(t)$ by proving Lemma \ref{lem:lower_bound_margin}. As a warm-up, we first calculate the derivative of $\rho^2$.
\begin{lemma}
\label{lem:derivative_rho}
The derivative of $\rho^2$ is as follows:
\begin{equation*}
    \frac{1}{2}\frac{\mathrm{d} \rho(t)^2}{\mathrm{d}t}= L \nu(t)+ \left\langle \boldsymbol{v}(t),\boldsymbol{\beta}^{-\frac{1}{2}}(t)\odot \frac{\mathrm{d}\boldsymbol{\beta}^{-\frac{1}{2}}}{\mathrm{d}t}(t)\odot \boldsymbol{v}(t) \right\rangle,
\end{equation*}
where $\nu(t)$ is defined as $\sum_{i=1}^{N} e^{-f\left(\tilde{q}_{i}(\boldsymbol{v}(t))\right)} f^{\prime}\left(\tilde{q}_{i}(\boldsymbol{v}(t))\right) \tilde{q}_{i}(\boldsymbol{v}(t))$. Furthermore, we have that $\nu(t)>\frac{g\left(\log \frac{1}{\tilde{\mathcal{L}}(\boldsymbol{v}(t))}\right)}{g^{\prime}\left(\log \frac{1}{\tilde{\mathcal{L}}(\boldsymbol{v}(t))}\right)} \tilde{\mathcal{L}}(\boldsymbol{v}(t))$.
\end{lemma}
\begin{proof}

By taking derivative directly, we have that 
\begin{align*}
    \frac{1}{2}\frac{\mathrm{d} \rho(t)^2}{\mathrm{d} t}&= \left\langle\boldsymbol{\beta}^{-\frac{1}{2}}(t)\odot\boldsymbol{v}(t), \frac{\mathrm{d}\boldsymbol{\beta}(t)^{-\frac{1}{2}}\odot\boldsymbol{v}(t)}{\mathrm{d}t}\right\rangle
    \\
    &=\left\langle\boldsymbol{\beta}(t)^{-\frac{1}{2}}\odot\boldsymbol{v}(t), \boldsymbol{v}(t)\odot\frac{\mathrm{d}\boldsymbol{\beta}(t)^{-\frac{1}{2}}}{\mathrm{d}t}+\boldsymbol{\beta}(t)^{-\frac{1}{2}}\odot\frac{\mathrm{d}\boldsymbol{v}(t)}{\mathrm{d}t}\right\rangle
    \\
    &=-\left\langle\boldsymbol{v}(t),\partial^s \tilde{\mathcal{L}}(\boldsymbol{v}(t))\right\rangle+\left\langle\boldsymbol{\beta}^{-\frac{1}{2}}(t)\odot\boldsymbol{v}(t), \boldsymbol{v}(t)\odot\frac{\mathrm{d}\boldsymbol{\beta}(t)^{-\frac{1}{2}}}{\mathrm{d}t}\right\rangle
    \\
    &\overset{(*)}{=}L \sum_{i=1}^{N} e^{-f\left(\tilde{q}_{i}(\boldsymbol{v}(t))\right)}f^{\prime}\left(\tilde{q}_{i}(\boldsymbol{v}(t))\right) \tilde{q}_{i}(\boldsymbol{v}(t))+\left\langle\boldsymbol{\beta}(t)^{-\frac{1}{2}}\odot\boldsymbol{v}(t), \boldsymbol{v}(t)\odot\frac{\mathrm{d}\boldsymbol{\beta}(t)^{-\frac{1}{2}}}{\mathrm{d}t}\right\rangle ,
\end{align*}
where eq. $(*)$ comes from Homogeneity Assumption \ref{assum: continuous}. I. 

Furthermore, since $\tilde{q}_{i}(\boldsymbol{v}(t))\ge\tilde{q}_{\min}(\boldsymbol{v}(t))\ge g(\log \frac{1}{\tilde{\mathcal{L}}(\boldsymbol{v}(t))})$ for all $i\in[N]$ and $f'(x)x$ keeps increasing on $(0,\infty)$ by Proposition \ref{lem:property_loss}, 
\begin{equation*}
    f^{\prime}\left(\tilde{q}_{i}(\boldsymbol{v}(t))\right) \tilde{q}_{i} (\boldsymbol{v}(t))\geq f^{\prime}\left(g\left(\log \frac{1}{\tilde{\mathcal{L}}(\boldsymbol{v}(t))}\right)\right) \cdot g\left(\log \frac{1}{\tilde{\mathcal{L}}(\boldsymbol{v}(t))}\right)=\frac{g\left(\log \frac{1}{\tilde{\mathcal{L}}(\boldsymbol{v}(t))}\right)}{g^{\prime}\left(\log \frac{1}{\tilde{\mathcal{L}}(\boldsymbol{v}(t))}\right)}.
\end{equation*}
Therefore, 
\begin{align*}
    &\sum_{i=1}^{N} e^{-f\left(\tilde{q}_{i}(\boldsymbol{v}(t))\right(\boldsymbol{v}(t)))} f^{\prime}\left(\tilde{q}_{i}(\boldsymbol{v}(t))\right) \tilde{q}_{i}(\boldsymbol{v}(t)) 
    \\
    \geq& \frac{g\left(\log \frac{1}{\tilde{\mathcal{L}}(\boldsymbol{v}(t))}\right)}{g^{\prime}\left(\log \frac{1}{\tilde{\mathcal{L}}(\boldsymbol{v}(t))}\right)}\sum_{i=1}^{N} e^{-f\left(\tilde{q}_{i}(\boldsymbol{v}(t))\right)} =\frac{g\left(\log \frac{1}{\tilde{\mathcal{L}}(\boldsymbol{v}(t))}\right)}{g^{\prime}\left(\log \frac{1}{\tilde{\mathcal{L}}(\boldsymbol{v}(t))}\right)} \tilde{\mathcal{L}}(\boldsymbol{v}(t)).
\end{align*}

The proof is completed.
\end{proof}

With the estimation of $\frac{\mathrm{d} \rho^2}{\mathrm{d}t}$ above, we come to the proof of Lemma \ref{lem:lower_bound_margin}. 

\begin{proof}[Proof of Lemma \ref{lem:lower_bound_margin}]
We first construct time $t_1$ as follows:  by properties of $\boldsymbol{\beta}(t)$ in Definition \ref{def:approximate_GF}, there exists some large enough time $t_1>t_0$, such that for any $t>t_1$,
\begin{gather*}
\sum_{i=1}^p \int_{t_1}^{\infty}\left(\frac{\mathrm{d}\log \left(\boldsymbol{\beta}^{-\frac{1}{2}}_i(t)\right)}{\mathrm{d}t}\right)_{+} \le \min\left\{\frac{1}{2L},\frac{1}{4}\right\},
\\
\sum_{i=1}^p \int_{t_1}^{\infty}\left(\frac{\mathrm{d}\log \left(\boldsymbol{\beta}^{-\frac{1}{2}}_i(t)\right)}{\mathrm{d}t}\right)_{-} \ge -\min\left\{\frac{1}{2L},\frac{1}{4}\right\},
\end{gather*}
and
\begin{equation*}
    \frac{1}{2}\le\Vert \boldsymbol{\beta}^{-\frac{1}{2}}(t) \Vert_{\infty}\le \frac{3}{2}.
\end{equation*}

Taking 
logarithmic derivative to $\tilde{\gamma}(\boldsymbol{v}(t))$, we have
\begin{align*}
&\frac{\mathrm{d}}{\mathrm{d} t} \log \tilde{\gamma}(t)
\\
=&\frac{\mathrm{d}}{\mathrm{d} t}\left(\log \left(g\left(\log \frac{1}{\tilde{\mathcal{L}}(\boldsymbol{v}(t))}\right)\right)-L \log \rho(t)\right) 
\\
=&\frac{g^{\prime}\left(\log \frac{1}{\tilde{\mathcal{L}}(\boldsymbol{v}(t))}\right)}{g\left(\log \frac{1}{\tilde{\mathcal{L}}(\boldsymbol{v}(t))}\right)} \cdot \frac{1}{\tilde{\mathcal{L}}(\boldsymbol{v}(t))} \cdot\left(-\frac{\mathrm{d} \tilde{\mathcal{L}}(\boldsymbol{v}(t))}{\mathrm{d} t}\right)-L^{2} \cdot \frac{\nu(t)}{\rho(t)^{2}}- \frac{L\left\langle \boldsymbol{v}(t),\boldsymbol{\beta}^{-\frac{1}{2}}(t)\odot \frac{\mathrm{d}\boldsymbol{\beta}^{-\frac{1}{2}}}{\mathrm{d}t}(t)\odot \boldsymbol{v}(t) \right\rangle}{\rho(t)^2}.
\end{align*}

Let $A=\frac{g^{\prime}\left(\log \frac{1}{\tilde{\mathcal{L}}(\boldsymbol{v}(t))}\right)}{g\left(\log \frac{1}{\tilde{\mathcal{L}}(\boldsymbol{v}(t))}\right)} \cdot \frac{1}{\tilde{\mathcal{L}}(\boldsymbol{v}(t))} \cdot\left(-\frac{\mathrm{d} \tilde{\mathcal{L}}(\boldsymbol{v}(t))}{\mathrm{d} t}\right)-L^{2} \cdot \frac{\nu(t)}{\rho(t)^{2}}$ and $B=\frac{L\left\langle \boldsymbol{v}(t), \boldsymbol{\beta}^{-\frac{1}{2}}(t)\odot \frac{\mathrm{d}\boldsymbol{\beta}^{-\frac{1}{2}}}{\mathrm{d}t}(t)\odot \boldsymbol{v}(t) \right\rangle}{\rho(t)^2}$. We then have
\begin{align*}
A&  \ge\frac{1}{\nu(t)} \cdot\left(-\frac{\mathrm{d} \tilde{\mathcal{L}}(\boldsymbol{v}(t))}{\mathrm{d} t}\right)-L^{2} \cdot \frac{\nu(t)}{\rho(t)^{2}} 
\\
& \geq \frac{1}{\nu(t)} \cdot\left(-\frac{\mathrm{d} \tilde{\mathcal{L}}(\boldsymbol{v}(t))}{\mathrm{d} t}-\frac{L^{2} \nu(t)^{2}}{\rho(t)^{2}}\right)
\\
&=\frac{1}{\nu(t)}\left(\left\langle \frac{\mathrm{d} \boldsymbol{v}(t)}{\mathrm{d}t},\boldsymbol{\beta}(t)^{-1}\odot\frac{\mathrm{d} \boldsymbol{v}(t)}{\mathrm{d}t} \right\rangle-\left\langle\boldsymbol{\beta}(t)^{-\frac{1}{2}}\odot \frac{\mathrm{d} \boldsymbol{v}(t)}{\mathrm{d}t},\widehat{\boldsymbol{\beta}(t)^{-\frac{1}{2}}\odot \boldsymbol{v}(t)} \right\rangle^2\right) 
\\
\overset{(*)}{\ge}0,
\end{align*}
where inequality $(*)$ comes from Cauchy-Schwarz inequality.

As for $B$, we have that 
\begin{align*}
    B&=-\frac{L\left\langle \boldsymbol{v}(t),\boldsymbol{\beta}^{-\frac{1}{2}}\odot \frac{\mathrm{d}\boldsymbol{\beta}^{-\frac{1}{2}}(t)}{\mathrm{d}t}(t)\odot \boldsymbol{v}(t) \right\rangle}{\rho^2}
    \\
    &=-L\frac{\sum_{i=1}^p v_i^2(t)\beta_i^{-\frac{1}{2}}(t)\frac{\mathrm{d}\beta_i^{-\frac{1}{2}}(t)}{\mathrm{d}t}}{\sum_{i=1}^p v_i^2(t)\beta_i^{-1}(t) }
    \\
    &\ge-L \sum_{i=1}^p \left(\frac{ \mathrm{d}\log \boldsymbol{ \beta}^{-\frac{1}{2}}_i(t)}{\mathrm{d}t}\right)_{+}.
\end{align*}

Combining the estimation of $A$ and $B$, we then have
\begin{align}
\label{eq:derivative_gamma}
    \frac{\mathrm{d}}{\mathrm{d} t} \log \tilde{\gamma} (t)\ge A+B\ge -L \sum_{i=1}^p \left(\frac{ \mathrm{d}\log \boldsymbol{ \beta}^{-\frac{1}{2}}_i(t)}{\mathrm{d}t}\right)_{+},
\end{align}
and integrating both sides leads to
\begin{equation*}
    \log \tilde{\gamma}(t)- \log \tilde{\gamma}(t_1)\ge-\frac{1}{2}. 
\end{equation*}
The proof is completed.

\end{proof}

By the proof of Lemma \ref{lem:lower_bound_margin}, we can then prove  convergence of surrogate margin $\tilde{\gamma}$.

\begin{proof}[Proof of Lemma \ref{lem:conver_gamma}]
By eq. (\ref{eq:derivative_gamma}), $\log \tilde{\gamma}(t)+\int_{t_1}^t L \sum_{i=1}^p \left(\frac{ \mathrm{d}\log \boldsymbol{ \beta}^{-\frac{1}{2}}_i(\tau)}{\mathrm{d}\tau}\right)_{+} \mathrm{d} \tau$ is non-decreasing. Furthermore, since $g(\log \frac{1}{\tilde{\mathcal{L}}(\boldsymbol{v}(t))})\le\tilde{q}_{\min}(\boldsymbol{v}(t))$, we have that 
\begin{equation*}
    \overline{\lim}_{t\rightarrow\infty} \tilde{\gamma}(t)\le \overline{\lim}_{t\rightarrow\infty} \frac{\tilde{q}_{\min}(\boldsymbol{v}(t))}{\rho(t)^L}=\overline{\lim}_{t\rightarrow\infty} \frac{\tilde{q}_{\min}(\boldsymbol{v}(t))}{\Vert\boldsymbol{v}(t)\Vert^L},
\end{equation*}
which by Lemma \ref{lem:appen_bound_gamma}, the last term is bounded.

Therefore,  $\log \tilde{\gamma}(t)+\int_{t_1}^t L \sum_{i=1}^p \left(\frac{ \mathrm{d}\log \boldsymbol{ \beta}^{-\frac{1}{2}}_i(\tau)}{\mathrm{d}t}\right)_{+} \mathrm{d} \tau$ is upper bounded, which further indicates that it converges due to its monotony. The proof is completed by the convergence of $\int_{t_1}^t L \sum_{i=1}^p \left(\frac{ \mathrm{d}\log \boldsymbol{ \beta}^{-\frac{1}{2}}_i(\tau)}{\mathrm{d}\tau}\right)_{+} \mathrm{d} \tau$.

\end{proof}

\subsubsection{Convergence of $\tilde{\mathcal{L}}$ and $\rho$: Proof of Lemma \ref{lem:assymptotic_loss}}

 Here we restate the complete  Lemma \ref{lem:assymptotic_loss}.
\begin{lemma}
\label{lem:estimation_loss_G}
Let $\boldsymbol{v}$ obey an adaptive gradient flow which satisfies Assumption \ref{assum: continuous}. Let $t_1$ be constructed as Lemma \ref{lem:lower_bound_margin}. Then, for any $t\ge t_1$, define 
\begin{equation*}
    G(t)= \int_{\frac{1}{\tilde{\mathcal{L}}(t_1)}}^{ \frac{1}{\tilde{\mathcal{L}}(t)}}  \frac{g^{\prime}\left(\log x\right)^{2}}{g\left(\log x\right)^{2-2/ L}} \cdot  dx.
\end{equation*} Then, for any $t\ge t_1$, the following inequality holds:
\begin{equation*}
    G(t)-G(t_1)\ge\int_{\frac{1}{\tilde{\mathcal{L}}(t_1)}}^{ \frac{1}{\tilde{\mathcal{L}}(t)}}  \frac{g^{\prime}\left(\log x\right)^{2}}{g\left(\log x\right)^{2-2/ L}} \cdot  dx \ge (t-t_1)e^{-\frac{1}{L}}L^{2} \tilde{\gamma}\left(t_{1}\right)^{2/ L}.
\end{equation*}
Consequently, $\lim_{t\rightarrow\infty} \tilde{\mathcal{L}}=0$ and $\lim_{t\rightarrow\infty}\rho=\infty$.
\end{lemma}
\begin{proof}

\begin{equation*}
    -\frac{\mathrm{d} \tilde{\mathcal{L}}(\boldsymbol{v}(t))}{\mathrm{d} t}=\left\|\boldsymbol{\beta}^{-\frac{1}{2}}(t)\odot\frac{\mathrm{d} \boldsymbol{v}(t)}{\mathrm{d} t}\right\|_{2}^{2} \geq\langle\boldsymbol{\beta}(t)^{-\frac{1}{2}}\odot \frac{\mathrm{d} \boldsymbol{v}(t)}{\mathrm{d}t},\widehat{\boldsymbol{\beta}^{-\frac{1}{2}}(t)\odot \boldsymbol{v}(t)} \rangle^{2}= L^{2} \cdot \frac{\nu(t)^{2}}{\rho(t)^{2}}.
\end{equation*}

Furthermore, we have that 
\begin{align*}
    L^{2} \frac{\nu(t)^{2}}{\rho(t)^{2}}
    &\ge 
    L^{2} \cdot\left(\frac{g\left(\log \frac{1}{\tilde{\mathcal{L}}(\boldsymbol{v}(t))}\right)}{g^{\prime}\left(\log \frac{1}{\tilde{\mathcal{L}}(\boldsymbol{v}(t))}\right)} \tilde{\mathcal{L}}(\boldsymbol{v}(t))\right)^{2} \cdot\left(\frac{\tilde{\gamma}(t)}{g\left(\log \frac{1}{\tilde{\mathcal{L}}(\boldsymbol{v}(t))}\right)}\right)^{2/ L}
    \\
    &\ge L^{2}\cdot\left(\frac{g\left(\log \frac{1}{\tilde{\mathcal{L}}(\boldsymbol{v}(t))}\right)}{g^{\prime}\left(\log \frac{1}{\tilde{\mathcal{L}}(\boldsymbol{v}(t))}\right)} \tilde{\mathcal{L}}(\boldsymbol{v}(t))\right)^{2} \cdot\left(\frac{e^{-1/2}\tilde{\gamma}(t_1)}{g\left(\log \frac{1}{\tilde{\mathcal{L}}(\boldsymbol{v}(t))}\right)}\right)^{2/ L}.
\end{align*}

By simple calculation, we have that
\begin{equation*}
    \frac{g^{\prime}\left(\log \frac{1}{\tilde{\mathcal{L}}(\boldsymbol{v}(t))}\right)^{2}}{g\left(\log \frac{1}{\tilde{\mathcal{L}}(\boldsymbol{v}(t))}\right)^{2-2/ L}} \cdot \frac{\mathrm{d}}{\mathrm{d} t} \frac{1}{\tilde{\mathcal{L}}(\boldsymbol{v}(t))} \geq e^{-\frac{1}{L}}L^{2} \tilde{\gamma}\left(t_{1}\right)^{2/ L}.
\end{equation*}

Taking integration to both sides, we have 
\begin{equation*}
    \int_{\frac{1}{\tilde{\mathcal{L}}(t_1)}}^{ \frac{1}{\tilde{\mathcal{L}}(t)}}  \frac{g^{\prime}\left(\log x\right)^{2}}{g\left(\log x\right)^{2-2/ L}} \cdot  dx \ge (t-t_1)e^{-\frac{1}{L}}L^{2} \tilde{\gamma}\left(t_{1}\right)^{2/ L}.
\end{equation*}

Since $\lim_{t\rightarrow\infty}(t-t_1)e^{-\frac{1}{L}}L^{2} \tilde{\gamma}\left(t_{1}\right)^{2/ L}=\infty $, we have that $\lim_{t\rightarrow\infty}  \frac{1}{\tilde{\mathcal{L}}(t)}=\infty$.

The proof is completed.
\end{proof}

\subsubsection{Verification of KKT Condition}
\label{sec:KKT_def}

In Lemma \ref{lem:construction_kkt}, we omit the construction of coefficients $\lambda_i$ to highlight the key factors of coefficients $(\mathcal{O}(1-\langle\hat{v}(t),\widehat{-\partial^s \tilde{\mathcal{L}}(t)}\rangle),\mathcal{O}\left(\log \frac{1}{\tilde{\mathcal{L}}(\boldsymbol{v}(t))}\right))$. We restate Lemma \ref{lem:construction_kkt} and provide the detailed construction as follows:

\begin{lemma}[Lemma \ref{lem:construction_kkt} restated]
\label{lem: construction of KKT_appen}
Let $\boldsymbol{v}$ obey adaptive gradient flow $\mathcal{F}$ with empirical loss $\tilde{\mathcal{L}}$ satisfying Assumption \ref{assum: continuous}. Let time $t_1$ be constructed as Lemma \ref{lem:lower_bound_margin}. Then, define coefficients in Definition \ref{def:general_KKT} as $\lambda_i(t)=\tilde{q}_{\min }(\boldsymbol{v}(t))^{1-2 / L} \Vert \boldsymbol v(t)\Vert \cdot e^{-f\left(\tilde{q}_{i}(\boldsymbol{v}(t))\right)} f^{\prime}\left(\tilde{q}_{i}(\boldsymbol{v}(t
))\right) /\Vert\partial^s \tilde{\mathcal{L}}(\boldsymbol{v}(t)) \Vert_{2}$. Then, for any time $t\ge t_1$, $\tilde{\boldsymbol{v}}(t)=\tilde{q}_{\min}(\boldsymbol{v}(t))^{-\frac{1}{L}}\boldsymbol{v}(t)$ is an $(\varepsilon(t),\delta(t))$ KKT point of $L^2$ max-margin problem $(P)$ defined in Theorem \ref{thm:approximate_flow} , where $\varepsilon(t)$, $\delta(t)$ are defined as follows:
\begin{align*}
     \varepsilon(t)&=8e^{\frac{1}{L}}\frac{1}{\tilde{\gamma}(t_1)^{2/L}}(1-\cos(\boldsymbol{\theta}(t)))
     \\
     \delta(t)&=\frac{2e^{-1+\frac{1}{L}}K N}{L}K^{ \log _{2} \left(2^{L+1}e^{\frac{1}{2}} B_{1} / \tilde{\gamma}(t_1)\right)} \tilde{\gamma}(t_1)^{-2/L} \frac{1}{\log\frac{1}{\tilde{\mathcal{L}}(\boldsymbol{v}(t))}},
\end{align*}
where $\cos(\boldsymbol{\theta}(t))$  is defined as the cosine of angle between  $ \boldsymbol{v}(t)$ and $-\partial^s\tilde{\mathcal{L}}(\boldsymbol{v}(t))$, i.e., 
\begin{equation*}
    \cos(\boldsymbol{\theta}(t))=\left\langle\hat{\boldsymbol{v}}(t), -\widehat{\partial^s\tilde{\mathcal{L}}(\boldsymbol{v}(t))}\right\rangle.
\end{equation*}

\end{lemma}

\begin{proof}
We verify the definition of approximate KKT point directly.

\begin{align*}
    &\quad\left\|\tilde{\boldsymbol{v}}(t)-\sum_{i=1}^{N} \lambda_{i}(t) \partial^s \tilde{q}_i\left(\tilde{\boldsymbol{v}}(t)\right)\right\|_{2}^{2}
    \\
    &=\left\|\tilde{q}_{\min}(\boldsymbol{v}(t))^{- 1/ L} \Vert \boldsymbol{v}(t)\Vert\hat{\boldsymbol{v}}(t)-\frac{\tilde{q}_{\min}(\boldsymbol{v}(t))^{- 1/ L} \Vert \boldsymbol{v}(t)\Vert e^{-f\left(\tilde{q}_{i}(\boldsymbol{v}(t))\right)} f^{\prime}\left(\tilde{q}_{i}(\boldsymbol{v}(t))\right)\partial^s \tilde{q}_i\left(\boldsymbol{v}(t)\right)} {\Vert\partial^s \tilde{\mathcal{L}} (\boldsymbol{v}(t))\Vert_{2}} \right\|_{2}^{2}
    \\
    &=\tilde{q}_{\min}(\boldsymbol{v}(t))^{- 2/ L} \Vert \boldsymbol {v}(t)\Vert^2\left\|\hat{\boldsymbol{v}}(t)+\frac{\partial^s \tilde{\mathcal{L}}(\boldsymbol{v}(t))} {\Vert\partial^s \tilde{\mathcal{L}} (\boldsymbol{v}(t))\Vert_{2}} \right\|_{2}^{2}
    \\
    &=\tilde{q}_{\min}(\boldsymbol{v}(t))^{-2 / L} \Vert \boldsymbol v(t)\Vert^{2}\left(2-2 \left\langle\hat{\boldsymbol{v}}(t), \widehat{-\partial^s\tilde{\mathcal{L}}(\boldsymbol{v}(t))}\right\rangle\right) 
    \\
    &\leq 2g\left(\log \left(\frac{1}{\tilde{\mathcal{L}}(\boldsymbol{v}(t))}\right)\right)^{-2 / L} \Vert \boldsymbol{v}(t)\Vert^{2}(1-\cos(\boldsymbol{\theta}(t)))
    \\
    &\leq 8g\left(\log \left(\frac{1}{\tilde{\mathcal{L}}(\boldsymbol{v}(t))}\right)\right)^{-2 / L} \rho(t)^{2}(1-\cos(\boldsymbol{\theta}(t)))
    \\
    &= 8\frac{1}{\tilde{\gamma}(t)^{2/L}}(1-\cos(\boldsymbol{\theta}(t)))\le 8\frac{1}{\tilde{\gamma}(t)^{2/L}}(1-\cos(\boldsymbol{\theta}(t)))
    \\
    &\le  8e^{\frac{1}{L}}\frac{1}{\tilde{\gamma}(t_1)^{2/L}}(1-\cos(\boldsymbol{\theta}(t))).
\end{align*}

As for $\delta$, we have that
\begin{align*}
    &\sum_{i=1}^N \lambda_i\left( \tilde{q}_i(\tilde{\boldsymbol{v}}(t))-1\right)
    \\
    =&\frac{\sum_{i=1}^N \tilde{q}_{\min}^{1-2 / L}(\boldsymbol{v}(t)) \Vert \boldsymbol{v}(t)\Vert \cdot e^{-f\left(\tilde{q}_{i}(\boldsymbol{v}(t)\right)} f^{\prime}\left(\tilde{q}_{i}(\boldsymbol{v}(t)\right) \left(\frac{\tilde{q}_i(\boldsymbol{v}(t))}{\tilde{q}_{\min}(\boldsymbol{v}(t))}-1\right)}{\Vert\partial^s \tilde{\mathcal{L}}(\boldsymbol{v}(t) \Vert_{2}}
    \\
    =&\frac{\sum_{i=1}^N \tilde{q}_{\min}(\boldsymbol{v}(t))^{-2 / L} \Vert \boldsymbol {v}(t)\Vert \cdot e^{-f\left(\tilde{q}_{i}(\boldsymbol{v}(t))\right)} f^{\prime}\left(\tilde{q}_{i}(\boldsymbol{v}(t))\right) \left(\tilde{q}_i(\boldsymbol{v}(t))-\tilde{q}_{\min}(\boldsymbol{v}(t))\right)}{\Vert\partial^s \tilde{\mathcal{L}}(\boldsymbol{v}(t)) \Vert_{2}}
    \\
    \overset{(*)}{\le}& \sum_{i=1}^N \frac{2K}{L} \tilde{q}_{\min}(\boldsymbol{v}(t))^{-2 / L} \Vert \boldsymbol {v}(t)\Vert^2 e^{f\left(\tilde{q}_{\min}(\boldsymbol{v}(t))\right)-f\left(\tilde{q}_{i}(\boldsymbol{v}(t))\right)} f^{\prime}\left(\tilde{q}_{i}(\boldsymbol{v}(t))\right)
    \\\cdot&  \left(\tilde{q}_i(\boldsymbol{v}(t))-\tilde{q}_{\min}(\boldsymbol{v}(t))\right)\frac{1}{\log\frac{1}{\tilde{\mathcal{L}}(\boldsymbol{v}(t))}},
    \end{align*}
    where inequality $(*)$ is because 
\begin{align*}
    \Vert \partial^s \tilde{\mathcal{L}}(\boldsymbol{v}(t))\Vert\ge& \left\langle\partial^s \tilde{\mathcal{L}}(\boldsymbol{v}(t)), \hat{\boldsymbol{v}}(t)\right\rangle=\frac{L\nu(t)}{\Vert \boldsymbol{v}(t)\Vert}
    \\
    \ge&\frac{L}{\Vert \boldsymbol{v}(t)\Vert}\frac{g\left(\log \frac{1}{\tilde{\mathcal{L}}(\boldsymbol{v}(t))}\right)}{g^{\prime}\left(\log \frac{1}{\tilde{\mathcal{L}}(\boldsymbol{v}(t))}\right)} \tilde{\mathcal{L}}(\boldsymbol{v}(t)) \geq\frac{L}{\Vert \boldsymbol{v}(t)\Vert} \frac{1}{2 K} \log \frac{1}{\tilde{\mathcal{L}}(\boldsymbol{v}(t))} \cdot \tilde{\mathcal{L}}(\boldsymbol{v}(t)) 
    \\
    \geq& \frac{L}{\Vert \boldsymbol{v}(t)\Vert}\frac{1}{2 K}  e^{-f\left(\tilde{q}_{\min}(\boldsymbol{v}(t))\right)}\log \frac{1}{\tilde{\mathcal{L}}(\boldsymbol{v}(t))}.
\end{align*}
Bounding $\Vert \boldsymbol{v}\Vert$ using $\rho$ and applying the definition of surrogate margin $\tilde{\gamma}$, we further have

    \begin{align*}
    &\sum_{i=1}^N \lambda_i\left( \tilde{q}_i(\tilde{\boldsymbol{v}}(t))-1\right)
    \\
    \le & \sum_{i=1}^N \frac{8K}{L} \tilde{q}_{\min}^{-2 / L}(\boldsymbol{v}(t)) \rho(t)^2  e^{f\left(\tilde{q}_{\min}(\boldsymbol{v}(t))\right)-f\left(\tilde{q}_{i}(\boldsymbol{v}(t))\right)} f^{\prime}\left(\tilde{q}_{i}(\boldsymbol{v}(t))\right)
    \left(\tilde{q}_i(\boldsymbol{v}(t))-\tilde{q}_{\min}(\boldsymbol{v}(t)\right)\frac{1}{\log\frac{1}{\tilde{\mathcal{L}}(\boldsymbol{v}(t))}}
    \\
    \le &\sum_{i=1}^N \frac{8K}{L} \tilde{\gamma}(t)^{-2/L} \cdot e^{f\left(\tilde{q}_{\min}(\boldsymbol{v}(t))\right)-f\left(\tilde{q}_{i}(\boldsymbol{v}(t))\right)} f^{\prime}\left(\tilde{q}_{i}(\boldsymbol{v}(t))\right)
    \left(\tilde{q}_i(\boldsymbol{v}(t))-\tilde{q}_{\min}(\boldsymbol{v}(t))\right)\frac{1}{\log\frac{1}{\tilde{\mathcal{L}}(\boldsymbol{v}(t))}}
    \\
    \le &\sum_{i=1}^N \frac{8e^{\frac{1}{L}}K}{L} \tilde{\gamma}(t_1)^{-2/L} \cdot e^{f\left(\tilde{q}_{\min}(\boldsymbol{v}(t))\right)-f\left(\tilde{q}_{i}(\boldsymbol{v}(t))\right)} f^{\prime}\left(\tilde{q}_{i}(\boldsymbol{v}(t))\right) \left(\tilde{q}_i(\boldsymbol{v}(t))-\tilde{q}_{\min}(\boldsymbol{v}(t))\right)\frac{1}{\log\frac{1}{\tilde{\mathcal{L}}(\boldsymbol{v}(t))}}.
\end{align*}

Since $f(\tilde{q}_{\min}(\boldsymbol{v}(t))(\boldsymbol{v}(t)))-f(\tilde{q}_i(\boldsymbol{v}(t)))=(\tilde{q}_{\min}(\boldsymbol{v}(t))-\tilde{q}_i(\boldsymbol{v}(t)))f'(\xi_i(t))$ ($\xi_i(t)\in [\tilde{q}_{\min}(\boldsymbol{v}(t)),\tilde{q}_i(\boldsymbol{v}(t))]$ is guaranteed by Mean Value Theorem), we then have 
\begin{align*}
    &\sum_{i=1}^N \lambda_i\left( \tilde{q}_i(\tilde{\boldsymbol{v}}(t))-1\right)
    \\
    \le&
    \sum_{i=1}^N \frac{2e^{\frac{1}{L}}K}{L} \tilde{\gamma}(t_1)^{-2/L} \cdot e^{f\left(\tilde{q}_{\min}(\boldsymbol{v}(t))\right)-f\left(\tilde{q}_{i}(\boldsymbol{v}(t))\right)} f^{\prime}\left(\tilde{q}_{i}(\boldsymbol{v}(t))\right)
    \\
    & \left(\tilde{q}_i(\boldsymbol{v}(t))-\tilde{q}_{\min}(\boldsymbol{v}(t))\right)\frac{1}{\log\frac{1}{\tilde{\mathcal{L}}(\boldsymbol{v}(t))}}
    \\
    =&\sum_{i=1}^N \frac{2e^{\frac{1}{L}}K}{L} \tilde{\gamma}(t_1)^{-2/L} \cdot e^{(\tilde{q}_{\min}(\boldsymbol{v}(t))-\tilde{q}_i(\boldsymbol{v}(t)))f'(\xi_i(t))} f^{\prime}\left(\tilde{q}_{i}(\boldsymbol{v}(t))\right) 
    \\
    &\left(\tilde{q}_i(\boldsymbol{v}(t))-\tilde{q}_{\min}(\boldsymbol{v}(t))\right)\frac{1}{\log\frac{1}{\tilde{\mathcal{L}}(\boldsymbol{v}(t))}}
    \\
    \le &\sum_{i=1}^N \frac{2e^{\frac{1}{L}}K}{L} \tilde{\gamma}(t_1)^{-2/L} \cdot e^{(\tilde{q}_{\min}(\boldsymbol{v}(t))-\tilde{q}_i(\boldsymbol{v}(t)))f'(\xi_i(t))}K^{\left\lceil\log _{2}\left(\tilde{q}_{i}(\boldsymbol{v}(t)) / \xi_{i}(t)\right)\right\rceil} f^{\prime}\left(\xi_{i}(t)\right)
    \\
    &\left(\tilde{q}_i(\boldsymbol{v}(t))-\tilde{q}_{\min}(\boldsymbol{v}(t))\right)\frac{1}{\log\frac{1}{\tilde{\mathcal{L}}(\boldsymbol{v}(t))}}
    \\
    \overset{(*)}{\le}&\sum_{i=1}^N \frac{2e^{\frac{1}{L}}K}{L} \tilde{\gamma}(t_1)^{-2/L} \cdot e^{(\tilde{q}_{\min}(\boldsymbol{v}(t))-\tilde{q}_i(\boldsymbol{v}(t)))f'(\xi_i(t))}K^{ \log _{2}\left(2e^{\frac{1}{2}} B_{1} / \tilde{\gamma}(t_1)\right)} f^{\prime}\left(\xi_{i}(t)\right)
    \\& \left(\tilde{q}_i(\boldsymbol{v}(t))-\tilde{q}_{\min}(\boldsymbol{v}(t))\right)\frac{1}{\log\frac{1}{\tilde{\mathcal{L}}(\boldsymbol{v}(t))}}
    \\
    \overset{(**)}{\le}&\frac{2e^{-1+\frac{1}{L}}K N}{L}K^{ \log _{2} \left(2^{L+1}e^{\frac{1}{2}} B_{1} / \tilde{\gamma}(t_1)\right)} \tilde{\gamma}(t_1)^{-2/L} \frac{1}{\log\frac{1}{\tilde{\mathcal{L}}(\boldsymbol{v}(t))}},
\end{align*}
where eq. $(*)$ is because 
\begin{align*}
    \left\lceil\log _{2}\left(\tilde{q}_{i}(\boldsymbol{v}(t)) / \xi_{i}(t)\right)\right\rceil &\leq \log _{2}\left(\tilde{q}_{i}(\boldsymbol{v}(t)) / \tilde{q}_{\min}(\boldsymbol{v}(t))\right)+1
    \\&\leq \log_{2}\left(\tilde{q}_{i}(\boldsymbol{v}(t))\Vert\boldsymbol{v}(t)\Vert^L /\Vert\boldsymbol{v}(t)\Vert^L \tilde{q}_{\min}(\boldsymbol{v}(t))\right)+1
    \\
    &\leq\log _{2}\left(2^{L+1} B_{1} \rho(t)^L / \tilde{q}_{\min}\right) \le \log _{2}\left(2^{L+1} B_{1} / \tilde{\gamma}(t)\right)
    \\
    &\leq \log _{2}\left(2^{L+1}e^{\frac{1}{2}} B_{1} / \tilde{\gamma}(t_1)\right),
\end{align*}
which further leads to 
\begin{equation*}
    f^{\prime}\left(\tilde{q}_{i}(\boldsymbol{v}(t))\right) \le K^{ \log _{2}\left(2e^{\frac{1}{2}} B_{1} / \tilde{\gamma}(t_1)\right)} f^{\prime}\left(\xi_{i}(t)\right),
\end{equation*}
and eq. $(**)$ is because $e^{-x}x\le e^{-1}$.

The proof is completed.
\end{proof}

By Lemma \ref{lem:assymptotic_loss}, we have proved that $\lim_{t\rightarrow\infty} \tilde{\mathcal{L}}(t)=0$, which leads to $\lim_{t\rightarrow\infty} \delta(t)=0$. As stated in the main text, we only need to bound $\varepsilon(t)$, or equivalently $(1-\cos(\boldsymbol{\theta}(t)))$. 

Before moving forward, we introduce an equivalent proposition of that  $(1-\cos(\boldsymbol{\theta}))$ goes to zero.
\begin{lemma}
\label{lem:equivalence of b}
If there exists a time sequence $\{t_i\}_{i=1}^\infty$, such that, $\lim_{i\rightarrow\infty} t_i=\infty$ and  $\lim_{i\rightarrow\infty}\left\langle\widehat{\boldsymbol{\beta}^{-\frac{1}{2}}\odot\boldsymbol{v}},   \widehat{\boldsymbol{\beta}^{\frac{1}{2}}\odot\bar{\partial \tilde{\mathcal{L}}}}\right\rangle=-1$, then
\begin{equation*}
    \lim_{i\rightarrow \infty} \cos(\boldsymbol{\theta}(t_i))= 1.
\end{equation*}
\end{lemma}
\begin{proof}
\begin{align*}
    &-\cos(\boldsymbol{\theta}(t))
    \\
    =&\left\langle\widehat{\boldsymbol{\beta}(t)^{-\frac{1}{2}}\odot\boldsymbol{v}(t)}+\left( \hat{\boldsymbol{v}}(t)-\widehat{\boldsymbol{\beta}(t)^{-\frac{1}{2}}\odot\boldsymbol{v}}(t)\right),\widehat{\boldsymbol{\beta}(t)^{\frac{1}{2}}\odot\partial^s \tilde{\mathcal{L}}(t)}+\left(\widehat{\partial^s \tilde{\mathcal{L}}(t)}-\widehat{\boldsymbol{\beta}(t)^{\frac{1}{2}}\odot\partial^s \tilde{\mathcal{L}}(t)}\right) \right\rangle
    \\
    =&\left\langle\widehat{\boldsymbol{\beta}(t)^{-\frac{1}{2}}\odot\boldsymbol{v}(t)},\widehat{\boldsymbol{\beta}(t)^{\frac{1}{2}}\odot\partial^s \tilde{\mathcal{L}}(t)} \right\rangle
    +
    \left\langle\widehat{\boldsymbol{\beta}(t)^{-\frac{1}{2}}\odot\boldsymbol{v}(t)},\left(\widehat{\partial^s \tilde{\mathcal{L}}(t)}-\widehat{\boldsymbol{\beta}(t)^{\frac{1}{2}}\odot\partial^s \tilde{\mathcal{L}}(t) }\right) \right\rangle
    \\
    +&
    \left\langle\left( \hat{\boldsymbol{v}}(t)-\widehat{\boldsymbol{\beta}(t)^{-\frac{1}{2}}\odot\boldsymbol{v}(t)}\right),\widehat{\boldsymbol{\beta}(t)^{\frac{1}{2}}\odot\partial^s \tilde{\mathcal{L}}(t)} \right\rangle
    \\
    +&
    \left\langle\left( \hat{\boldsymbol{v}}(t)-\widehat{\boldsymbol{\beta}(t)^{-\frac{1}{2}}\odot\boldsymbol{v}(t)}\right),\left(\widehat{\partial^s \tilde{\mathcal{L}}(t)}-\widehat{\boldsymbol{\beta}(t)^{\frac{1}{2}}\odot\partial^s \tilde{\mathcal{L}}(t)}\right) \right\rangle.
\end{align*}

Furthermore,
\begin{equation*}
    \lim_{i\rightarrow\infty} \left\Vert\widehat{\partial^s \tilde{\mathcal{L}}(t)}-\widehat{\boldsymbol{\beta}(t)^{\frac{1}{2}}\odot\partial^s \tilde{\mathcal{L}}(t)}\right\Vert=2-2\lim_{i\rightarrow\infty}\left\langle\widehat{\partial^s \tilde{\mathcal{L}}(t)},\widehat{\boldsymbol{\beta}(t)^{\frac{1}{2}}\odot\partial^s \tilde{\mathcal{L}}(t)}\right\rangle=0.
\end{equation*}

Following the same routine, we have $\lim_{i\rightarrow\infty} \left\Vert\hat{\boldsymbol{v}}(t)-\widehat{\boldsymbol{\beta}(t)^{-\frac{1}{2}}\odot\boldsymbol{v}(t)}\right\Vert=0$.

The proof is completed.
\end{proof}

Let $\cos(\tilde{\boldsymbol{\theta}}(t))=-\left\langle\widehat{\boldsymbol{\beta}(t)^{-\frac{1}{2}}\odot\boldsymbol{v}(t)},   \widehat{\boldsymbol{\beta}(t)^{\frac{1}{2}}\odot\partial^s \tilde{\mathcal{L}}(t)}\right\rangle=\left\langle\widehat{\boldsymbol{\beta}(t)^{-\frac{1}{2}}\odot\boldsymbol{v}(t)},   \widehat{\boldsymbol{\beta}(t)^{-\frac{1}{2}}\odot\frac{\mathrm{d} \boldsymbol{v}(t)}{\mathrm{d}t}}\right\rangle$. Then we only need to bound $1-\cos(\tilde{\boldsymbol{\theta}}(t))$.

In Section \ref{sec:convergent_KKT}, we briefly state the methodology of proving the convergence of  $\varepsilon(t)$. We will make it more clear here: In the proof Lemma \ref{lem:lower_bound_margin}, we show that sum of the derivative of $g(\log \frac{1}{\tilde{\mathcal{L}}(t)})$ and $\frac{1}{\rho(t)^L}$ with $\boldsymbol{\beta}(t)$ fixed can be bounded as
\begin{align}
\nonumber
    &\frac{\mathrm{d}}{\mathrm{d} t}\left(\log \left(g\left(\log \frac{1}{\tilde{\mathcal{L}}(t)}\right)\right)\right)-L \left\langle\partial^s_{\boldsymbol{v}}\log \rho(t),\frac{\mathrm{d}\boldsymbol{v}(t)}{\mathrm{d}t}\right\rangle
\\
\nonumber
=&\frac{g^{\prime}\left(\log \frac{1}{\tilde{\mathcal{L}}(t)}\right)}{g\left(\log \frac{1}{\tilde{\mathcal{L}}(t)}\right)} \cdot \frac{1}{\tilde{\mathcal{L}}(t)} \cdot\left(-\frac{\mathrm{d} \tilde{\mathcal{L}}(t)}{\mathrm{d} t}\right)-L^{2} \cdot \frac{\nu(t)}{\rho(t)^{2}}
\\
\nonumber
\ge &\frac{1}{\nu(t)}\left(\left\langle \frac{\mathrm{d} \boldsymbol{v}(t)}{\mathrm{d}t},\boldsymbol{\beta}(t)^{-1}\odot\frac{\mathrm{d} \boldsymbol{v}(t)}{\mathrm{d}t} \right\rangle-\left\langle\boldsymbol{\beta}(t)^{-\frac{1}{2}}\odot \frac{\mathrm{d} \boldsymbol{v}(t)}{\mathrm{d}t},\widehat{\boldsymbol{\beta}(t)^{-\frac{1}{2}}\odot \boldsymbol{v}(t)} \right\rangle^2\right) 
\\
\nonumber
=&\frac{\left\langle\boldsymbol{\beta}(t)^{-\frac{1}{2}}\odot \frac{\mathrm{d} \boldsymbol{v}(t)}{\mathrm{d}t},\widehat{\boldsymbol{\beta}(t)^{-\frac{1}{2}}\odot \boldsymbol{v}(t)} \right\rangle^2}{\nu(t)}(\cos(\tilde{\boldsymbol{\theta}}(t))^{-2}-1)
\\
\nonumber
=&\frac{L^2\nu(t)}{\rho(t)^2}(\cos(\tilde{\boldsymbol{\theta}}(t))^{-2}-1)
\\
\label{eq:reason_tilde_rho}
=&L \left\langle\partial^s_{\boldsymbol{v}}\log \rho(t), \frac{\mathrm{d}\boldsymbol{v}}{\mathrm{d}t}\right\rangle(\cos(\tilde{\boldsymbol{\theta}}(t))^{-2}-1).
\end{align}

Eq. (\ref{eq:reason_tilde_rho}) indicates that, intuitively, $(\cos(\tilde{\boldsymbol{\theta}}(t))^{-2}-1)$ can be bounded by the division of change of $\tilde{\gamma}$ to change of parameter part $\boldsymbol{v}$ in $\rho$. For this purpose, we define $\tilde{\rho}$ in Section \ref{sec:convergent_KKT} to describe the accumulated change of $\boldsymbol{v}$ in $\rho$. The following lemma describe the basic property of $\tilde{\rho}$ and its relationship with $\rho$.

\begin{lemma}
\label{lem:property_of_tilde_rho}
(1). The derivative of $\tilde\rho^2$ is as follows: 
\begin{equation*}
    \frac{1}{2}\frac{d \tilde\rho(t)^2}{\mathrm{d}t}= \langle\partial^s \tilde{\mathcal{L}}(\boldsymbol{v}(t)),\boldsymbol{v}(t) \rangle=L \nu(t);
\end{equation*}

(2). $\rho(t) $ satisfies that, for any $t^2>t^1>t_1$,
\begin{equation*}
    \rho(t_2)\ge e^{-\frac{1}{2}}\rho(t_1);
\end{equation*}

(3). For any $t>t_1$, $\sqrt{1-\frac{e^{\frac{1}{2}}}{2}}\le\frac{\rho(t)}{\tilde{\rho}(t)}\le \sqrt{1+\frac{e^{\frac{1}{2}}}{2}}$.
\end{lemma}

\begin{proof}
(1). can be directly verified similar to Lemma \ref{lem:derivative_rho}. As for (2).,
since $\frac{d \log \rho(t)}{\mathrm{d}t}=\frac{1}{2}\frac{\frac{d \rho(t)^2}{\mathrm{d}t}}{\rho(t)^2}$, we have that
\begin{align*}
    \frac{d \log \rho(t)}{\mathrm{d}t}&=\frac{L \sum_{i=1}^{N} e^{-f\left(\tilde{q}_{i}(\boldsymbol{v}(t))\right)}f^{\prime}\left(\tilde{q}_{i}(\boldsymbol{v}(t))\right) \tilde{q}_{i}(\boldsymbol{v}(t))+\left\langle\boldsymbol{\beta}(t)^{-\frac{1}{2}}\odot\boldsymbol{v}(t), \boldsymbol{v}(t)\odot\frac{\mathrm{d}\boldsymbol{\beta}(t)^{-\frac{1}{2}}}{\mathrm{d}t}\right\rangle}{\rho(t)^2}
    \\
    &\ge \frac{\left\langle\boldsymbol{\beta}(t)^{-\frac{1}{2}}\odot\boldsymbol{v}(t), \boldsymbol{v}(t)\odot\frac{\mathrm{d}\boldsymbol{\beta}(t)^{-\frac{1}{2}}}{\mathrm{d}t}\right\rangle}{\rho(t)^2}
    \\
    &\ge  \frac{\sum_{i=1}^p \boldsymbol{\beta}_i(t)^{-\frac{1}{2}}(t)\frac{\mathrm{d}\boldsymbol{\beta}_i^{-\frac{1}{2}}(t)}{\mathrm{d}t}\boldsymbol{v}_{i}(t)^2}{\sum_{i=1}^p \boldsymbol{\beta}_i^{-1}(t) \boldsymbol{v}_{i}(t)^2}
    \\
    &=  \frac{\sum_{i=1}^p \boldsymbol{\beta}_i^{\frac{1}{2}}(t)\frac{\mathrm{d}\boldsymbol{\beta}_i^{-\frac{1}{2}}(t)}{\mathrm{d}t}\boldsymbol{\beta}_i^{-1}(t)\boldsymbol{v}_{i}(t)^2}{\sum_{i=1}^p \boldsymbol{\beta}_i^{-1}(t) \boldsymbol{v}_{i}(t)^2}
    \\
    &= \frac{\sum_{i=1}^p \frac{d \log\boldsymbol{\beta}_i^{-\frac{1}{2}}(t)}{\mathrm{d}t}\boldsymbol{\beta}_i^{-1}(t)\boldsymbol{v}_{i}(t)^2}{\sum_{i=1}^p \boldsymbol{\beta}_i^{-1}(t) \boldsymbol{v}_{i}(t)^2}
    \\
    &\ge \sum_{i=1}^p \left(\frac{d \log\boldsymbol{\beta}_i^{-\frac{1}{2}}(t)}{\mathrm{d}t}\right)_{-}.
\end{align*}
The proof for (2). is completed by integration.

As for (3)., by expanding $\langle \boldsymbol{v}(t),\boldsymbol{\beta}^{-\frac{1}{2}}(t)\odot \frac{\mathrm{d}\boldsymbol{\beta}^{-\frac{1}{2}}(t)}{\mathrm{d} t}\odot \boldsymbol{v}(t) \rangle$, we have that 
\begin{align*}
    \tilde{\rho}(t)^2=&\int_{t_1}^t \langle \boldsymbol{v}(\tau),\boldsymbol{\beta}^{-\frac{1}{2}}(\tau)\odot \frac{\mathrm{d}\boldsymbol{\beta}^{-\frac{1}{2}}(\tau)}{\mathrm{d} \tau}\odot \boldsymbol{v}(\tau) \rangle \mathrm{d} \tau\\
    =&\sum_{i=1}^p\int_{t_1}^t  \boldsymbol{v}_i^2(\tau)\boldsymbol{\beta}_i^{-1}(\tau)\boldsymbol{\beta}_i^{\frac{1}{2}}(\tau) \frac{\mathrm{d}\boldsymbol{\beta}_i^{-\frac{1}{2}}(\tau)}{\mathrm{d}t}  \mathrm{d} \tau
    \\
    \le &\sum_{i=1}^p\int_{t_1}^t  \boldsymbol{\beta}_i^{-1}(\tau)\boldsymbol{v}_i^2(\tau)\left(\boldsymbol{\beta}_i^{\frac{1}{2}}(\tau) \frac{\mathrm{d}\boldsymbol{\beta}_i^{-\frac{1}{2}}(\tau)}{\mathrm{d}t} \right)_{+}\mathrm{d} \tau
    \\
    \le&\int_{t_1}^t \left(\sum_{i=1}^p \boldsymbol{\beta}_i^{-1}\boldsymbol{v}_i^2(\tau)\right)\left(\sum_{i=1}^p\left(\boldsymbol{\beta}_i^{\frac{1}{2}}(\tau) \frac{\mathrm{d}\boldsymbol{\beta}_i^{-\frac{1}{2}}(\tau)}{\mathrm{d}t}\right)_{+}\right) \mathrm{d} \tau
    \\
    \overset{(*)}{=}& \Vert \boldsymbol{\beta}(\tau_0)^{-\frac{1}{2}}\odot\boldsymbol{v}(\tau_0)\Vert^2 \int_{t_1}^t \left(\sum_{i=1}^p\left(\boldsymbol{\beta}_i^{\frac{1}{2}}(\tau) \frac{\mathrm{d}\boldsymbol{\beta}_i^{-\frac{1}{2}}(\tau)}{\mathrm{d}t}\right)_{+}\right) \mathrm{d} \tau
    \\
    \le&\frac{1}{4}\Vert \boldsymbol{\beta}(\tau_0)^{-\frac{1}{2}}\odot\boldsymbol{v}(\tau_0)\Vert^2 
    \\
    \le&\frac{e^{\frac{1}{2}}}{4} \Vert\boldsymbol{\beta}(t)^{-\frac{1}{2}}\odot\boldsymbol{v}(t)\Vert^2,
\end{align*}
where $\tau_0$ in eq. $(*)$ is in $[t_1,t]$ and First Mean Value Theorem guarantees its existence.

The second inequality follows by lower bound $\int_{t_1}^t \langle \boldsymbol{v}(\tau),\boldsymbol{\beta}^{-\frac{1}{2}}(\tau)\odot \frac{\mathrm{d}\boldsymbol{\beta}^{-\frac{1}{2}}(\tau)}{\mathrm{d} \tau}\odot \boldsymbol{v}(\tau) \rangle \mathrm{d} \tau$ similarly.

The proof is completed.
\end{proof}

With Lemma \ref{lem:property_of_tilde_rho}, we integrate the analysis of the change of $\tilde{\gamma}$ into the following Lemma.

\begin{lemma}
\label{lem: b_inte}
For any time $t_3>t_2\ge t_1$, 
\begin{align*}
    \int_{t_{2}}^{t_{3}}\left(\cos(\tilde{\boldsymbol{\theta}}(\tau))^{-2}-1\right) \cdot \frac{\mathrm{d}}{d \tau} \log \tilde{\rho}(\tau)  d \tau \leq& \frac{1}{L\left(1-\frac{e^{\frac{1}{2}}}{2}\right)} \log \frac{\tilde{\gamma}\left(t_{3}\right)}{\tilde{\gamma}\left(t_{2}\right)}
    \\
        +&\frac{L}{\left(1-\frac{e^{\frac{1}{2}}}{2}\right)}\sum_{i=1}^p\int_{t_2}^{t_3} \sum_{i=1}^p \left(\frac{ \mathrm{d}\log \boldsymbol{\beta}^{-\frac{1}{2}}_i(t)}{\mathrm{d}t}\right)_{+} \mathrm{d}t.
\end{align*}
\end{lemma}
\begin{proof}
Recall that 
\begin{equation*}
   \frac{d \log \tilde{\rho}(t)}{\mathrm{d}t}=\frac{L\nu(t)}{\tilde{\rho}(t)^2}=\frac{\left\langle\boldsymbol{v}(t),\boldsymbol{\beta}(t)^{-1}\odot\frac{\mathrm{d}\boldsymbol{v}(t)}{\mathrm{d}t}\right\rangle}{\tilde{\rho}(t)^2},
\end{equation*}
and 
\begin{align*}
    \frac{\mathrm{d}}{\mathrm{d}t}\log \tilde{\gamma}(t)\ge&\frac{1}{\nu(t)}\left(\left\langle \frac{\mathrm{d} \boldsymbol{v}(t)}{\mathrm{d}t},\boldsymbol{\beta}(t)^{-1}\odot\frac{\mathrm{d} \boldsymbol{v}(t)}{\mathrm{d}t} \right\rangle-\left\langle\boldsymbol{\beta}(t)^{-\frac{1}{2}}\odot \frac{\mathrm{d} \boldsymbol{v}(t)}{\mathrm{d}t},\widehat{\boldsymbol{\beta}(t)^{-\frac{1}{2}}\odot \boldsymbol{v}(t)} \right\rangle\right) 
    \\
    -&L \sum_{i=1}^p \left(\frac{ \mathrm{d}\log \boldsymbol{\beta}_i(t)^{-\frac{1}{2}}(t)}{\mathrm{d}t}\right)_{+}.
\end{align*}

We then have
\begin{align*}
     \frac{\mathrm{d}}{\mathrm{d}t}\log \tilde{\gamma}(t)\ge&\frac{1}{\nu(t)}\left(\left\langle \frac{\mathrm{d} \boldsymbol{v}(t)}{\mathrm{d}t},\boldsymbol{\beta}(t)^{-1}\odot\frac{\mathrm{d} \boldsymbol{v}(t)}{\mathrm{d}t} \right\rangle-\left\langle\boldsymbol{\beta}(t)^{-\frac{1}{2}}\odot \frac{\mathrm{d} \boldsymbol{v}(t)}{\mathrm{d}t},\widehat{\boldsymbol{\beta}(t)^{-\frac{1}{2}}\odot \boldsymbol{v}(t)} \right\rangle^2\right) 
    \\
    -&L \sum_{i=1}^p \left(\frac{ \mathrm{d}\log \boldsymbol{\beta}(t)^{-\frac{1}{2}}_i(t)}{\mathrm{d}t}\right)_{+}
    \\
    =&L\frac{\tilde{\rho}(t)^2}{L^2\nu(t)^2}\frac{d \log \tilde{\rho}(t)}{\mathrm{d}t}\left(\left\langle \frac{\mathrm{d} \boldsymbol{v}(t)}{\mathrm{d}t},\boldsymbol{\beta}(t)^{-1}\odot\frac{\mathrm{d} \boldsymbol{v}(t)}{\mathrm{d}t} \right\rangle-\left\langle\boldsymbol{\beta}(t)^{-\frac{1}{2}}\odot \frac{\mathrm{d} \boldsymbol{v}(t)}{\mathrm{d}t},\widehat{\boldsymbol{\beta}(t)^{-\frac{1}{2}}\odot \boldsymbol{v}(t)} \right\rangle^2\right) 
    \\
    -&L \sum_{i=1}^p \left(\frac{ \mathrm{d}\log \boldsymbol{\beta}(t)^{-\frac{1}{2}}_i(t)}{\mathrm{d}t}\right)_{+}
    \\
    \ge&L\left(1-\frac{e^{\frac{1}{2}}}{2}\right)\frac{\rho(t)^2}{L^2\nu(t)^2}\frac{d \log \tilde{\rho}(t)}{\mathrm{d}t}\left(\left\langle \frac{\mathrm{d} \boldsymbol{v}(t)}{\mathrm{d}t},\boldsymbol{\beta}(t)^{-1}\odot\frac{\mathrm{d} \boldsymbol{v}(t)}{\mathrm{d}t} \right\rangle\right.
    \\
    -&\left.\left\langle\boldsymbol{\beta}(t)^{-\frac{1}{2}}\odot \frac{\mathrm{d} \boldsymbol{v}(t)}{\mathrm{d}t},\widehat{\boldsymbol{\beta}^{-\frac{1}{2}}\odot \boldsymbol{v}(t)} \right\rangle^2\right) 
    \\
    -&L \sum_{i=1}^p \left(\frac{ \mathrm{d}\log \boldsymbol{\beta}^{-\frac{1}{2}}_i(t)}{\mathrm{d}t}\right)_{+}
    \\
    =&L\left(1-\frac{e^{\frac{1}{2}}}{2}\right)\frac{d \log \tilde{\rho}}{\mathrm{d}t}(\cos(\tilde{\boldsymbol{\theta}}(t))^{-2}-1)-L \sum_{i=1}^p \left(\frac{ \mathrm{d}\log \boldsymbol{\beta}^{-\frac{1}{2}}_i(t)}{\mathrm{d}t}\right)_{+}.
\end{align*}

The proof is completed.
\end{proof}

Applying the First Mean Value Theorem together with Lemmas \ref{lem:equivalence of b} and  \ref{lem: b_inte}, one can easily obtain the first inequality in Lemma \ref{lem: bound_b}. To make the following proof simpler, we restate Lemma \ref{lem: bound_b}  as the following corollary, while using $\cos(\tilde{\boldsymbol{\theta}}(t))$ instead of $\cos(\boldsymbol{\theta}(t))$.

\begin{corollary}[First inequality in Lemma \ref{lem: bound_b}, restated]
\label{coro:bound}
For any time $t_3>t_2\ge t_1$, there exists a time $\xi\in(t_2,t_3)$, such that,
\begin{align*}
    \left(\cos(\tilde{\boldsymbol{\theta}}(\xi))^{-2}-1\right) ( \log \tilde{\rho}(t_2) -\log \tilde{\rho}(t_1)) \leq& \frac{1}{L\left(1-\frac{e^{\frac{1}{2}}}{2}\right)} \log \frac{\tilde{\gamma}\left(t_{3}\right)}{\tilde{\gamma}\left(t_{2}\right)}
    \\
        +&\frac{L}{\left(1-\frac{e^{\frac{1}{2}}}{2}\right)}\sum_{i=1}^p\int_{t_2}^{t_3} \sum_{i=1}^p \left(\frac{ \mathrm{d}\log \boldsymbol{\beta}^{-\frac{1}{2}}_i(t)}{\mathrm{d}t}\right)_{+} \mathrm{d}t.
\end{align*}
\end{corollary}

We then prove the second inequality in Lemma \ref{lem: bound_b} to complete the proof of  Lemma \ref{lem: bound_b}.

\begin{proof}[Proof of Lemma \ref{lem: bound_b}]
By direct calculation, we have that 
\begin{equation*}
    \left\|\frac{d \hat{\boldsymbol{v}}(t)}{\mathrm{d} t}\right\| =\frac{1}{\Vert \boldsymbol{v}(t)\Vert}\left\|\left(\boldsymbol{I}-\hat{\boldsymbol{v}}(t) \hat{\boldsymbol{v}}(t)^{\top}\right) \frac{\mathrm{d} \boldsymbol{v}(t)}{\mathrm{d} t}\right\|  \leq \frac{1}{\Vert \boldsymbol{v}(t)\Vert}\left\|\frac{\mathrm{d} \boldsymbol{v}(t)}{\mathrm{d} t}\right\| \le \frac{2}{\Vert \boldsymbol{v}(t)\Vert}\left\|\boldsymbol{\beta}(t)^{-1}\odot\frac{\mathrm{d} \boldsymbol{v}(t)}{\mathrm{d} t}\right\| .
\end{equation*}

Furthermore,
\begin{align*}
   &\quad\left\|\boldsymbol{\beta}(t)^{-1}\odot\frac{\mathrm{d} \boldsymbol{v}(t)}{\mathrm{d} t}\right\|=\left\|\partial^s \tilde{\mathcal{L}}(t)\right\|
\\
   &\le\sum_{i \in[N]} e^{-f\left(\tilde{q}_{i}(\boldsymbol{v}(t))\right)} f^{\prime}\left(\tilde{q}_{i}(\boldsymbol{v}(t))\right)\left\|
   \partial^s \tilde{q}_i(\boldsymbol{v}(t))\right\| 
\\
   &\le \sum_{i \in[N]} e^{-f\left(\tilde{q}_{i}(\boldsymbol{v}(t))\right)} f^{\prime}\left(\tilde{q}_{i}(\boldsymbol{v}(t))\right)\tilde{q}_i(\boldsymbol{v}(t))\frac{1}{\tilde{q}_i(\boldsymbol{v}(t))}B_1\Vert \boldsymbol{v}(t)\Vert^{L-1}
   \\
   &\le \sum_{i \in[N]}2^{L-1} e^{-f\left(\tilde{q}_{i}(\boldsymbol{v}(t))\right)} f^{\prime}\left(\tilde{q}_{i}(\boldsymbol{v}(t))\right)\tilde{q}_i(\boldsymbol{v}(t))\frac{1}{\tilde{q}_i(\boldsymbol{v}(t))}B_1 \rho(t)^{L-1}
   \\
   &\le \sum_{i \in[N]}2^{L-1} e^{-f\left(\tilde{q}_{i}(\boldsymbol{v}(t))\right)} f^{\prime}\left(\tilde{q}_{i}(\boldsymbol{v}(t))\right)\tilde{q}_i(\boldsymbol{v}(t))\frac{1}{g\left(\log \frac{1}{\tilde{\mathcal{L}}(t)}\right)}B_1\rho(t)^{L-1}
   \\
    &= 2^{L-1}\frac{ \nu(t)}{\rho(t)} \frac{\rho(t)^L}{g\left(\log \frac{1}{\tilde{\mathcal{L}}(t)}\right)}B_1=2^{L-1}\frac{ \nu(t)}{\rho(t)} \frac{1}{\tilde{\gamma}(t)}B_1
\end{align*}

Therefore,
\begin{align*}
    \left\|\frac{d \hat{\boldsymbol{v}}(t)}{\mathrm{d} t}\right\| &\le 2^{L}\frac{ \nu(t)}{\rho(t)\Vert \boldsymbol{v}(t)\Vert} \frac{1}{\tilde{\gamma}(t)}B_1\le 3\cdot 2^{L-1} \frac{ \nu(t)}{\rho(t)^2} \frac{1}{\tilde{\gamma}(t)}B_1\le \frac{ 3\cdot 2^{L-1}}{(1-\frac{e^{\frac{1}{2}}}{2})}\frac{ \nu(t)}{\tilde{\rho}(t)^2} \frac{1}{\tilde{\gamma}(t)}B_1
    \\
    &= \frac{ 3\cdot 2^{L-1}}{(1-\frac{e^{\frac{1}{2}}}{2})}\frac{ B_1}{L}\frac{\mathrm{d}\log \tilde{\rho}(t)}{\mathrm{d}t} \frac{1}{\tilde{\gamma}(t)}\le\frac{ 3\cdot 2^{L-1}}{(1-\frac{e^{\frac{1}{2}}}{2})} \frac{ e^{\frac{1}{2}} B_1}{L}\frac{\mathrm{d}\log \tilde{\rho}(t)}{\mathrm{d}t} \frac{1}{\tilde{\gamma}(t_1)}.
\end{align*}

The proof is completed.
\end{proof}

Applying Lemma \ref{lem:construction_kkt} and Lemma \ref{lem: bound_b}, we can then prove Theorem \ref{thm:approximate_flow}.

\begin{proof}[Proof of Theorem \ref{thm:approximate_flow}]
Let $\bar{\boldsymbol{v}}$ be any limit point of series $\{\boldsymbol{v}(t)\}_{t=0}^{\infty}$. We construct a series of approximate KKT point which converges to $\bar{\boldsymbol{v}}$ by induction.

Let $t^1=t_1$. Now suppose $t^{k-1}$ has been constructed. By Lemma \ref{lem:conver_gamma} and that $\bar{\boldsymbol{v}}$ is a limit point,  there exists $s_k>t^{k-1}$ such that, for any $t>s_k$
\begin{equation*}
    \left\|\hat{\boldsymbol{v}}\left(s_{k}\right)-\bar{\boldsymbol{v}}\right\|_{2} \leq \frac{1}{k}, \text{ }\frac{2}{L}\log \frac{\tilde{\gamma}(t)}{\tilde{\gamma}(s_k)}\le \frac{1}{2k^3}, \text{ and}, L\sum_{i=1}^p\int_{s_m}^\infty\left(\frac{\mathrm{d}\log \boldsymbol{\beta}_i(t)^{-\frac{1}{2}}}{\mathrm{d}t}\right)_{+}\mathrm{d}t\le \frac{1}{2k^3}.
\end{equation*}

Let $s'_k>s_k$ satisfy $\log \tilde\rho\left(s_{k}^{\prime}\right)=\log \tilde\rho\left(s_{k}\right)+\frac{1}{k}$ (which is guaranteed as $\lim_{t\rightarrow \infty} \rho=\infty$ and Lemma \ref{lem:property_of_tilde_rho}). Therefore, by Corollary \ref{coro:bound}, there exists $t^{k} \in\left(s_{k}, s_{k}^{\prime}\right)$, such that
\begin{equation}
\label{eq: bound_b}
    \cos(\tilde{\boldsymbol{\theta}}\left(t^{k}\right))^{-2}-1 \leq \frac{1}{k^2(1-\frac{e^{\frac{1}{2}}}{2})}.
\end{equation}

Furthermore,
\begin{equation*}
    \left\|\hat{\boldsymbol{v}}\left(t^{k}\right)-\bar{\boldsymbol{v}}\right\|_{2} \leq\left\|\hat{\boldsymbol{v}}\left(t^{k}\right)-\hat{\boldsymbol{v}}\left(s_{k}\right)\right\|_{2}+\left\|\hat{\boldsymbol{v}}\left(s_{k}\right)-\overline{\boldsymbol{v}}\right\|_{2} \leq  \frac{ 3\cdot 2^{L-1}}{(1-\frac{e^{\frac{1}{2}}}{2})}\frac{ e^{\frac{1}{2}}}{L} \frac{1}{\tilde{\gamma}(t_1)} \frac{1}{k}+\frac{1}{k} \rightarrow 0.
\end{equation*}

Therefore, $\lim_{t\rightarrow\infty}\hat{\boldsymbol{v}}(t^k)=\bar{\boldsymbol{v}}$. Furthermore, by eq. (\ref{eq: bound_b}) and Lemma \ref{lem: construction of KKT_appen}, we have that $\boldsymbol{v}(t^k)/\tilde{q}_{\min}^{\frac{1}{L}}(\boldsymbol{v}(t^k))$ is an ($\varepsilon_i$, $\delta_i$) KKT point with $\lim_{i\rightarrow\infty} \varepsilon_i=\lim_{i\rightarrow\infty} \delta_i=0$. Since 
\begin{equation*}
    \boldsymbol{v}(t^k)/\tilde{q}_{\min}^{\frac{1}{L}}(\boldsymbol{v}(t^k))=\hat{\boldsymbol{v}}(t^k)/\gamma(t^k),
\end{equation*}
and $\gamma(t^k)$ converges to a positive number, we further have
\begin{equation*}
    \lim_{k\rightarrow \infty}\boldsymbol{v}(t^k)/\tilde{q}_{\min}^{\frac{1}{L}}(\boldsymbol{v}(t^k))=\bar{\boldsymbol{v}}/\lim_{k\rightarrow\infty} \gamma(t^k)^{\frac{1}{L}},
\end{equation*}
is a KKT point of $(P)$, and along the same direction of $\bar{\boldsymbol{v}}$.

The proof is completed.
\end{proof}

\subsection{Convergent Direction of AdaGrad, RMSProp and Adam (w/m): proof of Theorems \ref{thm:AdaGrad_flow} and \ref{thm:RMS_flow}}
\label{sec:appen_direction_adagrad_rmsp}
First of all, we prove Theorem \ref{thm:AdaGrad_flow} by substitute $\boldsymbol{v}$ in Theorem \ref{thm:approximate_flow} with $\boldsymbol{h}_{\infty}\odot \boldsymbol{w}$. 
\begin{proof}[Proof of Theorem \ref{thm:AdaGrad_flow}]
    Let $\bar{\boldsymbol{w}}$ be any limit point of series $\{\hat{\boldsymbol{w}}(t)\}_{t=0}^{\infty}$. Since $\boldsymbol{v}(t)=\boldsymbol{h}^{-\frac{1}{2}}_{\infty}\odot \boldsymbol{w}(t)$, $\widehat{\boldsymbol{h}^{-\frac{1}{2}}_{\infty}\odot\bar{\boldsymbol{w}}}$ is a limit point of $\{\hat{\boldsymbol{v}}(t)\}$. By Theorem \ref{thm:approximate_flow}, $\widehat{\boldsymbol{h}^{-\frac{1}{2}}_{\infty}\odot\bar{\boldsymbol{w}}}/\tilde{q}_{\min}\left(\widehat{\boldsymbol{h}^{-\frac{1}{2}}_{\infty}\odot\bar{\boldsymbol{w}}}\right)^{1/L}=\boldsymbol{h}^{-\frac{1}{2}}_{\infty}\odot\bar{\boldsymbol{w}}/q_{\min}(\bar{\boldsymbol{w}})^{1/L}$ is a KKT point of $(\tilde{P})$. Therefore, there exist non-negative reals $\{\lambda_i\}_{i=1}^N$, such that
\begin{gather*}
    \boldsymbol{h}^{-\frac{1}{2}}_{\infty}\odot\bar{\boldsymbol{w}}/q_{\min}(\bar{\boldsymbol{w}})^{1/L}=\sum_{i=1}^N \lambda_i \partial^s \tilde{q}_{i}(\boldsymbol{h}^{-\frac{1}{2}}_{\infty}\odot\bar{\boldsymbol{w}}/q_{\min}(\bar{\boldsymbol{w}})^{1/L}),\\
    \sum_{i=1}^N \lambda_i (\tilde{q}_{i}(\boldsymbol{h}^{-\frac{1}{2}}_{\infty}\odot\bar{\boldsymbol{w}}/q_{\min}(\bar{\boldsymbol{w}})^{1/L})-1)=0.
\end{gather*}

Applying the relationship between $\tilde{q}_i$ and $q_i$, we then have
\begin{gather*}
    \boldsymbol{h}^{-1}_{\infty}\odot\bar{\boldsymbol{w}}/q_{\min}(\bar{\boldsymbol{w}})^{1/L}=\sum_{i=1}^N \lambda_i \partial^s q_{i}(\bar{\boldsymbol{w}}/q_{\min}(\bar{\boldsymbol{w}})^{1/L}),\\
    \sum_{i=1}^N \lambda_i (q_{i}(\bar{\boldsymbol{w}}/q_{\min}(\bar{\boldsymbol{w}})^{1/L})-1)=0.
\end{gather*}

The proof is completed.
\end{proof}

Theorem \ref{thm:RMS_flow} can be obtained in the same way.
\begin{proof}[Proof of Theorem \ref{thm:RMS_flow}] 
The claim holds since 
$\boldsymbol{v}^R$ is just $\boldsymbol{w}$ with a positive scaling factor, $\boldsymbol{v}^R$ and $\boldsymbol{w}$ share the same direction.

The proof is completed.

\end{proof}

\subsection{Convergence Rate of Empirical Loss and Parameter Norm}
\label{sec:appen_convergence_rate}
In the end of this section, we will give a tight bound for the convergence rate of empirical loss and parameter norm, which is derived by estimating $G(x)$ in Lemma \ref{lem:estimation_loss_G}. These results will further be used in Appendix \ref{sec:proof_definable}.
\begin{theorem}
Let $\boldsymbol{v}$ obey an adaptive gradient flow $\mathcal{F}$ with empirical loss $\tilde{L}$ satisfying Assumption \ref{assum: continuous}. Let $G(x)$ be defined as Lemma \ref{lem:estimation_loss_G}, i.e.,
\begin{equation*}
    G(x)= \int_{\frac{1}{\tilde{\mathcal{L}}(t_1)}}^{x}  \frac{g^{\prime}\left(\log z\right)^{2}}{g\left(\log z\right)^{2-2/ L}} \cdot  d z.
\end{equation*}
Then $G(x)=\Theta(x(\log x)^{\frac{2}{L}-2})$, and $G^{-1}(x)=\Theta(x(\log x)^{2-\frac{2}{L}})$. Consequently,
\begin{equation*}
    \tilde{\mathcal{L}}(t)=\Theta \left(\frac{1}{t\log t^{2-\frac{2}{L}}}\right), \text{ and }\Vert \boldsymbol{v}(t)\Vert=\Theta \left(\frac{1}{\log t^{\frac{1}{L}}}\right).
\end{equation*}
\end{theorem}

\begin{proof}
 For any large enough $x$, 
 \begin{align*}
     G(x)&= \int_{\frac{1}{\tilde{\mathcal{L}}(t_1)}}^{x}  \frac{g^{\prime}\left(\log z\right)^{2}}{g\left(\log z\right)^{2-2/ L}} \cdot  d z
     \\
     &\overset{(*)}{=} \Theta\left(\int_{\frac{1}{\tilde{\mathcal{L}}(t_1)}}^{x}  \frac{g\left(\log z\right)^{2/ L}}{(\log z)^2} \cdot  d z\right)
     \\
      &= \Theta\left(\int_{\frac{1}{\tilde{\mathcal{L}}(t_1)}}^{x}  \frac{1}{(\log z)^{2-\frac{2}{L}}} \cdot  d z\right)
      \\
      &= \Theta\left(x (\log x)^{\frac{2}{L}-2}\right),
 \end{align*}
 where eq. $(*)$ is due to Proposition \ref{lem:property_loss}.
 
 Since $G(x)$ is monotonously increasing, and $\lim_{x\rightarrow\infty}G(x)=\infty$, we have $x=\Theta\left(G^{-1}(x) (\log G^{-1}(x))^{\frac{2}{L}-2}\right)$, which further leads to $G^{-1}(x)=\Theta(x(\log x)^{2-\frac{2}{L}})$.

 Since $\frac{1}{\tilde{\mathcal{L}}(t)}=G^{-1}(\Omega(t))$, we have 
 $\frac{1}{\tilde{\mathcal{L}}(t)}=\Omega( t(\log t)^{2-\frac{2}{L}})$, which further leads to $\tilde{\mathcal{L}}(t)=\mathcal{O}\left(\frac{1}{ t(\log t)^{2-\frac{2}{L}}}\right)$.
 
 On the other hand,
 \begin{equation*}
 -\frac{\mathrm{d} \tilde{\mathcal{L}}(\boldsymbol{v}(t))}{\mathrm{d} t}=\left\|\boldsymbol{\beta}^{\frac{1}{2}}(t)\odot\partial^s 
 \tilde{\mathcal{L}}(t)\right\|_{2}^{2}\le 2\left\|\partial^s 
 \tilde{\mathcal{L}}(t)\right\|_{2}^{2}= \mathcal{O}\left(\tilde{\mathcal{L}}(\boldsymbol{v}(t))g\left(\log \frac{1}{\tilde{\mathcal{L}}(t)}\right)^{1-\frac{1}{L}}\right),
 \end{equation*}
 which leads to $\frac{1}{\tilde{\mathcal{L}}(t)}=G^{-1}(\mathcal{O}(t))$, and $\tilde{\mathcal{L}}(t)=\Omega\left(\frac{1}{ t(\log t)^{2-\frac{2}{L}}}\right)$. Therefore, $\tilde{\mathcal{L}}(t)=\Theta\left(\frac{1}{ t(\log t)^{2-\frac{2}{L}}}\right)$
 
 By Lemma \ref{lem:appen_bound_gamma}, $\Vert \boldsymbol{v}(t) \Vert^L=\Theta \left(\log \frac{1}{\tilde{\mathcal{L}}(t)}\right)$, which leads to $\Vert\boldsymbol{v}(t)\Vert=\Theta \left(\frac{1}{\log t^{\frac{1}{L}}}\right)$.
 
 The proof is completed.
\end{proof}

The convergent behavior of $\partial^s \tilde{\mathcal{L}}(\boldsymbol{v}(t))$ can be derived immediately by the above Theorem.
\begin{corollary}
Let $\boldsymbol{v}$ obey an adaptive gradient flow $\mathcal{F}$ with empirical loss $\tilde{L}$ satisfying Assumption \ref{assum: continuous}. Then, $\Vert\partial^s \tilde{\mathcal{L}}(\boldsymbol{v}(t))\Vert=\Theta\left(\frac{1}{t(\log t)^{1-\frac{1}{L}}}\right) $.
\end{corollary}
\begin{proof}
 Since 
 \begin{equation*}
     \Omega\left(\tilde{\mathcal{L}}(\boldsymbol{v}(t))\Vert \boldsymbol{v}(t)\Vert^{L-1}\right)=\langle\partial^s \tilde{\mathcal{L}}(\boldsymbol{v}(t)), \hat{\boldsymbol{v}}(t)\rangle\le\Vert\partial^s \tilde{\mathcal{L}}(\boldsymbol{v}(t))\Vert=\mathcal{O}\left(\tilde{\mathcal{L}}(\boldsymbol{v}(t))\Vert \boldsymbol{v}(t)\Vert^{L-1}\right),
 \end{equation*}
 the proof is completed.
 
\end{proof}

\input{Definable}

\section{Proof for the Discrete Case}
\label{sec:discrete case}
We prove the result for AdaGrad and experiential loss, with the result for RMSProp and logistic loss follows exactly as the continuous case. We slightly change the order of four stages in the flow: First, in Section \ref{sec:discrete_learning rate}, we prove that the conditioner has a limit with no zero entry; secondly, in Section \ref{sec:loss_discrete}, we prove that the empirical loss converges to zero;  then, in Section \ref{sec:discrete_surrogate margin}, we construct a further smoothed approximate margin, and prove it has a lower bound; finally, in Section \ref{sec:verification_discrete}, we prove that every limit point of AdaGrad is along some KKT point of optimization problem $(P^A)$ defined in Theorem \ref{thm:AdaGrad_flow}.

\subsection{Convergence of conditioners}
\label{sec:discrete_learning rate}
Before the proof, we give a formal definition of the learning rate bound $C(t)$ in Assumption \ref{assum:discrete}: let $M$ be the smooth constant in Assumption\ \ref{assum:discrete}. I. Then, $C(t)=\max\{\min_{i}\{\boldsymbol{h}^A_i(t)^{-1}\}/M,1,\frac{C^2_1}{2LNe^{-1}}\}$, where $C_0$ will be clear below. By the monotony of $\boldsymbol{h}^A_i(t)$, apparently $C(t)$ is non-decreasing. Now we can prove $\sum_{t=1}^\infty \partial^s \mathcal{L}(\boldsymbol{w}(t))^2<\infty$.

\begin{lemma}
\label{lem:conver_learning_rate_adagrad_discrete}
Suppose $\mathcal{L}$ is $M$ smooth with respect to $\boldsymbol{w}$. Then, for $\{\partial^s \mathcal{L}(\boldsymbol{w}(t))\}_{t=1}^{\infty}$ updated by AdaGrad (eq. (\ref{eq:update})), $\sum_{t=1}^\infty \partial^s \mathcal{L}(\boldsymbol{w}(t))^2<\infty$.
\end{lemma}

\begin{proof}
For any $t>t_0$,
\begin{align*}
    \mathcal{L}(\boldsymbol{w}(t))-\mathcal{L}(\boldsymbol{w}(t+1))&=\eta_t\langle \partial^s \mathcal{L}(\boldsymbol{w}(t)),\boldsymbol{h}^A(t)\odot\partial^s \mathcal{L}(\boldsymbol{w}(t))\rangle-\frac{M}{2}\eta_t^2\Vert \boldsymbol{h}^A(t)\odot\partial^s \mathcal{L}(\boldsymbol{w}(t))\Vert^2
    \\
    &\ge \eta_t\frac{1}{2}\langle \partial^s \mathcal{L}(\boldsymbol{w}(t)),\boldsymbol{h}^A(t)\odot\partial^s \mathcal{L}(\boldsymbol{w}(t))\rangle
    \\
    &\ge\tilde{\eta} \frac{1}{2}\langle \partial^s \mathcal{L}(\boldsymbol{w}(t)),\boldsymbol{h}^A(t)\odot\partial^s \mathcal{L}(\boldsymbol{w}(t))\rangle.
\end{align*}

Thus, since $\sum_{t=1}^\infty a_t$ share the same convergent behavior with $\sum_{t=1}^\infty \frac{a_t}{\sum_{\tau=1}^t a_{\tau}}$ ($a_t\ge 0$), by similar routine of Lemma \ref{lem:adagrad_flow_rate_conver}, the proof is completed.
\end{proof}

Therefore,  $\boldsymbol{h}_{\infty}\overset{\triangle}{=}\lim_{t\rightarrow\infty} \boldsymbol{h}^A(t)$ has no zero entry. We can then define a discrete version of adaptive gradient flow as 
\begin{gather*}
    \boldsymbol{v}(t)=\boldsymbol{h}_{\infty}^{-1 / 2} \odot \boldsymbol{w}(t),\\
    \boldsymbol{\beta}(t)=\boldsymbol{h}_{\infty}^{-1} \odot \boldsymbol{h}(t),\\
    \tilde{\mathcal{L}}(\boldsymbol{v})=\mathcal{L}(\boldsymbol{h}_{\infty}^{\frac{1}{2}}\odot \boldsymbol{v}),
\end{gather*}
and 
\begin{equation*}
    \tilde{q}_i(\boldsymbol{v})=q_i(\boldsymbol{h}_{\infty}^{\frac{1}{2}}\odot\boldsymbol{v}),
\end{equation*}
which further leads to
\begin{equation}
\label{eq:adagrad_normalized_discrete}
    \boldsymbol{v}(t+1)-\boldsymbol{v}(t)=- \boldsymbol{\beta}(t) \odot \partial^s \mathcal{L}_{i n d}(\boldsymbol{v}(t)),
\end{equation}
and $\boldsymbol{\beta}(t)$ decreases component-wisely to $\mathbf{1}_p$.

By Lemma \ref{lem:conver_learning_rate_adagrad_discrete}, for any $t\ge t_0$, $\mathcal{L}(\boldsymbol{w}(t))\le \mathcal{L}(\boldsymbol{w}(t_0))<N$. Therefore, there exists a positive real constant $C_0$ only depending on $\mathcal{L}(\boldsymbol{w}(t_0))$, such that, $\Vert\boldsymbol{w}(t)\Vert\ge C_0$. Furthermore, since $\boldsymbol{h}_{\infty}^{-\frac{1}{2}}\ge\boldsymbol{h}(t_0)^{-\frac{1}{2}} $, $\Vert\boldsymbol{v}(t)\Vert\ge C_0 \max_i\{\boldsymbol{h}_i(t_0)^{-\frac{1}{2}}\}$. Define $C_1= C_0 \max_i\{\boldsymbol{h}_i(t_0)^{-\frac{1}{2}}\}$, which only depends on $\mathcal{L}(\boldsymbol{w}(t_0))$ and $\partial^s \mathcal{L}(\boldsymbol{w}(t))$ $(t\le t_0)$.

Moreover, similar to approximate flow, there exists a time $t_1$, such that, for any time $t\ge t_1$,
\begin{gather*}
    \sum_{i=1}^p \log \frac{1}{\boldsymbol{\beta}_i(t)^{-\frac{1}{2}}}\le \frac{1}{2},
\\
    \Vert\boldsymbol{\beta}^{-1}(t)\Vert\ge \sqrt{\frac{1}{2}}.
\end{gather*}

\subsection{Convergence of Empirical Loss}
\label{sec:loss_discrete}

Define function $\tilde{\gamma}'$ as the rate of $\log \frac{1}{\mathcal{L}}$ to $\Vert \boldsymbol{v}\Vert^{\frac{L}{4}}$:
\begin{equation*}
    \tilde{\gamma}'(t)=\frac{\log \frac{1}{\tilde{\mathcal{L}}(\boldsymbol{v}(t))}}{\Vert \boldsymbol{v}(t)\Vert^{\frac{L}{4}}}.
\end{equation*}
 
Then, we have the following lemma.
\begin{lemma}
\label{lem:tilde_gamma'}
For any $t\ge t_1$, $\tilde{\gamma}'(t)$ is non-decreasing.
\end{lemma}
\begin{proof}
    Since by Lemma \ref{lem:conver_learning_rate_adagrad_discrete}, $\tilde{\mathcal{L}}$ is non-increasing, if $\Vert \boldsymbol{v}(t+1)\Vert\le \Vert \boldsymbol{v}(t)\Vert$,  the proposition trivially holds. Therefore, we consider the case that $\Vert \boldsymbol{v}(t+1)\Vert> \Vert \boldsymbol{v}(t)\Vert$ in the following proof.
    
    The change of $\Vert \boldsymbol{v}(t)\Vert$ can be calculated as 
    \begin{align*}
    &\Vert \boldsymbol{v}(t+1)\Vert^2-\Vert \boldsymbol{v}(t)\Vert^2
    \\
    =&-\eta_t\langle \boldsymbol{\beta}(t)\odot \partial^s \tilde{\mathcal{L}}(\boldsymbol{v}(t)), \boldsymbol{v}(t)\rangle+ \eta_t^2\Vert\boldsymbol{\beta}(t)\odot\partial^s \tilde{\mathcal{L}} (\boldsymbol{v}(t))  \Vert^2.
\end{align*}

Let $A=-\eta_t\langle \boldsymbol{\beta}(t)\odot \partial^s \tilde{\mathcal{L}}(\boldsymbol{v}(t)), \boldsymbol{v}(t)\rangle$, and $B=\eta_t^2\Vert\boldsymbol{\beta}(t)\odot\partial^s \tilde{\mathcal{L}} (\boldsymbol{v}(t))  \Vert^2$. We estimate them separately.

As for $A$:
\begin{align*}
    A&=-\eta_t\langle \boldsymbol{\beta}(t)\odot \partial^s \tilde{\mathcal{L}}(\boldsymbol{v}(t)), \boldsymbol{v}(t)\rangle
    \\
    &\le \eta_t\Vert \boldsymbol{\beta}(t)\odot \partial^s \tilde{\mathcal{L}}(\boldsymbol{v}(t))\Vert \Vert\boldsymbol{v}(t)\Vert
    \\
    &\le  2\eta_t\Vert \boldsymbol{\beta}^{\frac{1}{2}}(t)\odot \partial^s \tilde{\mathcal{L}}(\boldsymbol{v}(t))\Vert \Vert\boldsymbol{v}(t)\Vert
    \\
    &\le 2\eta_t\Vert \boldsymbol{\beta}^{\frac{1}{2}}(t)\odot \partial^s \tilde{\mathcal{L}}(\boldsymbol{v}(t))\Vert \Vert\boldsymbol{v}(t)\Vert \frac{\Vert \partial^s \tilde{\mathcal{L}}(\boldsymbol{v}(t))\Vert\Vert \boldsymbol{v}(t)\Vert}{\langle\partial^s \tilde{\mathcal{L}}(\boldsymbol{v}(t)), \boldsymbol{v}(t)\rangle}
    \\
    &\le 2\eta_t \frac{\Vert \boldsymbol{\beta}^{\frac{1}{2}}(t)\odot \partial^s \tilde{\mathcal{L}}(\boldsymbol{v}(t))\Vert^2 \Vert\boldsymbol{v}(t)\Vert^2}{\langle\partial^s \tilde{\mathcal{L}}(\boldsymbol{v}(t)), \boldsymbol{v}(t)\rangle}.
\end{align*}

As for $B$:
\begin{align*}
    B&=\eta_t^2\Vert\boldsymbol{\beta}(t)\odot\partial^s \tilde{\mathcal{L}} (\boldsymbol{v}(t))  \Vert^2
\\
    &\le 2\eta_t^2\Vert\boldsymbol{\beta}^{-
    \frac{1}{2}}(t)\odot\partial^s \tilde{\mathcal{L}} (\boldsymbol{v}(t))  \Vert^2
    \\
    &\le \frac{LNe^{-1}}{L\nu(t)}2\eta_t^2\Vert\boldsymbol{\beta}^{-
    \frac{1}{2}}(t)\odot\partial^s \tilde{\mathcal{L}} (\boldsymbol{v}(t))  \Vert^2
    \\
    &\le  \frac{LNe^{-1}}{L\nu(t)C_1^2}2\eta^2_t\Vert\boldsymbol{\beta}^{-
    \frac{1}{2}}(t)\odot\partial^s \tilde{\mathcal{L}} (\boldsymbol{v}(t))  \Vert^2\Vert \boldsymbol{v}(t)\Vert^2
    \\
    &\le  2\eta_t \frac{\Vert \boldsymbol{\beta}^{\frac{1}{2}}(t)\odot \partial^s \tilde{\mathcal{L}}(\boldsymbol{v}(t))\Vert^2 \Vert\boldsymbol{v}(t)\Vert^2}{\langle\partial^s \tilde{\mathcal{L}}(\boldsymbol{v}(t)), \boldsymbol{v}(t)\rangle}.
\end{align*}

Therefore, 
\begin{equation}
\label{eq:estimation_v_discre}
    \Vert \boldsymbol{v}(t+1)\Vert^2-\Vert \boldsymbol{v}(t)\Vert^2\le 4\eta_t \frac{\Vert \boldsymbol{\beta}^{\frac{1}{2}}(t)\odot \partial^s \tilde{\mathcal{L}}(\boldsymbol{v}(t))\Vert^2 \Vert\boldsymbol{v}(t)\Vert^2}{\langle\partial^s \tilde{\mathcal{L}}(\boldsymbol{v}(t)), \boldsymbol{v}(t)\rangle}.
\end{equation}

On the other hand, by Lemma \ref{lem:conver_learning_rate_adagrad_discrete}, 
\begin{align}
\nonumber
    \tilde{\mathcal{L}}(\boldsymbol{v}(t))-\tilde{\mathcal{L}}(\boldsymbol{v}(t+1))&\ge\frac{1}{2}\langle \partial^s \mathcal{L}(\boldsymbol{w}(t)),\eta_t\boldsymbol{h}^A(t)\odot\partial^s \mathcal{L}(\boldsymbol{w}(t))\rangle
    \\
\label{eq:estimation_L_discre}
    &=\eta_t\frac{1}{2}\Vert \boldsymbol{\beta}^{\frac{1}{2}}(t)\odot \partial^s \tilde{\mathcal{L}}(\boldsymbol{v}(t))\Vert^2.
\end{align}

Combining eqs. (\ref{eq:estimation_v_discre}) and (\ref{eq:estimation_L_discre}), we have 
\begin{equation*}
    4\frac{1}{\nu(t)}(\tilde{\mathcal{L}}(\boldsymbol{v}(t))-\tilde{\mathcal{L}}(\boldsymbol{v}(t+1)))\ge \frac{L}{2}\frac{ \Vert \boldsymbol{v}(t+1)\Vert^2-\Vert \boldsymbol{v}(t)\Vert^2}{\Vert\boldsymbol{v}(t)\Vert^2}.
\end{equation*}

Since $\nu(t)\ge \tilde{\mathcal{L}}\log\frac{1}{\tilde{\mathcal{L}}(t)}$, we further have
\begin{equation*}
    \frac{1}{\tilde{\mathcal{L}}(\boldsymbol{v}(t))\log\frac{1}{\tilde{\mathcal{L}}(\boldsymbol{v}(t))}}(\tilde{\mathcal{L}}(\boldsymbol{v}(t))-\tilde{\mathcal{L}}(\boldsymbol{v}(t+1)))\ge \frac{L}{8}\frac{ \Vert \boldsymbol{v}(t+1)\Vert^2-\Vert \boldsymbol{v}(t)\Vert^2}{\Vert\boldsymbol{v}(t)\Vert^2}.
\end{equation*}

By the convexity of $\log \log \frac{1}{x}$ (when $x$ is small) and $-\log x$, 
\begin{equation*}
    \log \log \frac{1}{\mathcal{L}(\boldsymbol{v}(t+1))}-\log \log \frac{1}{\tilde{\mathcal{L}}(\boldsymbol{v}(t))}\ge \frac{L}{4}(\log \Vert\boldsymbol{v}(t+1)\Vert-\log \Vert\boldsymbol{v}(t)\Vert).
\end{equation*}

The proof is completed.
\end{proof}
\begin{remark}
Actually, the convexity does not hold for  $x\in[e^{-1},\tilde{\mathcal{L}}(\boldsymbol{v}(t_1))]$ (if $\tilde{\mathcal{L}}(\boldsymbol{v}(t_1))>e^{-1}$). However, we can instead define 
\begin{equation*}
    \Phi_0(x)=\log\log \frac{1}{x}+\int_{0}^{x}\inf\{-\frac{1}{w\log \frac{1}{w}}:w\in[x,\log\frac{1}{\tilde{\mathcal{L}}(\boldsymbol{v} (t_1))}]\}+\frac{1}{x\log\frac{1}{x}} dw.
\end{equation*}
Which satisfies for $x\in[0,\tilde{\mathcal{L}}(\boldsymbol{v}(t_1))]$,  $ \frac{1}{x\log\frac{1}{x}}\le - \Phi'_0(x)$, $ \Phi_0(x)\le \log\log x$, and $\Phi_0(x)$ is convex. We can then replace $\log x$ by $e^{-\Phi_0(x)}$ and prove the above theorem in exactly the same way.0
\end{remark}
With the relationship between $\Vert\boldsymbol{v}\Vert$ and $\tilde{\mathcal{L}}$, we can now prove that the empirical loss goes to zero.
\begin{theorem}
$\lim_{t\rightarrow\infty} \tilde{\mathcal{L}}(t)=0$. Furthermore, $\lim_{t\rightarrow\infty}\rho(t)=\infty$.
\end{theorem}
\begin{proof}
    By Lemma \ref{lem:moving}, for any integer time $t\ge t_0$ 
    \begin{equation*}
    \tilde{\mathcal{L}}(t+1)-\tilde{\mathcal{L}}(t)\le-\frac{1}{2}\eta_t\frac{L^2\nu(t)^2}{\Vert\boldsymbol{v}(t)\Vert^2}\le -\frac{1}{2}\eta_t\frac{L^2\tilde{\mathcal{L}}(t)^2\log\frac{1}{\tilde{\mathcal{L}}(t)^2}}{\Vert\boldsymbol{v}(t)\Vert^2}.
    \end{equation*}
    
    Furthermore, since $\tilde{\gamma}'(t)\ge \tilde{\gamma}'(t_1)$, we have that
    \begin{equation*}
        \left(\frac{\log (\frac{1}{\tilde{\mathcal{L}}(t)})}{M_1^L\hat{\gamma}(t_0)}\right)^{\frac{1}{L}}\ge \rho.
    \end{equation*}
    
    Therefore, 
    \begin{align*}
    \tilde{\mathcal{L}}(t+1)-\tilde{\mathcal{L}}(t)&\le  -\frac{1}{2}\eta_t\frac{L^2\tilde{\mathcal{L}}(t)^2(\log\frac{1}{\tilde{\mathcal{L}}(t)})^2}{\Vert\boldsymbol{v}(t)\Vert^2}
    \\
    &\le -\frac{1}{2}\eta_tL^2\tilde{\mathcal{L}}(t)^2\left(\log\frac{1}{\tilde{\mathcal{L}}(t)}\right)^2 \left(\frac{\tilde{\gamma}'(t_1)}{\log (\frac{1}{\tilde{\mathcal{L}}(t)})}\right)^{\frac{8}{L}}
    \\
    &=-\frac{1}{2}\eta_tL^2\tilde{\mathcal{L}}(t)^2\left(\log\frac{1}{\tilde{\mathcal{L}}(t)}\right)^{2-\frac{8}{L}} \left(\tilde{\gamma}'(t_1)\right)^{\frac{8}{L}}.
    \end{align*}
    
   Let $E_0=\tilde{\mathcal{L}}\left(t_{1}\right)^{2}\left(\log \frac{1}{\tilde{\mathcal{L}}\left(t_{1}\right)}\right)^{2-8 / L}$, then $\psi(x)=\min\{x^2(\log\frac{1}{x})^{2-\frac{8}{L}}, E_0\}$. Apparently, $\psi(x)$ is non-decreasing in $(0,\tilde{\mathcal{L}}(t_1)]$. Therefore, $E(x)=\int_{x}^{\tilde{\mathcal{L}}(t_1)} \psi(s) ds$  is convex with respect to $x$, 
   and
   \begin{equation*}
       \begin{aligned}
        E(\tilde{\mathcal{L}}(t+1))-E(\tilde{\mathcal{L}}(t)) & \geq E^{\prime}(\tilde{\mathcal{L}}(t))(\tilde{\mathcal{L}}(t+1)-\tilde{\mathcal{L}}(t)) \\
        & \geq \frac{1}{2} \eta_tL^2\left(\tilde{\gamma}'(t_1)\right)^{\frac{8}{L}},
        \end{aligned}
   \end{equation*}
   which further implies
   \begin{equation*}
       E(\tilde{\mathcal{L}}(t))-E(\tilde{\mathcal{L}}(t_1))\geq \sum_{\tau=t_0}^{t-1} \frac{1}{2} \eta_\tau L^2\left(\tilde{\gamma}'(t_1)\right)^{\frac{2}{L}}.
   \end{equation*}
    
   Since $\lim_{t\rightarrow\infty}\sum_{\tau=t_0}^{t-1} \frac{1}{2} \eta_\tau L^2\left(\tilde{\gamma}'(t_1)\right)^{\frac{8}{L}}=\infty$, we then have 
   \begin{equation*}
       \lim_{t\rightarrow\infty} E(\tilde{\mathcal{L}}(t))=\infty,
   \end{equation*}
   and as a result,
   \begin{equation*}
       \lim_{t\rightarrow\infty} \tilde{\mathcal{L}} (t)=0.
   \end{equation*}
\end{proof}
\subsection{Convergence of surrogate margin}
\label{sec:discrete_surrogate margin}
As a preparation, define $B_1=\max\{\Vert\partial^s\tilde{q}_i(\boldsymbol{v})\Vert:i\in[N], \boldsymbol{v}\in\mathcal{B}(1)\}$, and $B_2=\max_\{\Vert\mathcal{H}\tilde{q}_i(\boldsymbol{v})\Vert:i\in[N], \boldsymbol{v}\in\mathcal{B}(1)\}$.

Furthermore, we define $\lambda(x)=(\log \frac{1}{x})^{-1}$, and $\mu(x)=\frac{\log \frac{1}{\tilde{\mathcal{L}}(\boldsymbol{v}(t_1))}}{\log \frac{1}{x}}$. Since $\lim_{t\rightarrow \infty} \tilde{\mathcal{L}}(t)=0$, we have the following lemma.
\begin{lemma}
There exists a large enough time $t_2\ge t_1$, such that, for any $t\ge t_2$,
\begin{equation*}
B_1^2\tilde{\mathcal{L}}(\boldsymbol{v}) \log^{7-\frac{8}{L}} \frac{1}{\tilde{\mathcal{L}}(\boldsymbol{v})}/\tilde{\gamma}'(t_1)^{\frac{8}{L}}\le \lambda(\tilde{\mathcal{L}}(\boldsymbol{v}(t)))\mu(\tilde{\mathcal{L}}(\boldsymbol{v}(t))),
\end{equation*}
and,
\begin{equation*}
 \frac{1}{\tilde{\gamma}'(t_1)^{8-8/L}}\tilde{\mathcal{L}}(\boldsymbol{v}(t))\log^{8-8/L} \frac{1}{\tilde{\mathcal{L}}(\boldsymbol{v}(t))}\left(B_{1}^{2}+C_1^{-L} B_{2}\right)
    \le\mu(t).
\end{equation*}
\end{lemma}
\begin{proof}
The proposition is obvious since $\log^i(\frac{1}{x})=\mathbf{o}(x)$, $\forall i$ as $x\rightarrow 0$.
\end{proof}

Then, we define a further surrogate margin $\hat{\gamma}$ of the discrete case as following:
\begin{equation*}
    \hat{\gamma}(t):=\frac{e^{\Phi(\tilde{\mathcal{L}})}}{\rho^{L}},
\end{equation*}
where $\rho(t)$ is defined as $\Vert\boldsymbol{\beta}^{-\frac{1}{2}}(t)\odot\boldsymbol{v}(t) \Vert$, and $\Phi(x)$ is defined as
\begin{equation*}
    \Phi(x)=\log \log \frac{1}{x}+\int_{0}^{x}\left(-\sup \left\{\frac{1+2(1+\lambda(\tilde w) / L) \mu(\tilde w)}{\tilde w \log \frac{1}{\tilde w}}: \tilde{w} \in\left[w, \tilde{\mathcal{L}}\left(t_{2}\right)\right]\right\}+\frac{1}{w \log \frac{1}{w}}\right) d w.
\end{equation*}

The following properties hold for $\hat{\gamma}$.

\begin{lemma}
\label{lem: relation_hat_tilde}
\begin{itemize}
    
    \item Let a series of $\boldsymbol{v}_i$ satisfies $\lim_{i\rightarrow\infty} \Vert\boldsymbol{v}_i\Vert=\infty$. Then, $\lim_{i\rightarrow\infty} \frac{\hat{\gamma}(\boldsymbol{v}_i)}{\tilde{\gamma}(\boldsymbol{v}_i)}=1$;
    \item If $\tilde{\mathcal{L}}(\boldsymbol{v}(t))\le \tilde{\mathcal{L}}(\boldsymbol{v}(t_2))$, then $\hat{\gamma}(t)<\tilde{\gamma}(t) \leq \bar{\gamma}(t)$.
\end{itemize}
\end{lemma}
\begin{proof}
 As beginning, we verify the existence of $\Phi$. Actually, when $x$ is small enough,
$\frac{1+2(1+\lambda(x) / L) \mu(x)}{x \log \frac{1}{x}}$ decreases, and $\lim_{x\rightarrow0}\frac{1+2(1+\lambda(x) / L) \mu(x)}{x \log \frac{1}{x}}=\infty$. Therefore, there exists a small enough $\varepsilon$, such that, for any $w<\varepsilon$, 
\begin{equation*}
    \sup \left\{\frac{1+2(1+\lambda(\tilde w) / L) \mu(\tilde w)}{\tilde w \log \frac{1}{\tilde w}}: \tilde{w} \in\left[w, \tilde{\mathcal{L}}\left(t_{2}\right)\right]\right\}=\frac{1+2(1+\lambda(w) / L) \mu(w)}{w \log \frac{1}{w}},
\end{equation*}
which further leads to
\begin{align*}
     &-\sup \left\{\frac{1+2(1+\lambda(\tilde w) / L) \mu(\tilde w)}{\tilde w \log \frac{1}{\tilde w}}: \tilde{w} \in\left[w, \tilde{\mathcal{L}}\left(t_{2}\right)\right]\right\}+\frac{1}{w\log\frac{1}{w}}
     \\
     &=-\frac{2(1+\lambda(w) / L) \mu(w)}{w \log \frac{1}{w}},
\end{align*}
which is integrable as $w\rightarrow0$. Concretely, for any $x<\varepsilon$,
\begin{align*}
    &\int_{0}^{x}\left(-\sup \left\{\frac{1+2(1+\lambda(\tilde w) / L) \mu(\tilde w)}{\tilde w \log \frac{1}{\tilde w}}: \tilde{w} \in\left[w, \tilde{\mathcal{L}}\left(t_{2}\right)\right]\right\}+\frac{1}{w \log \frac{1}{w}}\right) d w
    \\
    &=\int_{0}^{x}-\frac{2(1+\lambda(w) / L) \mu(w)}{w \log \frac{1}{w}} d w
    \\
    &=-\log \frac{1}{\mathcal{L}(t_1)} \left(\frac{1}{\log\frac{1}{x}}+\frac{1}{2L\log^2\frac{1}{x}}\right).
\end{align*}

Therefore, for a series $\{\boldsymbol{v}_i\}_{i=1}^{\infty}$ satisfying $\lim_{i\rightarrow\infty} \Vert\boldsymbol{v}_i\Vert=\infty$, 
\begin{equation*}
    \lim_{i\rightarrow\infty}\frac{\tilde{\gamma}(\boldsymbol{v}_i)}{\hat{\gamma}(\boldsymbol{v}_i)}= \lim_{i\rightarrow\infty} e^{-\mathcal{O}\left(\frac{1}{\log^2\frac{1}{\tilde{\mathcal{L}}(\boldsymbol{v}_i)}}\right)}=1.
\end{equation*}

Furthermore, if $x\le \tilde{\mathcal{L}}(\boldsymbol{v}(t_2))$, 
\begin{equation*}
    \Phi'(x)\le \frac{1}{w\log \frac{1}{w}}- \frac{1+2(1+\lambda(\tilde w) / L) \mu(\tilde w)}{\tilde w \log \frac{1}{\tilde w}}<0, 
\end{equation*}
which proves that $\hat{\gamma}(t)<\tilde{\gamma}(t)$.

The proof is completed.
\end{proof}

To bound the norm of first and second derivatives of $\tilde{\mathcal{L}}$, we further need the following lemma.

The next lemma characterizes the behavior of surrogate margin $\hat{\gamma}(t)$.

\begin{lemma}
\label{lem:moving}
For positive integer time $t\ge t_2$, $\hat{\gamma}(t)\ge e^{-\frac{1}{2}}\hat{\gamma}(t_2)$.
\end{lemma}
\begin{proof}
 For  any time $t\ge t_2$, 
\begin{equation*}
    \rho(t+1)^2- \rho(t)^2= (\Vert \boldsymbol{\beta}^{-\frac{1}{2}}(t)\odot\boldsymbol{v}(t+1)\Vert^2- \rho(t)^2)+(\rho(t+1)^2-\Vert \boldsymbol{\beta}^{-\frac{1}{2}}(t)\odot\boldsymbol{v}(t+1)\Vert^2).
\end{equation*}
We calculate two parts separately
\begin{align*}
    &\Vert \boldsymbol{\beta}^{-\frac{1}{2}}(t)\odot\boldsymbol{v}(t+1)\Vert^2-\Vert \boldsymbol{\beta}^{-\frac{1}{2}}(t)\odot\boldsymbol{v}(t)\Vert^2
    \\
    =& \eta_t^2\Vert\boldsymbol{\beta}^{\frac{1}{2}}(t)\odot\partial^s \tilde{\mathcal{L}} (\boldsymbol{v}(t))  \Vert^2+2L\eta_t \nu(t)
    \\
    \ge& 0.
\end{align*}

On the other hand, since $\boldsymbol{\beta}^{-\frac{1}{2}}$ is non-decreasing, 
\begin{equation*}
    \rho(t+1)=\Vert \boldsymbol{\beta}^{-\frac{1}{2}}(t+1)\odot\boldsymbol{v}(t+1) \Vert\ge \Vert \boldsymbol{\beta}^{-\frac{1}{2}}(t)\odot\boldsymbol{v}(t+1) \Vert.
\end{equation*}

$\Vert \boldsymbol{\beta}^{-\frac{1}{2}}(t)\odot\boldsymbol{v}(t+1)\Vert^2-\Vert \boldsymbol{\beta}^{-\frac{1}{2}}(t)\odot\boldsymbol{v}(t)\Vert^2$ can also be upper bounded as follows: 
\begin{align*}
    &\Vert \boldsymbol{\beta}^{-\frac{1}{2}}(t)\odot\boldsymbol{v}(t+1)\Vert^2-\Vert \boldsymbol{\beta}^{-\frac{1}{2}}(t)\odot\boldsymbol{v}(t)\Vert^2
    \\
    =& \eta_{t}^2\Vert\boldsymbol{\beta}^{\frac{1}{2}}(t)\odot\partial^s \tilde{\mathcal{L}} (\boldsymbol{v}(t))  \Vert^2+2L\eta_{t} \nu(t)
    \\
    = &2L\eta_{t} \nu(t)\left(\frac{\eta_{t}\Vert\boldsymbol{\beta}^{\frac{1}{2}}(t)\odot\partial^s \tilde{\mathcal{L}} (\boldsymbol{v}(t))  \Vert^2}{2L\nu(t)}+1\right)
    \\
    \overset{(*)}{\le} &2L\eta_{t} \nu(t)\left(\frac{\lambda(\tilde{\mathcal{L}}(\boldsymbol{v}(t)))\mu(\tilde{\mathcal{L}}(\boldsymbol{v}(t)))}{L}+1\right),
\end{align*}
where inequality (*) comes from the estimation of $\Vert\boldsymbol{\beta}^{\frac{1}{2}}(t)\odot\partial^s \tilde{\mathcal{L}} (\boldsymbol{v}(t))  \Vert^2$ as follows: by the homogeneity of $\tilde{q}_i$, we have
\begin{align*}
    \|\partial^s \tilde{\mathcal{L}}(\boldsymbol{v})\| &=\left\|-\sum_{i=1}^{N} e^{-\tilde{q}_{i}(\boldsymbol{v})} \partial^s \tilde{q}_{i}(\boldsymbol{v})\right\| \overset{(**)}{\leq} B_{1}\tilde{\mathcal{L}}(\boldsymbol{v}) \Vert \boldsymbol{v}\Vert^{L-1}\\
    &\overset{(***)}{\le} B_{1}\frac{1}{\tilde{\gamma}'(t_1)^{\frac{4}{L}}}\tilde{\mathcal{L}}(\boldsymbol{v}) \log^{4-\frac{4}{L}} \frac{1}{\tilde{\mathcal{L}}(\boldsymbol{v})},
\end{align*}
where inequality $(**)$ is due to $\partial^s\tilde{q}_i$ is $(L-1)$ homogeneous, and inequality $(***)$ holds by Lemma \ref{lem:tilde_gamma'}. 
On the other hand, $\nu(t)\ge \tilde{\mathcal{L}}(\boldsymbol{v}(t))\log \frac{1}{\tilde{\mathcal{L}}(\boldsymbol{v}(t))}$. Combining the estimation of $\nu$ and $ \Vert \boldsymbol{\beta}(t)^{-\frac{1}{2}} \odot \partial^s \tilde{\mathcal{L}}(\boldsymbol{\beta}(t)) \Vert$, we have
\begin{align*}
    \frac{\eta_{t}\Vert\boldsymbol{\beta}^{\frac{1}{2}}(t)\odot\partial^s \tilde{\mathcal{L}} (\boldsymbol{v}(t))  \Vert^2}{2L\nu(t)}
    &\le\frac{B_1^2\tilde{\mathcal{L}}(\boldsymbol{v}) \log^{7-\frac{8}{L}} \frac{1}{\tilde{\mathcal{L}}(\boldsymbol{v})}}{L\tilde{\gamma}'(t_1)^{\frac{8}{L}}}
    \\
    &\le \frac{\lambda(\tilde{\mathcal{L}}(\boldsymbol{v}(t)))\mu(\tilde{\mathcal{L}}(\boldsymbol{v}(t)))}{L}.
\end{align*}

Similar to Lemma \ref{lem:conver_learning_rate_adagrad_discrete}, the decrease of $\tilde{\mathcal{L}}(\boldsymbol{v})$ can be calculated by second order Taylor Expansion:

\begin{align}
\nonumber
    \tilde{\mathcal{L}}(t+1)-\tilde{\mathcal{L}}(t)
    = &-\eta_{t}\left\langle\boldsymbol{\beta}(t)\odot \partial^s \tilde{\mathcal{L}}(\boldsymbol{v}(t)), \partial^s \tilde{\mathcal{L}}(\boldsymbol{v}(t))\right\rangle
    \\
    \label{eq:estimation_loss_discrete}
    +&\frac{1}{2}\eta_{t}^2(\boldsymbol{\beta}(t)\odot\partial^s \tilde{\mathcal{L}})^T \mathcal{H}(\tilde{\mathcal{L}}(\boldsymbol{v}(\xi)))(\boldsymbol{\beta}(t)\odot\partial^s \tilde{\mathcal{L}}),
\end{align}
where $\xi\in(0,1)$. 

By homogeneity of $\tilde{q}_i$, the norm of Hessian matrix $\Vert\mathcal{H}(\tilde{\mathcal{L}}(\boldsymbol{v}(\xi))) \Vert$ can be bounded as
\begin{align*}
    \Vert\mathcal{H}(\tilde{\mathcal{L}}(\boldsymbol{v}(\xi))) \Vert&=\left\|\sum_{i=1}^{N} e^{-\tilde{q}_{i}}\left(\partial^s \tilde{q}_{i} \partial^s \tilde{q}_{i}^{\top}-\mathcal{H} \tilde{q}_{i}\right)\right\|_{2} \\
& \overset{(*)}{\leq} \sum_{i=1}^{N} e^{-\tilde{q}_{i}}\left(B_{1}^{2} \Vert \boldsymbol{v}(\xi)\Vert^{2 L-2}+B_2 \Vert \boldsymbol{v}(\xi) \Vert^{L-2}\right)
\\
&\leq \tilde{\mathcal{L}}(\boldsymbol{v}(\xi)) \Vert\boldsymbol{v}(\xi)\Vert^{2 L-2}\left(B_{1}^{2}+C_1^{-L} B_{2}\right)
\\
&\leq \frac{1}{\tilde{\gamma}'(t_1)^{8-8/L}}\tilde{\mathcal{L}}(\boldsymbol{v}(\xi))\log^{8-8/L} \frac{1}{\tilde{\mathcal{L}}(\boldsymbol{v}(\xi))}\left(B_{1}^{2}+C_1^{-L} B_{2}\right).
\end{align*}

Therefore,
\begin{align}
\nonumber
    &\frac{1}{2}\eta_{t}^2(\boldsymbol{\beta}(t)\odot\partial^s \tilde{\mathcal{L}})^T \mathcal{H}(\tilde{\mathcal{L}}(\boldsymbol{v}(\xi)))(\boldsymbol{\beta}(t)\odot\partial^s \tilde{\mathcal{L}})
    \\
    \nonumber
    \le& \eta_{t}\Vert\boldsymbol{\beta}(t)^{\frac{1}{2}}\odot\partial^s \tilde{\mathcal{L}})\Vert^2 \frac{1}{\tilde{\gamma}'(t_1)^{8-8/L}}\tilde{\mathcal{L}}(\boldsymbol{v}(\xi))\log^{8-8/L} \frac{1}{\tilde{\mathcal{L}}(\boldsymbol{v}(\xi))}\left(B_{1}^{2}+C_1^{-L} B_{2}\right)
    \\
    \label{eq:estimation_second_order}
    \le& \eta_{t}\Vert\boldsymbol{\beta}(t)^{\frac{1}{2}}\odot\partial^s \tilde{\mathcal{L}})\Vert^2\mu(t).
\end{align}

Taking the estimation eq. (\ref{eq:estimation_second_order}) back to eq. (\ref{eq:estimation_loss_discrete}), we have
\begin{equation}
\label{eq:final_estim_loss}
    \tilde{\mathcal{L}}(t)-\tilde{\mathcal{L}}(t+1)
    \ge (1-\mu(t))\eta_{t}\Vert\boldsymbol{\beta}(t)^{\frac{1}{2}}\odot\partial^s \tilde{\mathcal{L}})\Vert^2
\end{equation}

By multiplying $\frac{1+\lambda(\tilde{\mathcal{L}}(t)) \mu(\tilde{\mathcal{L}}(t)) / L}{(1-\mu(\tilde{\mathcal{L}}(t))) \nu(t)}$ to both sides of eq. (\ref{eq:final_estim_loss}), we then have
\begin{align*}
    &\frac{1+\lambda(\tilde{\mathcal{L}}(t)) \mu(\tilde{\mathcal{L}}(t)) / L}{(1-\mu(\tilde{\mathcal{L}}(t))) \nu(t)}(\tilde{\mathcal{L}}(t+1)-\tilde{\mathcal{L}}(t)) 
    \\\leq&- \eta_{t}\left(1+\frac{\lambda(\tilde{\mathcal{L}}(t)) \mu(\tilde{\mathcal{L}}(t))}{L}\right)\frac{L^{2} \nu(t)}{\rho(t)^{2}}
    \\
    \le &-\frac{L}{2}\frac{\Vert\boldsymbol{ \beta}^{-\frac{1}{2}}(t)\odot\boldsymbol{v}(t+1)\Vert^2-\Vert \boldsymbol{ \beta}^{-\frac{1}{2}}(t)\odot\boldsymbol{v}(t)\Vert^2}{\Vert \boldsymbol{ \beta}^{-\frac{1}{2}}(t)\odot\boldsymbol{v}(t)\Vert^2}.
\end{align*}

Furthermore, since $-\Phi^{\prime}(\tilde{\mathcal{L}}(t)) \geq \frac{1+\lambda(\tilde{\mathcal{L}}(t)) \mu(\tilde{\mathcal{L}}(t)) / L}{(1-\mu(\tilde{\mathcal{L}}(t))) \tilde{\mathcal{L}}(t) / \lambda(\tilde{\mathcal{L}}(t))}$, by the convexity of $\Phi$ and $-\log x$, we have that
\begin{align*}
   &L\left(\log \frac{1}{\Vert \boldsymbol{ \beta}^{-\frac{1}{2}}(t)\odot\boldsymbol{v}(t+1)\Vert}-\log \frac{1}{\Vert \boldsymbol{ \beta}^{-\frac{1}{2}}(t)\odot\boldsymbol{v}(t)\Vert}\right)\\ 
   + & (\Phi(\tilde{\mathcal{L}}(t+1))-\Phi(\tilde{\mathcal{L}}(t)))
   \ge 0.
\end{align*}

Therefore, 
\begin{equation*}
    \log \frac{\Phi(\tilde{\mathcal{L}}(t+1))}{\Vert \boldsymbol{ \beta}^{-\frac{1}{2}}(t)\odot\boldsymbol{v}(t+1)\Vert^L}-\log \frac{\Phi(\tilde{\mathcal{L}}(t)}{\Vert \boldsymbol{ \beta}^{-\frac{1}{2}}(t)\odot\boldsymbol{v}(t)\Vert^L}\ge 0.
\end{equation*}
Furthermore, since 
\begin{equation*}
    \frac{\Vert\boldsymbol{ \beta}^{-\frac{1}{2}}(t+1)\odot\boldsymbol{v}(t+1)\Vert^L}{\Vert\boldsymbol{ \beta}^{-\frac{1}{2}}(t+1)\odot\boldsymbol{v}(t)\Vert^L}\ge\Pi_{i=1}^p \frac{\beta_i^{-L/2}(t)}{\beta_i^{-L/2}(t+1)},
\end{equation*}
we have
\begin{align*}
    \hat{\gamma}(t+1)&=\frac{e^{-\Phi(t+1)}}{\Vert\boldsymbol{ \beta}^{-\frac{1}{2}}(t+1)\odot\boldsymbol{v}(t+1)\Vert^L}\ge \frac{e^{-\Phi(t+1)}}{\Vert\boldsymbol{ \beta}^{-\frac{1}{2}}(t)\odot\boldsymbol{v}(t+1)\Vert^L}\Pi_{i=1}^p \frac{\beta_i^{-L/2}(t)}{\beta_i^{-L/2}(t+1)}
    \\
    &\ge \hat{\gamma}(t)\Pi_{i=1}^p \frac{\beta_i^{-L/2}(t)}{\beta_i^{-L/2}(t+1)}.
\end{align*}

Thus, by induction,
\begin{equation*}
    \hat{\gamma}(t+1)\ge \hat{\gamma}(t)\Pi_{i=1}^p \frac{\beta_i^{-L/2}(t_0)}{\beta_i^{-L/2}(t+1)}\ge e^{-\frac{1}{2}}\hat{\gamma}(t_1).
\end{equation*}

The proof is completed.
\end{proof}

Similar to the flow case, we can then prove the convergence of $\hat{\gamma}$.
\begin{lemma}
There exists a positive real $\hat{\gamma}_{\infty}$, such that
\begin{equation*}
    \lim_{t\rightarrow\infty} \hat{\gamma}(t)=\hat{\gamma}_{\infty}.
\end{equation*}
\end{lemma}

\begin{proof}
 Since for any $t \ge t_2$
 \begin{equation*}
     \log \frac{\hat{\gamma}(t+1)}{\hat{\gamma}(t)}\ge \log\Pi_{i=1}^p \frac{\boldsymbol{\beta}^{-\frac{1}{2}}_i(t)}{\boldsymbol{\beta}^{-\frac{1}{2}}_i(t+1)},
 \end{equation*}
 we have that $\hat{\gamma}(t)\Pi_{i=1}^p \frac{1}{\boldsymbol{\beta}^{\frac{1}{2}}_i(t)}$ monotonously increases. Furthermore, since  $\hat{\gamma}(t)< \gamma(t)$ is bounded, so does $\hat{\gamma}(t)\Pi_{i=1}^p \frac{1}{\boldsymbol{\beta}^{\frac{1}{2}}_i(t)}$. Therefore, $\hat{\gamma}(t)\Pi_{i=1}^p \frac{1}{\boldsymbol{\beta}^{\frac{1}{2}}_i(t)}$ converges to a positive real. Since $\lim_{t\rightarrow\infty} \Pi_{i=1}^p \frac{1}{\boldsymbol{\beta}^{\frac{1}{2}}_i(t)}=1$, the proof is completed.
\end{proof}

\subsection{Verification of KKT point}
\label{sec:verification_discrete}
Similar to the flow case, we have the following construction of $(\varepsilon,\delta)$ KKT point. The proof is exactly the same as Lemma \ref{lem:construction_kkt}, and we omit it here.
\begin{lemma}
Let $\lambda_i=q_{\min }^{1-2 / L} \Vert \boldsymbol v\Vert \cdot e^{-f\left(q_{i}\right)} f^{\prime}\left(q_{i}\right) /\Vert\partial^s \tilde{\mathcal{L}} \Vert_{2}$. Then $\tilde{\boldsymbol{v}}(t)$ is a $(\varepsilon,\delta)$ KKT point of $(\tilde{P})$, where $\varepsilon$, $\delta$ are defined as follows:
\begin{align*}
     \varepsilon&=C_1(1+\cos(\boldsymbol{\theta}))
     \\
     \delta&=C_2 \frac{1}{\log\frac{1}{\mathcal{L}}},
\end{align*}
where $\cos(\boldsymbol{\theta})$ is defined as $\langle \hat{v}(t), \widehat{\partial^s\tilde{\mathcal{L}}}(t)\rangle$, and $C_1, C_2$ are positive real constants.

\end{lemma}

By Lemma \ref{lem:equivalence of b}, we only need to prove that $\lim_{t\rightarrow \infty} \tilde {\cos(\boldsymbol{\theta})}=1 $. Furthermore, the estimation of $ \cos(\boldsymbol{\tilde{\theta}})$ can be given by the following lemma.

\begin{lemma}
\label{lem:b_sum}
For any $t_3>t_4\ge t_2$,
\begin{equation*}
    \sum_{\tau=t_3}^{t_4-1} \left( \tilde{\cos(\boldsymbol{\theta})}(\tau)^{-2}-1\right) \left(\log \frac{1}{\rho(t)}-\log \frac{1}{\Vert \boldsymbol{\beta}^{-\frac{1}{2}}(\tau)\odot\boldsymbol{v}(\tau+1)\Vert}\right)\le \frac{1}{L} \log \frac{\hat{\gamma}(t_4)}{\hat{\gamma}(t_3)}+ \log \left(\Pi_{i=1}^p\frac{\boldsymbol{\beta}^{-\frac{1}{2}}_i(t_3)}{\boldsymbol{\beta}^{-\frac{1}{2}}_i(t_4)}\right).
\end{equation*}
\end{lemma}
\begin{proof}
    By Lemma \ref{lem:moving}, for any $t\ge t_0$
    \begin{align*}
        \frac{1}{L} \log \frac{\hat{\gamma}(t+1)}{\hat{\gamma}(t)}&=\frac{1}{L}\left(\Phi(t+1)-\Phi(t)\right)+\left(\log\frac{1}{\rho(t+1)}-\log\frac{1}{\rho(t)}\right)
        \\
        &\ge \left(\log\frac{1}{\rho(t)}-\log\frac{1}{\Vert\boldsymbol{\beta}^{-\frac{1}{2}}(t)\odot \boldsymbol{v}(t+1)\Vert}\right)\frac{\rho(t)^2}{L^2\nu(t)^2} \Vert \boldsymbol{\beta}^{\frac{1}{2}}(t)\odot \partial^s \tilde{\mathcal{L}}(t)\Vert^2
        \\
        &+\left(\log\frac{1}{\Vert\boldsymbol{\beta}^{-\frac{1}{2}}(t)\odot \boldsymbol{v}(t+1)\Vert}-\log\frac{1}{\rho(t)}\right)-\log \left(\Pi_{i=1}^p \frac{\boldsymbol{\beta}_i^{-\frac{1}{2}}(t+1)}{\boldsymbol{\beta}_i^{-\frac{1}{2}}(t)}\right).
    \end{align*}
    
    Furthermore, since $L\nu(t)=\left\langle \widehat{\boldsymbol{\beta}^{\frac{1}{2}}(t)\odot\partial^s\tilde{\mathcal{L}}(t)}, \widehat{\boldsymbol{\beta}^{-\frac{1}{2}}(t)\odot\boldsymbol{v}} \right\rangle$, we have 
    \begin{align*}
        \frac{1}{L} \log \frac{\hat{\gamma}(t+1)}{\hat{\gamma}(t)}
        &\ge \left(\log\frac{1}{\rho(t)}-\log\frac{1}{\Vert\boldsymbol{\beta}^{-\frac{1}{2}}(t)\odot \boldsymbol{v}(t+1)\Vert}\right)\cos(\boldsymbol{\theta})^{-2}
        \\
        &+\left(\log\frac{1}{\Vert\boldsymbol{\beta}^{-\frac{1}{2}}(t)\odot \boldsymbol{v}(t+1)\Vert}-\log\frac{1}{\rho(t)}\right)-\log \left(\Pi_{i=1}^p \frac{\boldsymbol{\beta}_i^{-\frac{1}{2}}(t+1)}{\boldsymbol{\beta}_i^{-\frac{1}{2}}(t)}\right).
    \end{align*}
    
    The proof is completed.
\end{proof}

We still need a lemma to bound the change of the direction of $\boldsymbol{\beta}^{-\frac{1}{2}}\odot\boldsymbol{v}$.
\begin{lemma}
For any $t\ge t_2$, 
 \begin{align*}
         &\Vert \widehat{\boldsymbol{\beta}^{-\frac{1}{2}}(t+1)\odot\boldsymbol{v}}(t+1)-\widehat{\boldsymbol{\beta}^{-\frac{1}{2}}(t)\odot\boldsymbol{v}(t)}\Vert
         \\
          \le & \mathcal{O}(1)\sum_{i=1}^p \left(\boldsymbol{\beta}_i^{-1}(t+1)-\boldsymbol{\beta}_i^{-1}(t)\right)+\left(\mathcal{O}(1)\frac{\Vert\boldsymbol{\beta}^{-\frac{1}{2}} \odot \boldsymbol{v}(t+1)\Vert}{\rho(t)}+1\right) \log \frac{\Vert\boldsymbol{\beta}^{-\frac{1}{2}}(t) \odot \boldsymbol{v}(t+1)\Vert}{\rho(t)}
          \\
          +&\left(\mathcal{O}(1)\frac{\Vert\boldsymbol{\beta}^{-\frac{1}{2}} \odot \boldsymbol{v}(t+1)\Vert}{\rho(t)}+1\right) \log \Pi_{i=1}^p \frac{\boldsymbol{\beta}^{-\frac{1}{2}}_i(t+1)}{\boldsymbol{\beta}^{-\frac{1}{2}}_i(t)}.
     \end{align*}
\end{lemma}

\begin{proof}
    Since triangular inequality,  
    \begin{align*}
        &\Vert \widehat{\boldsymbol{\beta}^{-\frac{1}{2}}(t+1)\odot\boldsymbol{v}}(t+1)-\widehat{\boldsymbol{\beta}^{-\frac{1}{2}}(t)\odot\boldsymbol{v}(t)}\Vert
        \\
        \le&\left\Vert \frac{1}{\rho(t+1)}\boldsymbol{\beta}^{-\frac{1}{2}}(t+1)\odot\boldsymbol{v}(t+1)-\frac{1}{\rho(t+1)}\boldsymbol{\beta}^{-\frac{1}{2}}(t+1)\odot\boldsymbol{v}(t)\right\Vert  
        \\
        +&\left\Vert \frac{1}{\rho(t+1)}\boldsymbol{\beta}^{-\frac{1}{2}}(t+1)\odot\boldsymbol{v}(t)-\frac{1}{\rho(t+1)}\boldsymbol{\beta}^{-\frac{1}{2}}(t)\odot\boldsymbol{v}(t)\right\Vert
        \\
        +&\left\Vert \frac{1}{\rho(t+1)}\boldsymbol{\beta}^{-\frac{1}{2}}(t)\odot\boldsymbol{v}(t)-\frac{1}{\rho(t)}\boldsymbol{\beta}^{-\frac{1}{2}}(t)\odot\boldsymbol{v}(t)\right\Vert.
    \end{align*}
    
    Let 
    \begin{gather*}
        A=\left\Vert \frac{1}{\rho(t+1)}\boldsymbol{\beta}^{-\frac{1}{2}}(t+1)\odot\boldsymbol{v}(t+1)-\frac{1}{\rho(t+1)}\boldsymbol{\beta}^{-\frac{1}{2}}(t+1)\odot\boldsymbol{v}(t)\right\Vert;\\
        B=\left\Vert \frac{1}{\rho(t+1)}\boldsymbol{\beta}^{-\frac{1}{2}}(t+1)\odot\boldsymbol{v}(t)-\frac{1}{\rho(t+1)}\boldsymbol{\beta}^{-\frac{1}{2}}(t)\odot\boldsymbol{v}(t)\right\Vert;\\
        C=\left\Vert \frac{1}{\rho(t+1)}\boldsymbol{\beta}^{-\frac{1}{2}}(t)\odot\boldsymbol{v}(t)-\frac{1}{\rho(t)}\boldsymbol{\beta}^{-\frac{1}{2}}(t)\odot\boldsymbol{v}(t)\right\Vert.
    \end{gather*}
    
    Then,
    \begin{align*}
        A=&\mathcal{O}(1)\frac{1}{\rho(t+1)}\Vert \boldsymbol{v}(t+1)-\boldsymbol{v}(t)\Vert
        \\
        =&\mathcal{O}(1)\frac{1}{\rho(t+1)}\Vert\eta_t \boldsymbol{\beta}(t)\odot \partial^s \tilde{\mathcal{L}}\Vert
        =\mathcal{O}(1)\frac{\eta_t }{\rho(t+1)}\Vert \partial^s \tilde{\mathcal{L}}\Vert
        \\
        \overset{(*)}{\le} & \mathcal{O}(1)\frac{\eta_t \nu(t)}{\hat{\gamma}(t_2)\rho(t+1)\rho(t)}
        \le  \mathcal{O}(1)\frac{\rho(t+1)^2-\rho(t)^2}{\rho(t+1)\rho(t)}
        \\
        \le & \mathcal{O}(1)\frac{\rho(t+1)^2-\rho(t)^2}{\rho(t+1)^2}\frac{\rho(t+1)}{\rho(t)}
        \le  \mathcal{O}(1)\frac{\rho(t+1)}{\rho(t)} \log \frac{\rho(t+1)}{\rho(t)},
    \end{align*}
    where eq. $(*)$ can be derived in the same way as Lemma \ref{lem: bound_b};
     \begin{align*}
         B&\le \sum_{i=1}^p (\boldsymbol{\beta}_i^{-\frac{1}{2}}(t+1)-\boldsymbol{\beta}_i^{-\frac{1}{2}}(t))^2\frac{\Vert \boldsymbol{v}(t)\Vert}{\rho(t+1)}
         \\
         &\le \mathcal{O}(1)\sum_{i=1}^p \boldsymbol{\beta}_i^{-1}(t+1)-\boldsymbol{\beta}_i^{-1}(t);
     \end{align*}
     and 
     \begin{align*}
         C=1-\frac{\rho(t)}{\rho(t+1)}.
     \end{align*}
     
     Therefore,
     \begin{align*}
         &\Vert \widehat{\boldsymbol{\beta}^{-\frac{1}{2}}(t+1)\odot\boldsymbol{v}}(t+1)-\widehat{\boldsymbol{\beta}^{-\frac{1}{2}}(t)\odot\boldsymbol{v}(t)}\Vert
         \\
         \le & \mathcal{O}(1)\sum_{i=1}^p \left(\boldsymbol{\beta}_i^{-1}(t+1)-\boldsymbol{\beta}_i^{-1}(t)\right)+\left(\mathcal{O}(1)\frac{\rho(t+1)}{\rho(t)}+1\right) \log \frac{\rho(t+1)}{\rho(t)}
         \\
          \le & \mathcal{O}(1)\sum_{i=1}^p \left(\boldsymbol{\beta}_i^{-1}(t+1)-\boldsymbol{\beta}_i^{-1}(t)\right)+\left(\mathcal{O}(1)\frac{\Vert\boldsymbol{\beta}^{-\frac{1}{2}} \odot \boldsymbol{v}(t+1)\Vert}{\rho(t)}+1\right) \log \frac{\Vert\boldsymbol{\beta}(t)^{-\frac{1}{2}}(t) \odot \boldsymbol{v}(t+1)\Vert}{\rho(t)}
          \\
          +&\left(\mathcal{O}(1)\frac{\Vert\boldsymbol{\beta}(t)^{-\frac{1}{2}} \odot \boldsymbol{v}(t+1)\Vert}{\rho(t)}+1\right) \log \Pi_{i=1}^p \frac{\boldsymbol{\beta}_i^{-\frac{1}{2}}(t+1)}{\boldsymbol{\beta}^{-\frac{1}{2}}_i(t)}
     \end{align*}
     
\end{proof}

We then prove that $\sum_{\tau=t_2}^{\infty} \log\frac{\Vert\boldsymbol{\beta}^{-\frac{1}{2}}(\tau) \odot \boldsymbol{v}(\tau+1)\Vert}{\rho(\tau)}=\infty$.
\begin{lemma}
\label{lem: sum_infty}
The sum of $\log\frac{\Vert\boldsymbol{\beta}^{-\frac{1}{2}}(\tau) \odot \boldsymbol{v}(\tau+1)\Vert}{\rho(\tau)}$ diverges, that is, $\sum_{\tau=t_2}^{\infty}\log\frac{\Vert\boldsymbol{\beta}^{-\frac{1}{2}}(\tau) \odot \boldsymbol{v}(\tau+1)\Vert}{\rho(\tau)}=\infty$.
\end{lemma}

\begin{proof}
    \begin{align*}
        \sum_{\tau=t}^{\infty} \log\frac{\Vert\boldsymbol{\beta}^{-\frac{1}{2}}(\tau) \odot \boldsymbol{v}(\tau+1)\Vert}{\rho(\tau)}\ge  \sum_{\tau=t}^{\infty} \log\frac{\rho(\tau+1)}{\rho(\tau)}-\log\Pi_{i=1}^p \frac{1}{\boldsymbol{\beta}^{-\frac{1}{2}}_i(t)}.
    \end{align*}
    The proof is completed since $\lim_{t\rightarrow\infty} \rho(t)=\infty$ and $\log\Pi_{i=1}^p \frac{1}{\boldsymbol{\beta}^{-\frac{1}{2}}_i(t)}$ is bounded.
\end{proof}

Now we can prove the following lemma.

\begin{lemma}
Let $\bar{\boldsymbol{v}}$ be any limit point of $\{\boldsymbol{v}(t)\}_{t=1}^{\infty}$. Then $\bar{\boldsymbol{v}}$ is a KKT point of optimization problem $(P)$.
\end{lemma}
\begin{proof}
 Let $t^1$ be any integer time larger than $t_2$. We construct a sequence $\{t^i\}_{i=1}^{\infty}$ by iteration. Suppose $t^1,\cdots,t^{k-1}$ have been constructed. Let $s^k>t^{k-1}$ be a large enough time which satisfies   
 \begin{gather*}
     \log\Pi_{i=1}^p \frac{1}{\boldsymbol{\beta}^{-\frac{1}{2}}_i(s^k)}\le \frac{1}{k^3},
     \\
     \Vert\hat{\boldsymbol{v}}(s^k)- \bar{\boldsymbol{v}}\Vert\le \frac{1}{k},
     \\
     \log\left(\frac{\hat{\gamma}_{\infty}}{\hat{\gamma}(s^k)}\right)\le \frac{1}{k^3}
     \\
     \Vert\boldsymbol{\beta}^{-1}(s^k)\Vert\le 1+\frac{1}{k-1}.
 \end{gather*}
 
 Then let $(s^k)'$ (guaranteed by Lemma \ref{lem: sum_infty}) be the first time greater than $s^k$ that $\sum_{\tau=s^k}^{(s^k)'-1} \log\frac{\Vert\boldsymbol{\beta}^{-\frac{1}{2}}(\tau) \odot
 \boldsymbol{v}(\tau+1)\Vert}{\rho(\tau)}\ge \frac{1}{k}$. By Lemma \ref{lem:b_sum}, there exists a time $t^k\in[s^k,(s^k)'-1]$, such that $\tilde{\cos(\boldsymbol{\theta})}^{-2}-1\le \frac{1}{k^2}$.
 
 Moreover, 
 \begin{align*}
     \Vert\hat{\boldsymbol{v}}(t^k)-\bar{\boldsymbol{v}}\Vert^2 &\le 2\left(\Vert\widehat{\boldsymbol{\beta}^{-\frac{1}{2}}(t^k)\odot\boldsymbol{v}}(t^k)-\bar{\boldsymbol{v}}\Vert^2+\Vert\widehat{\boldsymbol{\beta}^{-\frac{1}{2}}(t^k)\odot\boldsymbol{v}}(t^k)-\hat{\boldsymbol{v}}(t^k)\Vert^2\right)
     \\
     &=2\Vert\widehat{\boldsymbol{\beta}^{-\frac{1}{2}}(t^k)\odot\boldsymbol{v}(t^k)}-\bar{\boldsymbol{v}}\Vert^2+2\left(2-2\left\langle\widehat{\boldsymbol{\beta}^{-\frac{1}{2}}(t^k)\odot\boldsymbol{v}(t^k)},\hat{\boldsymbol{v}}(t^k)\right\rangle\right)
     \\
     &\le 2\Vert\widehat{\boldsymbol{\beta}(t^k)\odot\boldsymbol{v}(t^k)}-\bar{\boldsymbol{v}}\Vert^2+O\left(\frac{1}{k}\right)
     \\
     &\le  2\Vert\widehat{\boldsymbol{\beta}(t^k)\odot\boldsymbol{v}(t^k)}-\widehat{\boldsymbol{\beta}(s^k)\odot\boldsymbol{v}(s^k)}\Vert^2+O\left(\frac{1}{k}\right)
     \\
     &\le \mathcal{O}(1)\frac{1}{k}+O(e^{\frac{1}{k}})\frac{1}{k}\rightarrow 0.
 \end{align*}
 
 The proof is completed.

\end{proof}

Therefore, similar to the gradient flow case, we then have the following theorem.
 \begin{theorem}
 Let $\bar{\boldsymbol{w}}$ be any limit point of $\{\hat{\boldsymbol{w}}(t)\}$. Then $\bar{\boldsymbol{w}}$ is along the direction of a KKT point of the following optimization problem.
\begin{gather*}
    \text{Minimize } \frac{1}{2}\Vert \boldsymbol{h}_{\infty}^{-\frac{1}{2}}\odot \boldsymbol{w} \Vert\\
    \text{Subject to: } q_i(w)\ge 1.
\end{gather*}
  
 \end{theorem}

\section{Proof of Multi-class Classification with Logistic Loss}
\label{appen: multi-class-classification}
In this section, we prove the result for multi-class classification with logistic loss mentioned in Remark \ref{remark: multi_class}. Concretely, the dataset for this case can be represented as $\{(\boldsymbol{x}_i,y_i)\}_{i=1}^N$, where $y_i \in [C]$ represents the class $\boldsymbol{x}_i$ belongs to. Unlike the binary classification case, neural network $\boldsymbol \Phi$ outputs a $C$-dimension vector as scores for $C$ classes, and we use $\boldsymbol{\Phi}(\boldsymbol{w},\boldsymbol{x}_i)_j$ as the $j$-th component of $\boldsymbol{\Phi}(\boldsymbol{w},\boldsymbol{x}_i)$. The empirical loss can then be represented as 
\begin{equation}
\label{eq: def_empirical_loss_multi_class}
    \mathcal{L}(\boldsymbol{w})=\sum_{i=1}^N -\log \frac{e^{\boldsymbol{\Phi}(\boldsymbol{w},\boldsymbol{x}_i)_{y_i}}}{\sum_{j=1}^C e^{\boldsymbol{\Phi}(\boldsymbol{w},\boldsymbol{x}_i)_{j}}}.
\end{equation}
For AdaGrad, RMSProp, and Adam (w/m), limit $\boldsymbol{h}^A_{\infty}$,  $\boldsymbol{h}^R_{\infty}$, $\boldsymbol{h}^M_{\infty}$  remains non-zero, and we can then define $\boldsymbol{v}$, $\tilde{\mathcal{L}}$, and $\boldsymbol{\beta}$ the same as Theorems \ref{thm:AdaGrad_flow} and \ref{thm:RMS_flow} (we use $\boldsymbol{v}$ and $\tilde{\mathcal{L}}$ to represent all cases), and $\tilde{\boldsymbol{\Phi}}(\boldsymbol{v},\boldsymbol{x}_i)=\boldsymbol{\Phi}(\boldsymbol{h}_{\infty}^{\frac{1}{2}}\odot\boldsymbol{v},\boldsymbol{x}_i)$.

We can then define margins in the multi-class classification similarly as the binary case: surrogate norm and margin are defined exactly the same as the binary case; define $\tilde{q}_i(\boldsymbol{v})= \tilde{\boldsymbol{\Phi}}(\boldsymbol{v},\boldsymbol{x}_i)_{y_i}-\max_{j\ne y_i}\tilde{\boldsymbol{\Phi}}(\boldsymbol{v},\boldsymbol{x}_i)_{j}$ and normalized margin can be still defined  as $\frac{\tilde{q}_{\min}(\boldsymbol{v})}{\Vert \boldsymbol{v} \Vert^L}$. The corresponding convergent direction for adaptive gradient flow under multi-class setting can then be characterized by the following theorem:
\begin{theorem}
\label{thm:approximate_flow_multi}
Let $\boldsymbol{v}$ satisfy an adaptive gradient flow $\mathcal{F}$ which satisfies Assumption \ref{assum: continuous}. Let $\bar{\boldsymbol{v}}$ be any limit point of $\{\hat{\boldsymbol{v}}(t)\}_{t=0}^{\infty}$ (where $\hat{\boldsymbol{v}}(t)=\frac{\boldsymbol{v}(t)}{\Vert \boldsymbol{v}(t)\Vert}$ is normalized parameter). Then $\bar{\boldsymbol{v}}$ is along the direction of a KKT point of the following $L^2$ max-margin problem $(P)$:
\begin{gather*}
\min \frac{1}{2}\Vert \boldsymbol{v} \Vert^2\\
\text{subject to } \tilde{q}_i(\boldsymbol{v}) \ge 1,
\forall i \in[N].
\end{gather*}
\end{theorem}

Proof of Theorem \ref{thm:approximate_flow_multi} differs from that of Theorem \ref{thm:approximate_flow} only by Lemma \ref{lem:appen_smoothed_normal}, Lemma \ref{lem:derivative_rho}, and the construction of $\lambda_i$ in Lemma \ref{lem: construction of KKT_appen}. We show modifications respectively.

First of all, we show normalized margin $\gamma$ and surrogate margin $\tilde{\gamma}$ converge to the same limit: 
\begin{lemma}
\label{lem:appen_smoothed_normal_multi_class}
 Let a function $\boldsymbol{v}(t)$ obey an adaptive gradient flow $\mathcal{F}$ which satisfies Assumption \ref{assum: continuous}, with loss $\tilde{\mathcal{L}}$ and component learning rate $\boldsymbol{\beta}(t)$. Then we have $\lim_{t\rightarrow\infty}\frac{\rho(t)}{\Vert\boldsymbol{v}(t)\Vert}=1$. Furthermore, if further $\lim_{t\rightarrow \infty }\tilde{\mathcal{L}}(t)=\infty$, we have  $\lim_{t\rightarrow\infty}\frac{\gamma(t)}{\tilde{\gamma}(t)}=1$.
\end{lemma}
\begin{proof}
 By definition of empirical loss (eq. (\ref{eq: def_empirical_loss_multi_class})), 
 \begin{align*}
     \tilde{\mathcal{L}}(\boldsymbol{v})=&\sum_{i=1}^N -\log \frac{e^{\tilde{\boldsymbol{\Phi}}(\boldsymbol{v},\boldsymbol{x}_i)_{y_i}}}{\sum_{j=1}^C e^{\tilde{\boldsymbol{\Phi}}(\boldsymbol{v},\boldsymbol{x}_i)_{j}}}
     \\
     =&\sum_{i=1}^N -\log \frac{1}{1+\sum_{j\ne y_i} e^{\tilde{\boldsymbol{\Phi}}(\boldsymbol{v},\boldsymbol{x}_i)_{j}-\tilde{\boldsymbol{\Phi}}(\boldsymbol{v},\boldsymbol{x}_i)_{y_i}}}.
 \end{align*}
 
 Let $\tilde{q}_i'(\boldsymbol{v})=\log \sum_{j\ne y_i} e^{\tilde{\boldsymbol{\Phi}}(\boldsymbol{v},\boldsymbol{x}_i)_{j}-\tilde{\boldsymbol{\Phi}}(\boldsymbol{v},\boldsymbol{x}_i)_{y_i}}$. By exact the same routine of Lemma \ref{lem:appen_smoothed_normal}, we have
 \begin{equation}
 \label{eq:equiv_gamma_multi}
     \lim_{t\rightarrow\infty} \frac{\gamma(t)}{\frac{\tilde{q}'_{\min}(\boldsymbol{v})}{\rho^L}}=1.
 \end{equation}

 On the other hand,
 \begin{equation*}
     -\tilde{q}_{\min}(\boldsymbol{v})\le \log \sum_{j\ne y_i} e^{\tilde{\boldsymbol{\Phi}}(\boldsymbol{v},\boldsymbol{x}_i)_{j}-\tilde{\boldsymbol{\Phi}}(\boldsymbol{v},\boldsymbol{x}_i)_{y_i}}\le -\tilde{q}_{\min}(\boldsymbol{v})+ \log N,
 \end{equation*}
 which leads to 
  \begin{equation}
   \label{eq:equiv_tilgamma_multi}
     \lim_{t\rightarrow\infty} \frac{\tilde{\gamma}(t)}{\frac{\tilde{q}'_{\min}(\boldsymbol{v})}{\rho^L}}=1.
 \end{equation}
 
 Combining eqs. (\ref{eq:equiv_gamma_multi}) and  (\ref{eq:equiv_tilgamma_multi}), the proof is completed.
 
\end{proof}

Secondly, we calculate derivative of surrogate norm $\rho$ under multi-class classification setting.

\begin{lemma}
\label{lem:derivative_rho_multi_class}
The derivative of $\rho^2$ is as follows:
\begin{equation*}
    \frac{1}{2}\frac{\mathrm{d} \rho(t)^2}{\mathrm{d}t}= L \nu(t)+ \left\langle \boldsymbol{v}(t),\boldsymbol{\beta}^{-\frac{1}{2}}(t)\odot \frac{\mathrm{d}\boldsymbol{\beta}^{-\frac{1}{2}}}{\mathrm{d}t}(t)\odot \boldsymbol{v}(t) \right\rangle,
\end{equation*}
where $\nu(t)$ is defined as 
\begin{equation*}
    \nu(t)=\sum_{i=1}^{N} \frac{\sum_{j \neq y_{n}} e^{ \tilde{\boldsymbol{\Phi}}(\boldsymbol{v},\boldsymbol{x}_i)_{j}-\tilde{\boldsymbol{\Phi}}(\boldsymbol{v},\boldsymbol{x}_i)_{y_i}}(\tilde{\boldsymbol{\Phi}}(\boldsymbol{v},\boldsymbol{x}_i)_{y_i}-\tilde{\boldsymbol{\Phi}}(\boldsymbol{v},\boldsymbol{x}_i)_{j})}{1+\sum_{j \neq y_{n}} e^{\tilde{\boldsymbol{\Phi}}(\boldsymbol{v},\boldsymbol{x}_i)_{j}-\tilde{\boldsymbol{\Phi}}(\boldsymbol{v},\boldsymbol{x}_i)_{y_i}}}.
\end{equation*} 
Furthermore, we have that $\nu(t)>\frac{g\left(\log \frac{1}{\tilde{\mathcal{L}}(\boldsymbol{v}(t))}\right)}{g^{\prime}\left(\log \frac{1}{\tilde{\mathcal{L}}(\boldsymbol{v}(t))}\right)} \tilde{\mathcal{L}}(\boldsymbol{v}(t))$.
\end{lemma}

\begin{proof}
 We only need to show 
 \begin{equation}
 \label{eq:proof_derivative_rho_multi_1}
     -\langle\partial^s \tilde{\mathcal{L}}(\boldsymbol{v}(t)),\boldsymbol{v}(t) \rangle=L\nu(t),
 \end{equation}
 and 
 \begin{equation}
 \label{eq:proof_derivative_rho_multi_2}
     \nu(t)>\frac{g\left(\log \frac{1}{\tilde{\mathcal{L}}(\boldsymbol{v}(t))}\right)}{g^{\prime}\left(\log \frac{1}{\tilde{\mathcal{L}}(\boldsymbol{v}(t))}\right)} \tilde{\mathcal{L}}(\boldsymbol{v}(t)),
 \end{equation}
 while other parts of the proof follows exact the same as Lemma \ref{lem:derivative_rho}.
 
 By chain rule,
 \begin{equation*}
     \partial^s \tilde{\mathcal{L}}(\boldsymbol{v})=\sum_{i=1}^N \frac{\sum_{j\ne y_i} e^{\tilde{\boldsymbol{\Phi}}(\boldsymbol{v},\boldsymbol{x}_i)_{j}-\tilde{\boldsymbol{\Phi}}(\boldsymbol{v},\boldsymbol{x}_i)_{y_i}}\partial^s(\tilde{\boldsymbol{\Phi}}(\boldsymbol{v},\boldsymbol{x}_i)_{j}-\tilde{\boldsymbol{\Phi}}(\boldsymbol{v},\boldsymbol{x}_i)_{y_i})}{1+\sum_{j\ne y_i} e^{\tilde{\boldsymbol{\Phi}}(\boldsymbol{v},\boldsymbol{x}_i)_{j}-\tilde{\boldsymbol{\Phi}}(\boldsymbol{v},\boldsymbol{x}_i)_{y_i}}},
 \end{equation*}
 while by homogeneity of $\tilde{\boldsymbol{\Phi}}$,
 \begin{equation*}
     \langle\partial^s(\tilde{\boldsymbol{\Phi}}(\boldsymbol{v},\boldsymbol{x}_i)_{j}-\tilde{\boldsymbol{\Phi}}(\boldsymbol{v},\boldsymbol{x}_i)_{y_i}),\boldsymbol{v}\rangle=L(\tilde{\boldsymbol{\Phi}}(\boldsymbol{v},\boldsymbol{x}_i)_{j}-\tilde{\boldsymbol{\Phi}}(\boldsymbol{v},\boldsymbol{x}_i)_{y_i}),
 \end{equation*}
 which completes the proof of eq. (\ref{eq:proof_derivative_rho_multi_1}).

 As for eq. (\ref{eq:proof_derivative_rho_multi_2}),
 \begin{align*}
      \nu(t)=&\sum_{i=1}^{N} e^{-f(\tilde{q}_i'(\boldsymbol{v}(t)))} f'(\tilde{q}_i'(\boldsymbol{v}(t))) \frac{\sum_{j\ne y_i} e^{\tilde{\boldsymbol{\Phi}}(\boldsymbol{v},\boldsymbol{x}_i)_{j}-\tilde{\boldsymbol{\Phi}}(\boldsymbol{v},\boldsymbol{x}_i)_{y_i}}(\tilde{\boldsymbol{\Phi}}(\boldsymbol{v},\boldsymbol{x}_i)_{y_i}-\tilde{\boldsymbol{\Phi}}(\boldsymbol{v},\boldsymbol{x}_i)_{j})}{\sum_{j\ne y_i} e^{\tilde{\boldsymbol{\Phi}}(\boldsymbol{v},\boldsymbol{x}_i)_{j}-\tilde{\boldsymbol{\Phi}}(\boldsymbol{v},\boldsymbol{x}_i)_{y_i}}}
      \\
      \ge &\sum_{i=1}^{N} e^{-f(\tilde{q}_i'(\boldsymbol{v}(t)))} f'(\tilde{q}_i'(\boldsymbol{v}(t)))\tilde{q}_i'(\boldsymbol{v}(t))
      \\
      \ge & \frac{g\left(\log \frac{1}{\tilde{\mathcal{L}}(\boldsymbol{v}(t))}\right)}{g^{\prime}\left(\log \frac{1}{\tilde{\mathcal{L}}(\boldsymbol{v}(t))}\right)} \tilde{\mathcal{L}}(\boldsymbol{v}(t)).
 \end{align*}
    
The proof is completed.

\end{proof}

Finally, we provide construction of $\lambda_i$ similar to Lemma \ref{lem: construction of KKT_appen}. The proof follows the same routine as Lemma \ref{lem: construction of KKT_appen} and we omit it here.

\begin{lemma}
\label{lem: construction of KKT_appen_multi}
Let $\boldsymbol{v}$ obey adaptive gradient flow $\mathcal{F}$ with empirical loss $\tilde{\mathcal{L}}$ satisfying Assumption \ref{assum: continuous}. Let time $t_1$ be constructed as Lemma \ref{lem:lower_bound_margin}. Then, define coefficients in Definition \ref{def:general_KKT} as 
\begin{equation*}
    \lambda_{i,j}(t)=\tilde{q}_{\min }(\boldsymbol{v}(t))^{1-2 / L} \Vert \boldsymbol v(t)\Vert \cdot  \frac{1}{\Vert\partial^s \tilde{\mathcal{L}}(\boldsymbol{v}(t)) \Vert_{2}}\frac{1+ e^{\tilde{\boldsymbol{\Phi}}(\boldsymbol{v},\boldsymbol{x}_i)_{j}-\tilde{\boldsymbol{\Phi}}(\boldsymbol{v},\boldsymbol{x}_i)_{y_i}}}{\sum_{j\ne y_i} e^{\tilde{\boldsymbol{\Phi}}(\boldsymbol{v},\boldsymbol{x}_i)_{j}-\tilde{\boldsymbol{\Phi}}(\boldsymbol{v},\boldsymbol{x}_i)_{y_i}}},
\end{equation*}
(where
$i\in[N]$, $j\in[C]/\{i\}$). Then, for any time $t\ge t_1$, $\tilde{\boldsymbol{v}}(t)=\tilde{q}_{\min}(\boldsymbol{v}(t))^{-\frac{1}{L}}\boldsymbol{v}(t)$ is an $(\varepsilon(t),\delta(t))$ KKT point of $(\tilde{P})$, where $\varepsilon(t)$, $\delta(t)$ are defined as follows:
\begin{align*}
     \varepsilon(t)&=\mathcal{O}(1-\cos(\boldsymbol{\theta}(t)))
     \\
     \delta(t)&=\mathcal{O}\left( \frac{1}{\log\frac{1}{\tilde{\mathcal{L}}(\boldsymbol{v}(t))}}\right),
\end{align*}
where $\cos(\boldsymbol{\theta}(t))$  is defined as inner product of  $ \hat{\boldsymbol{v}}(t)$ and $-\widehat{\partial^s\tilde{\mathcal{L}}(\boldsymbol{v}(t))}$.

\end{lemma}

\section{Experiment Details}
In this section, we provide detailed explanation of experiments showed in Section \ref{sec:experiment} \footnote{https://github.com/bhwangfy/ICML-2021-Adaptive-Bias}. This section is divided into two parts according to Section \ref{sec:experiment}: in Section \ref{sec:expe_mnist}, we provide details of structure of neural network we use and hyper-parameters. We also further plot two additional experiments of Adam and SGD to show the influence of momentum; in Section \ref{sec:expe_toy}, we show construction of dataset in Section \ref{sec:exper_toy_main} and choose of hyper-parameters. We also show how direction of $\boldsymbol{h}_{\infty}^{-\frac{1}{2}}$ influence convergent direction of parameters.
\subsection{Experiment on MNIST}
\label{sec:expe_mnist}
\subsubsection{Construction of Neural Network and Choice of Hyper-parameters}
We use the $4$-layer convolutional neural network adopted by \cite{madry2017towards} as our model to conduct multi-class classification on MNIST \cite{lecun1998mnist}. Concretely, this convolutional neural network can be expressed in order as  convolutional layer with $32$ channel and filter size $5\times5$, max-pool layer with kernel size $2\time 2$ and stride $2$,
convolutional layer with $64$ channel and filter size $3\times 3$, max-pool with kernel size $2\time 2$, fully connected layer with width $1024$, and fully connected layer with width $10$. In order to guarantee this neural network is homogeneous, we further set bias in all layers to be zero. We use default method in Pytorch to initialize the neural network.

As for hyper-parameters, we set learning rate of AdaGrad to be the default value in Pytorch; while for RMSProp, we set learning rate and decay parameter $b$ as $0.001$ and $0.9$, which is  suggested by \cite{hinton2012neural} and used as a default value in Tensorflow;  for Adam, we set the learning rate to $0.0001$ as default value in Pytorch, and $b$ to be the same as RMSProp.

\subsubsection{Influence of Momentum}
We plot convergent behaviors for SGDm and Adam in this section. Figure \ref{fig:adding momentum} shows that adding momentum term will NOT keep normalized margin from lower bounded, which  indicates our theory might be extended to gradient based optimization methods with momentum. Specifically, for SGD, we use learning rate $0.1$ and momentum parameter $0.9$; for Adam, we use the same setting as Adam (w/m) with momentum parameter $0.9$.
\begin{figure*}[htb]
\centering
\begin{subfigure}
[b]{0.23\columnwidth}        \includegraphics[trim=0cm 0.5cm 0cm 0cm, width=\textwidth]{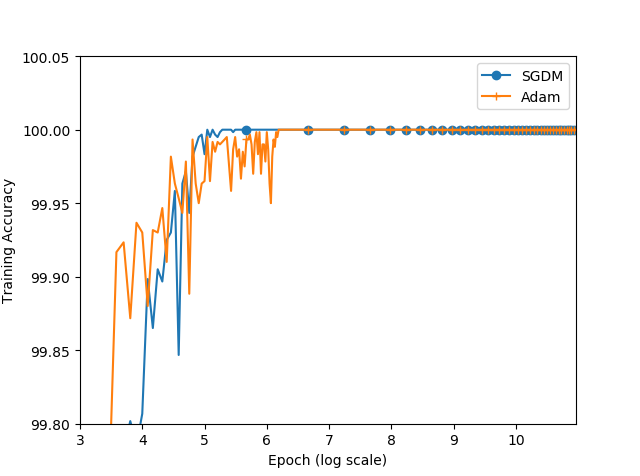}
\caption{ Training Accuracy }
\label{fig: training accuracy_m}
\end{subfigure}
\begin{subfigure}
   [b]{0.23\columnwidth}        \includegraphics[trim=0cm 0.5cm 0cm 0cm, width=\textwidth]{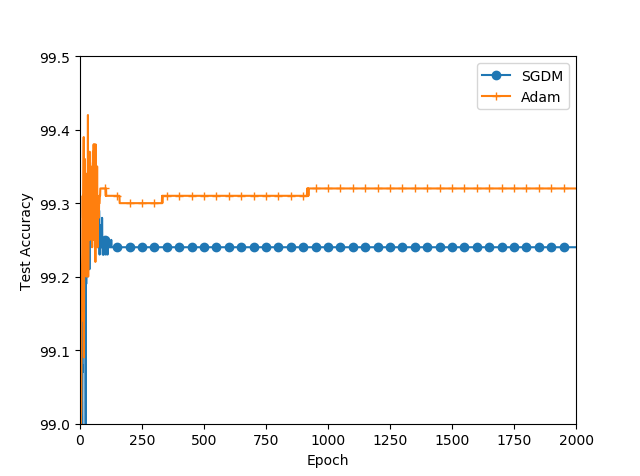}
\caption{ Test Accuracy}
\label{fig:test accuracy_m}
\end{subfigure}
\begin{subfigure}
[b]{0.23\columnwidth}        \includegraphics[trim=0cm 0.5cm 0cm 0cm, width=\textwidth]{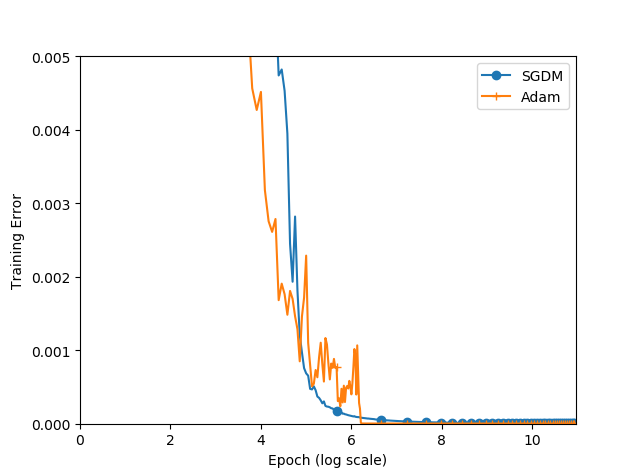}
\caption{ Training Loss}
\label{fig: loss_m}
\end{subfigure}
\begin{subfigure}
   [b]{0.23\columnwidth}        \includegraphics[trim=0cm 0.5cm 0cm 0cm, width=\textwidth]{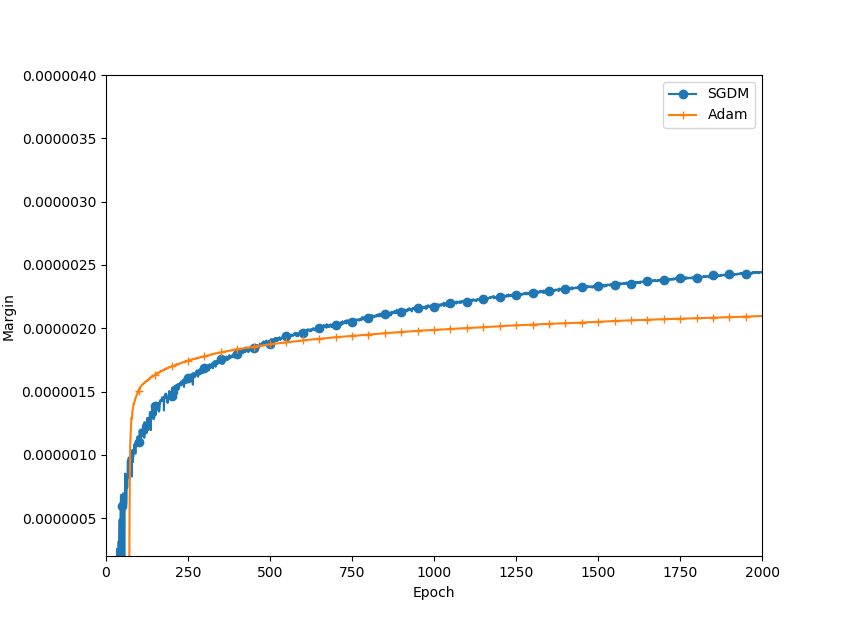}
\caption{ Normalized Margin}
\label{fig: margin_m}
\end{subfigure}
\vskip -0.3cm
\caption{Observation of convergent behavior after adding momentum for SGD and Adam. 
One can observe that loss still converge to zero, and margin will keep lower bounded.}
\label{fig:adding momentum}
\vskip -0.5cm
\end{figure*}
\begin{figure*}[htb]
\centering
\begin{subfigure}
[b]{0.25\columnwidth}        \includegraphics[trim=0cm 0.5cm 0cm 0cm, width=\textwidth]{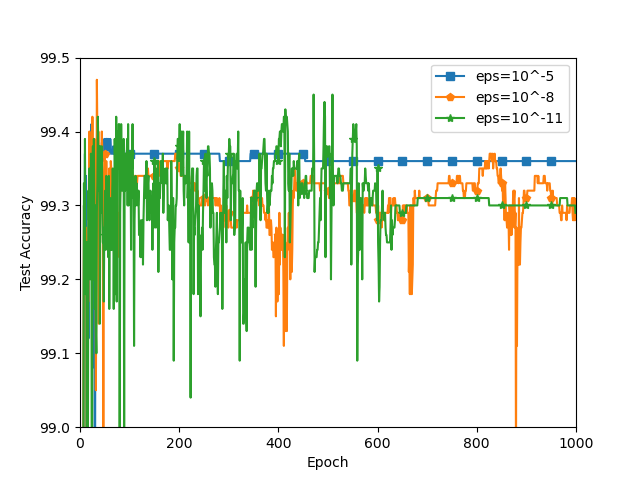}
\caption{ Test Accuracy }
\label{fig: eps_test}
\end{subfigure}
\begin{subfigure}
   [b]{0.25\columnwidth}        \includegraphics[trim=0cm 0.5cm 0cm 0cm, width=\textwidth]{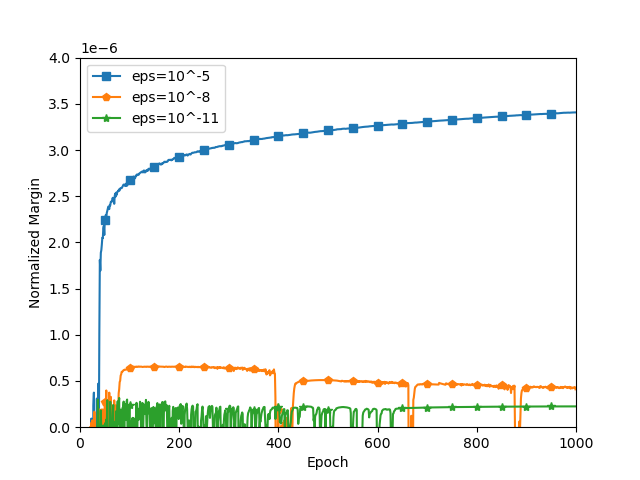}
\caption{Normalized margin}
\label{fig: eps_margin}
\end{subfigure}
\vskip -0.3cm
\caption{Observation of generalization behavior of RMSProp with different $\varepsilon$. 
One can observe that larger $\varepsilon$ leads to larger margin and smaller generalization error.}
\label{fig:different_eps}
\vskip -0.3cm
\end{figure*}

\subsection{Influence of $\varepsilon$}
We compare the generalization behaviors of RMSProp with different $\varepsilon$ selected in Figure \ref{fig:different_eps}. It is observed that as $\varepsilon$ decreases, normalized margin gets smaller and the generalization error gets larger, which indicates the importance of $\varepsilon$ on the generalization behavior. When $\varepsilon$ is completed removed (i.e., is set to $0$), the training does not converge. Therefore, we do not include the results for $\varepsilon=0$ here.

\subsection{Experiment on Two Layer MLP}
\label{sec:expe_toy}
\subsubsection{Dataset Construction and Choice of Hyper-parameters}
As mentioned in Section \ref{sec:exper_toy_main}, we use a two layer MLP $\Phi$ with leaky ReLU activation $\sigma$ defined as $\Phi(\boldsymbol{x},\boldsymbol{w},v)=v\sigma(\langle\boldsymbol{w},\boldsymbol{x}\rangle)$, where $x\in\mathbb{R}^2, \boldsymbol{w}\in\mathbb{R}^2$ and $v\in\mathbb{R}$ and $\sigma(t)$ is the Leaky ReLU activation function, i.e., $\sigma(t)=t$ for $t\geq 0$ and $\sigma(t)=\frac{t}{2}$ for $ t<0$. We construct binary classification dataset $S$ as $\{(\boldsymbol{x}_i,y_i)\}_{i=1}^{100}$ as follows:
\begin{gather*}
    (\boldsymbol{x}_i,y_i)=((\cos(0.5),\sin{0.5})+\boldsymbol{\varepsilon}_i,1),\text{ } i\in \{1,2,\cdots,50\};
    \\
     (\boldsymbol{x}_i,y_i)=((-\cos(0.5),-\sin{0.5})+\boldsymbol{\varepsilon}_i,-1),\text{ } i\in \{51,52,\cdots,100\},
\end{gather*}
where $\boldsymbol{\varepsilon}_i$ ($i=1,2,3,\cdots,100$) are random variables sampled uniformly and i.i.d. from $ [-0.6,0.6]\times [-0.6,0.6]$. We visualize  the dataset in  Figure \ref{fig:visualizing dataset}.
\begin{figure*}[t!]
\centering
\begin{subfigure}
[b]{0.3\columnwidth}        \includegraphics[trim=0cm 0.5cm 0cm 0cm, width=\textwidth]{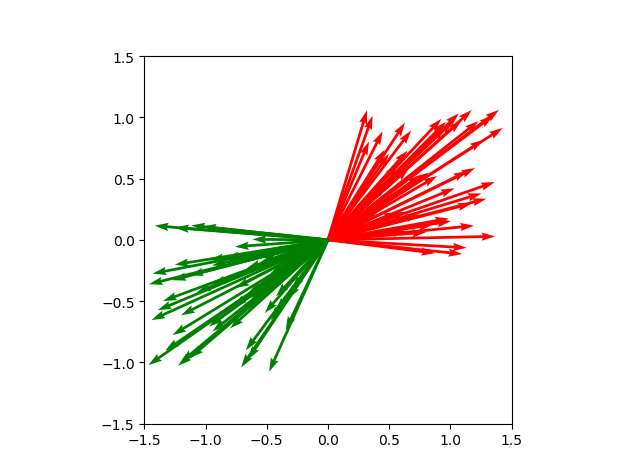}
\caption{ Visualizing Constructed Dataset  }
\label{fig:visualizing dataset}
\end{subfigure}
\begin{subfigure}
  [b]{0.3\columnwidth}        \includegraphics[trim=0cm 0.5cm 0cm 0cm, width=\textwidth]{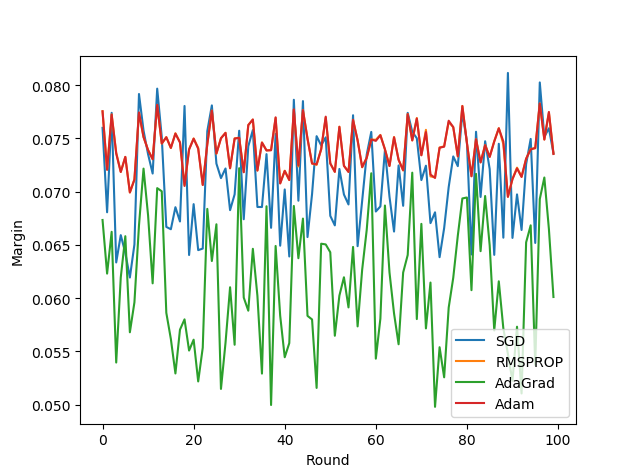}
\caption{ Margin varies across runs }
\label{fig:margin_toy}
\end{subfigure}
\begin{subfigure}
  [b]{0.3\columnwidth}        \includegraphics[trim=0cm 0.5cm 0cm 0cm, width=\textwidth]{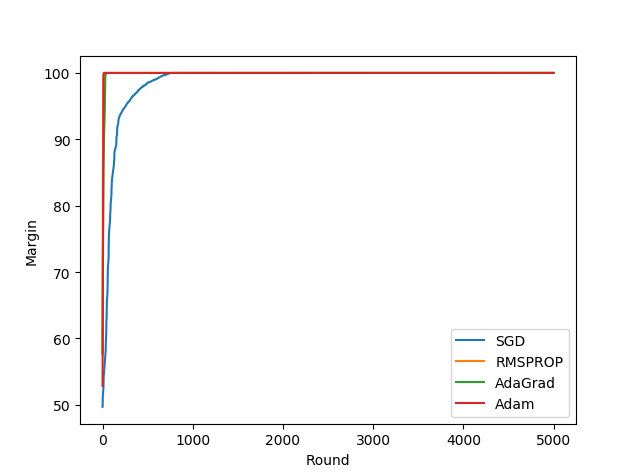}
\caption{Training accuracy}
\label{fig:training_toy}
\end{subfigure}

\begin{subfigure}
[b]{0.3\columnwidth}        \includegraphics[trim=0cm 0.5cm 0cm 0cm, width=\textwidth]{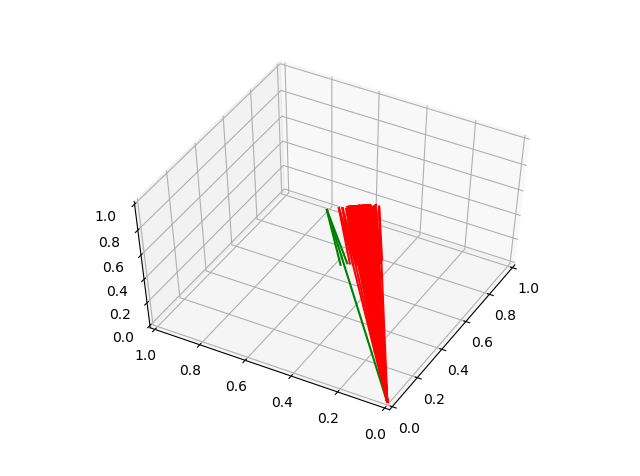}
\caption{ Direction of $\boldsymbol{h}_{\infty}^{-\frac{1}{2}}$ in AdaGrad}
\label{fig: directon_ada}
\end{subfigure}
\begin{subfigure}
  [b]{0.3\columnwidth}        \includegraphics[trim=0cm 0.5cm 0cm 0cm, width=\textwidth]{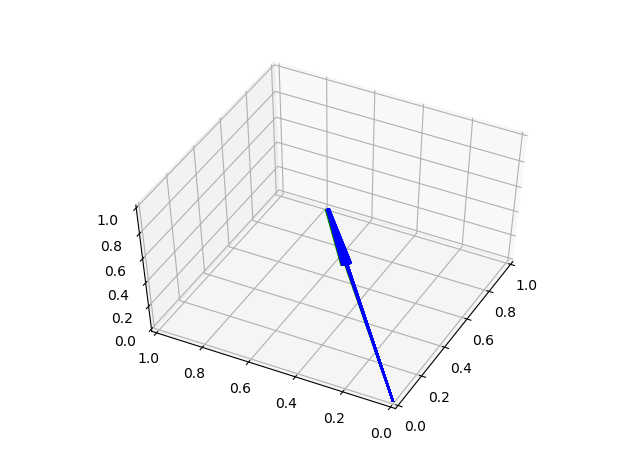}
\caption{ Direction of $\boldsymbol{h}_{\infty}^{-\frac{1}{2}}$ in RMSProp}
\label{fig: direction_rms}
\end{subfigure}
\begin{subfigure}
  [b]{0.3\columnwidth}        \includegraphics[trim=0cm 0.5cm 0cm 0cm, width=\textwidth]{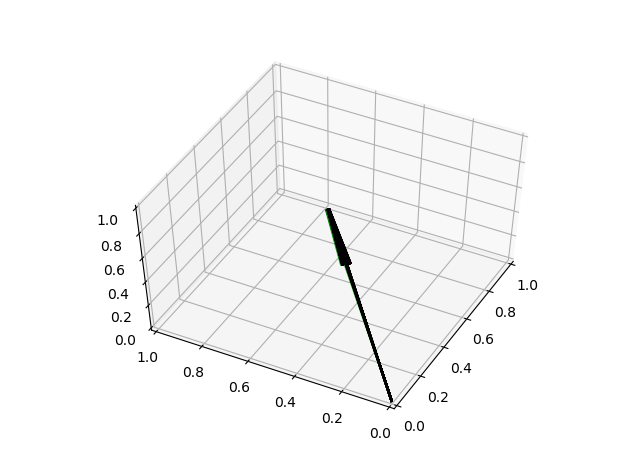}
\caption{Direction of $\boldsymbol{h}_{\infty}^{-\frac{1}{2}}$ in Adam}
\label{fig: direction_adam}
\end{subfigure}
\vskip -0.3cm
\caption{Experiment on two layer MLP. In (a), we visualize constructed dataset as vectors, with green vectors standing for data with label $-1$, and red stands for data with label $1$. In (b), we plot convergent margin of different optimizers across runs. In (c), averaged training accuracy across runs is shown, and one can observe all training accuracy achieves $100\%$. (d)-(e) respectively picture $\boldsymbol{h}_{\infty}^{-\frac{1}{2}}$ in AdaGrad (red vectors), RMSProp (blue vectors) and Adam (black vectors) across runs. While direction of $\boldsymbol{h}_{\infty}^{-\frac{1}{2}}$ in AdaGrad varies across runs, RMSProp and Adam stay the same and coincide with isotropic direction (green vector) $(\frac{1}{\sqrt{3}},\frac{1}{\sqrt{3}}, \frac{1}{\sqrt{3}})$}
\label{fig:MLP_two_layer}
\vskip -0.3cm
\end{figure*}

We then run SGD, AdaGrad, RMSProp (and Adam (w/m)) respectively with learning rates $\eta=0.1$, while Weight-decay hyper-parameter $b$ is set to be  
$0.9$. For each round, we train the model for $5000$ epochs to ensure that training accuracy achieves $100\%$ (see Figure \ref{fig:training_toy} for details); while for each optimizer, we conduct $100$ rounds of experiments with random initialization,  Convergent directions of square root of inverse conditioners $\boldsymbol{h}_{\infty}^{-\frac{1}{2}}$ are plotted in Figures \ref{fig: directon_ada}, \ref{fig: direction_rms}, and \ref{fig: direction_adam}. Since $\boldsymbol{h}_{\infty}^{-\frac{1}{2}}$ occurs in optimization target in $(P^A)$, different direction of $\boldsymbol{h}_{\infty}^{-\frac{1}{2}}$ may lead to different convergent direction of parameters, which further indicates convergent direction of parameters in AdaGrad can be vulnerable to random initialization.



%% file: Definable.tex
\section{Proof of Theorem \ref{thm:approximate_flow_definable}}

\label{sec:proof_definable}
In this section, we will prove that direction of parameters converges, that is, $\lim_{t\rightarrow\infty} \frac{\boldsymbol{v}(t)}{\Vert \boldsymbol{v}(t)\Vert}$ exists. Concretely, define the length swept by $\frac{\boldsymbol{v}(t)}{\Vert \boldsymbol{v}(t)\Vert}$ as $\zeta(t)$, i.e.,
\begin{equation*}
    \zeta(t)=\int_{0}^{t} \left\Vert\frac{\mathrm{d} \hat{\boldsymbol{v}}(\tau)}{\mathrm{d} \tau}\right\Vert \mathrm{d} \tau.
\end{equation*}
We will upper bound $\zeta(t)$ in the rest of this section.

To begin with, we first show that if the network is definable, then $\partial^s \tilde{\mL} (\bv(t))$ in adaptive flow is actually $\bar{\partial} \tilde{\mL} (\bv(t))$.

\begin{lemma}
\label{lem: optimal gradient}
If the neural network $\Phi(\bw,\bx)$ is definable with respect to $\bw$ for any $\bx$, for any $\bv$ satisfying the following adaptive gradient flow
\begin{equation*}
    \frac{\mathrm{d}\bv(t)}{\mathrm{d} t}=-\bbeta(t) \odot \partial^s \tilde{\mL}(\bv(t)), \text{a.e.} t>0,
\end{equation*}
with $\partial^s \tilde{\mL}(\bv(t))\in \partial \tilde{\mL} (\bv(t))$. Then $\Vert \bbeta(t)^{\frac{1}{2}}\odot \partial^s \tilde{\mL}(\bv(t))\Vert =\operatorname{dist} (0,\bbeta(t)^{\frac{1}{2}}\odot \partial \tilde{\mL}(\bv(t)))$.
\end{lemma}
The proof of Lemma \ref{lem: optimal gradient} follows the same routine as that of Lemma 5.2 in \citep{davis2020stochastic}, and we omit it here.

We then define another surrogate margin $\bar\gamma$ as
\begin{equation*}
    \bar{\gamma}(\boldsymbol{v})=\frac{g(\log \frac{1}{\tilde{\mathcal{L}}(\boldsymbol{v})})}{\Vert \boldsymbol{v}\Vert^L}+e^{-\tilde{\mathcal{L}}(\boldsymbol{v})}.
\end{equation*}

The following Lemma then lower bound the derivative of $\bar{\gamma}$.

\begin{lemma}
\label{lem:lower_bound_bar_gamma}
For large enough $t$, we have a.e., 
\begin{equation*}
    \frac{\mathrm{d} \bar{\gamma}(\boldsymbol{v}(t)) }{\mathrm{d} t}\ge \frac{1}{2}\left(\left\|\bar{\partial}_{\backslash\backslash} \bar{\gamma}\left(\boldsymbol{v}(t)\right)\right\|\left\|\bar{\partial}_{\backslash\backslash} \tilde{\mathcal{L}}\left(\boldsymbol{v}(t)\right)\right\|+\left\|\bar{\partial}_{\perp} \bar{\gamma}\left(\boldsymbol{v}(t)\right)\right\|\left\|\bar{\partial}_{\perp} \tilde{\mathcal{L}}\left(\boldsymbol{v}(t)\right)\right\|\right),
\end{equation*}
and 
\begin{equation*}
    \frac{\mathrm{d} \zeta(t)}{\mathrm{d} t}=\frac{\left\|(\boldsymbol{\beta}(t)\odot\partial^s \tilde{\mathcal{L}}\left(\boldsymbol{v}(t)\right))_{\perp}\right\|}{\left\|\boldsymbol{v}(t)\right\|}.
\end{equation*}
\end{lemma}
\begin{proof}
To begin with, we calculate the rate of $\boldsymbol{\beta}$ converging to $\mathbf{1}_p$. For AdaGrad, given a fixed index $i\in[N]$, we have that,
\begin{align}
\nonumber
    0\le\boldsymbol{\beta}_i(t)-1=&(\boldsymbol{h}_{\infty})^{-1}_i(\boldsymbol{h}_i(t)-(\boldsymbol{h}_{\infty})_i)
    =\Theta(\boldsymbol{h}_i(t)-(\boldsymbol{h}_{\infty})_i)
    \\
    \nonumber
    =&\Theta\left(\int_{t}^{\infty} \partial^s \mathcal{L}(\boldsymbol{w}(\tau))^2 d\tau\right)
    =\Theta\left(\int_{t}^{\infty} \frac{1}{\tau^2(\log \tau)^{2-\frac{2}{L}}} d\tau\right)
    \\
    \label{eq:rate_of_beta}
    =&\mathcal{O}\left(\frac{1}{t(\log t)^{2-\frac{2}{L}}} \right).
\end{align}

Similarly, for RMSProp, given a fixed index $i\in [N]$, 
\begin{align*}
    0&\le 1-\frac{\sqrt{\varepsilon}}{\sqrt{\varepsilon+(1-b)\int_{0}^t e^{-(1-b)(t-\tau)}(\partial^s \mathcal{L}(\tau))_i^2\mathrm{d}\tau}}
    =\Theta\left((1-b)\int_{0}^t e^{-(1-b)(t-\tau)} (\partial^s \mathcal{L}(\tau))_i^2\mathrm{d}\tau\right)
    \\
    =&\Theta\left((1-b)\int_{0}^{t-\sqrt{t}}e^{-(1-b)(t-\tau)} (\partial^s \mathcal{L}(\tau))_i^2\mathrm{d}\tau\right)+\Theta\left((1-b)\int_{t-\sqrt{t}}^{t}e^{-(1-b)(t-\tau)} (\partial^s \mathcal{L}(\tau))_i^2\mathrm{d}\tau\right)
    \\
    = &\mathcal{O}\left((1-b)e^{-(1-b)\sqrt{t}}\int_{0}^{\infty} (\partial^s {\mathcal{L}}(\tau))_i^2\mathrm{d}\tau\right)+\mathcal{O}\left((1-b)\frac{\sqrt{t}}{(t-\sqrt{t})^2}\right)
    \\
    =&\mathcal{O}\left(\frac{1}{t^{\frac{3}{2}}}\right)
     =\mathcal{O}\left(\frac{1}{t(\log t)^{2-\frac{2}{L}}} \right).
\end{align*}

We then directly calculate the derivative of $\bar{\gamma}$:
\begin{align*}
    \frac{\mathrm{d} \bar{\gamma}(\boldsymbol{v}(t)) }{\mathrm{d} t}=&\langle\partial \bar{\gamma}(\boldsymbol{v}(t)),\frac{d \boldsymbol{v}(t)}{dt} \rangle
    \\
    =&\left\langle-g'\left(\log\frac{1}{\tilde{\mathcal{L}}}\right)\frac{\partial^s \tilde{\mathcal{L}}}{\tilde{\mathcal{L}}\Vert\boldsymbol{v}(t)\Vert^L}-L\frac{g\left(\log \frac{1}{\tilde{\mathcal{L}}}\right)}{\Vert \boldsymbol{v}(t)\Vert^{L+1}}\hat{\boldsymbol{v}}(t)-e^{-\tilde{\mathcal{L}}}\partial^s \tilde{\mathcal{L}},\frac{d \boldsymbol{v}(t)}{dt}\right \rangle
    \\
    =&\left\langle g'\left(\log\frac{1}{\tilde{\mathcal{L}}}\right)\frac{\partial^s \tilde{\mathcal{L}}}{\tilde{\mathcal{L}}\Vert\boldsymbol{v}(t)\Vert^L}+L\frac{g\left(\log \frac{1}{\tilde{\mathcal{L}}}\right)}{\Vert \boldsymbol{v}(t)\Vert^{L+1}}\hat{\boldsymbol{v}}(t)+e^{-\tilde{\mathcal{L}}}\partial^s \tilde{\mathcal{L}},\boldsymbol{\beta}(t)\odot \partial^s \tilde{\mathcal{L}}\right \rangle
    \\
    =&\left\langle g'\left(\log\frac{1}{\tilde{\mathcal{L}}}\right)\frac{\partial^s \tilde{\mathcal{L}}}{\tilde{\mathcal{L}}\Vert\boldsymbol{v}(t)\Vert^L}, \partial^s \tilde{\mathcal{L}}\right \rangle
    +\left\langle L\frac{g\left(\log \frac{1}{\tilde{\mathcal{L}}}\right)}{\Vert \boldsymbol{v}(t)\Vert^{L+1}}\hat{\boldsymbol{v}}(t), \partial^s \tilde{\mathcal{L}}\right \rangle
    +\left\langle e^{-\tilde{\mathcal{L}}}\partial^s \tilde{\mathcal{L}},\boldsymbol{\beta}(t)\odot \partial^s \tilde{\mathcal{L}}\right \rangle
    \\
    +&
    \left\langle g'\left(\log\frac{1}{\tilde{\mathcal{L}}}\right)\frac{\partial^s \tilde{\mathcal{L}}}{\tilde{\mathcal{L}}\Vert\boldsymbol{v}(t)\Vert^L}+L\frac{g\left(\log \frac{1}{\tilde{\mathcal{L}}}\right)}{\Vert \boldsymbol{v}(t)\Vert^{L+1}}\hat{\boldsymbol{v}}(t),(\boldsymbol{\beta}(t)-\mathbf{1})\odot \partial^s \tilde{\mathcal{L}}\right \rangle
    \\
    \overset{(\diamond)}{\ge }&\left\langle g'\left(\log\frac{1}{\tilde{\mathcal{L}}}\right)\frac{\bar{\partial} \tilde{\mathcal{L}}}{\tilde{\mathcal{L}}\Vert\boldsymbol{v}(t)\Vert^L},\bbeta(t)\odot \bar{\partial} \tilde{\mathcal{L}}\right \rangle
    +\left\langle L\frac{g\left(\log \frac{1}{\tilde{\mathcal{L}}}\right)}{\Vert \boldsymbol{v}(t)\Vert^{L+1}}\hat{\boldsymbol{v}}(t), \bar{\partial} \tilde{\mathcal{L}}\right \rangle
    +\left\langle e^{-\tilde{\mathcal{L}}}\partial^s \tilde{\mathcal{L}},\boldsymbol{\beta}(t)\odot \partial^s \tilde{\mathcal{L}}\right \rangle
    \\
    +&
   \left\langle g'\left(\log\frac{1}{\tilde{\mathcal{L}}}\right)\frac{\partial^s \tilde{\mathcal{L}}}{\tilde{\mathcal{L}}\Vert\boldsymbol{v}(t)\Vert^L}+L\frac{g\left(\log \frac{1}{\tilde{\mathcal{L}}}\right)}{\Vert \boldsymbol{v}(t)\Vert^{L+1}}\hat{\boldsymbol{v}}(t),(\boldsymbol{\beta}(t)-\mathbf{1})\odot \partial^s \tilde{\mathcal{L}}\right \rangle
\end{align*}
where Eq. ($\diamond$) is due to $\bar{\partial} \tilde{\mL}$ has the smallest norm among $\partial \tilde{\mL}$ and the homogeneity of the neural network.

Let 
\begin{gather*}
A\overset{\triangle}{=}\left\langle e^{-\tilde{\mathcal{L}}}\partial^s \tilde{\mathcal{L}},\boldsymbol{\beta}(t)\odot \partial^s \tilde{\mathcal{L}}\right \rangle,
\\
    B\overset{\triangle}{=} \left\langle g'\left(\log\frac{1}{\tilde{\mathcal{L}}}\right)\frac{\partial^s \tilde{\mathcal{L}}}{\tilde{\mathcal{L}}\Vert\boldsymbol{v}(t)\Vert^L}+L\frac{g\left(\log \frac{1}{\tilde{\mathcal{L}}}\right)}{\Vert \boldsymbol{v}(t)\Vert^{L+1}}\hat{\boldsymbol{v}}(t),(\boldsymbol{\beta}(t)-\mathbf{1})\odot \partial^s \tilde{\mathcal{L}}\right \rangle.
\end{gather*}
On the one hand, as $e^{-\tilde{\mL}}\rightarrow 1$ and $\bbeta(t)\rightarrow \boldsymbol{1}$ as $t\rightarrow \infty$, we have 
\begin{equation*}
    A=(1+\boldsymbol{o}(1)) \Vert \partial^s \tilde{\mathcal{L}} (\bv(t))\Vert\ge (1+\boldsymbol{o}(1)) \Vert \bar{\partial} \tilde{\mathcal{L}} (\bv(t))\Vert.
\end{equation*}

On the other hand, we have 

\begin{align*}
     &\vert B\vert 
     \\
     \le&\left\vert\left\langle g'\left(\log\frac{1}{\tilde{\mathcal{L}}}\right)\frac{\partial^s \tilde{\mathcal{L}}}{\tilde{\mathcal{L}}\Vert\boldsymbol{v}(t)\Vert^L}+L\frac{g\left(\log \frac{1}{\tilde{\mathcal{L}}}\right)}{\Vert \boldsymbol{v}(t)\Vert^{L+1}}\hat{\boldsymbol{v}},(\mathbf{1}-\boldsymbol{\beta}(t))\odot \partial^s \tilde{\mathcal{L}}\right \rangle\right\vert
     \\
     \le& \Vert \mathbf{1}-\boldsymbol{\beta}(t) \Vert_{\infty} \left\Vert g'\left(\log\frac{1}{\tilde{\mathcal{L}}}\right)\frac{\partial^s \tilde{\mathcal{L}}}{\tilde{\mathcal{L}}\Vert\boldsymbol{v}(t)\Vert^L}+L\frac{g\left(\log \frac{1}{\tilde{\mathcal{L}}}\right)}{\Vert \boldsymbol{v}(t)\Vert^{L+1}}\hat{\boldsymbol{v}}\right \Vert \Vert \partial^s \tilde{\mathcal{L}} \Vert
     \\
     \le & \mathcal{O}\left(\frac{1}{t(\log t)^{2-\frac{2}{L}}} \right)\mathcal{O}\left(\frac{1}{\Vert\boldsymbol{v}(t) \Vert}\right)\Vert \partial^s \tilde{\mathcal{L}} \Vert
     \\
     =&\mathcal{O}\left(\tilde{\mathcal{L}}(\boldsymbol{v}(t)) \right)\mathcal{O}\left(\frac{1}{\Vert\boldsymbol{v}(t) \Vert}\right)\Vert \partial^s \tilde{\mathcal{L}} \Vert
     \\
     =&\mathcal{O}\left(\frac{1}{\Vert\boldsymbol{v}(t)\Vert^L}\right)\Vert \partial^s \tilde{\mathcal{L}} \Vert^2
     \\
     \le&\mathcal{O}\left(\frac{1}{\Vert\boldsymbol{v}(t)\Vert^L}\right)\Vert \bar{\partial} \tilde{\mathcal{L}} \Vert^2,
\end{align*}
where the last Inequality is due to Lemma \ref{lem: optimal gradient} and $\bar{\partial} \tilde{\mathcal{L}}$ having the smallest norm.

On the other hand, 
\begin{equation*}
    \left\langle e^{-\tilde{\mathcal{L}}(\boldsymbol{v}(t))}\bar{\partial} \tilde{\mathcal{L}}(\boldsymbol{v}(t)),\boldsymbol{\beta}(t)\odot \bar{\partial} \tilde{\mathcal{L}}(\boldsymbol{v}(t)) \right \rangle=\Theta(\Vert \bar{\partial} \tilde{\mathcal{L}}(\boldsymbol{v}(t)) \Vert^2).
\end{equation*}

Therefore, there exists a large enough $T$, such that, any $t\ge T$,
\begin{align*}
    A+B\ge\frac{1}{2} \left\langle e^{-\tilde{\mathcal{L}}(\boldsymbol{v}(t)) }\bar{\partial} \tilde{\mathcal{L}}(\boldsymbol{v}(t)) , \bar{\partial} \tilde{\mathcal{L}}(\boldsymbol{v}(t)) \right \rangle,
\end{align*}
which further leads to 
\begin{equation*}
    \frac{\mathrm{d} \bar{\gamma}(\boldsymbol{v}(t)) }{\mathrm{d} t}\ge \frac{1}{2} \left\langle  e^{-\tilde{\mathcal{L}}(\boldsymbol{v}(t))} \bar{\partial}\tilde{\mathcal{L}}(\boldsymbol{v}(t)), \bar{\partial} \tilde{\mathcal{L}}(\boldsymbol{v}(t))\right \rangle+\left\langle\frac{g'\left(\log \frac{1}{\tilde{\mathcal{L}}}\right)\bar{\partial} \tilde{\mathcal{L}}(\boldsymbol{v}(t))}{\tilde{\mathcal{L}}(\boldsymbol{v}(t))\Vert\boldsymbol{v}(t)\Vert^L}+L\frac{g\left(\log \frac{1}{\tilde{\mathcal{L}}(\boldsymbol{v}(t))}\right)}{\Vert \boldsymbol{v}(t)\Vert^{L+1}}\hat{\boldsymbol{v}}, \bar{\partial} \tilde{\mathcal{L}}(\boldsymbol{v}(t))\right  \rangle.
\end{equation*}
 We then calculate  axial component and radial component of $ e^{-\tilde{\mathcal{L}}(\boldsymbol{v})} \bar{\partial}\tilde{\mathcal{L}}(\boldsymbol{v})$, $g'\left(\log \frac{1}{\tilde{\mathcal{L}}}\right)\frac{\bar{\partial} \tilde{\mathcal{L}}(\boldsymbol{v})}{\tilde{\mathcal{L}}(\boldsymbol{v})\Vert\boldsymbol{v}\Vert^L}+L\frac{g\left(\log \frac{1}{\tilde{\mathcal{L}}(\boldsymbol{v})}\right)}{\Vert \boldsymbol{v}\Vert^{L+1}}\hat{\boldsymbol{v}}$, and $\bar{\partial} \tilde{\mathcal{L}}(\boldsymbol{v})$ respectively as follows:
 \begin{gather*}
     \left( e^{-\tilde{\mathcal{L}}(\boldsymbol{v})} \bar{\partial}\tilde{\mathcal{L}}(\boldsymbol{v})\right)_{\backslash\backslash}=\langle e^{-\tilde{\mathcal{L}}(\boldsymbol{v})}\bar{\partial} \tilde{\mathcal{L}}(\boldsymbol{v}),\hat{\boldsymbol{v}}\rangle  \hat{\boldsymbol{v}},
     \\
     \left( e^{-\tilde{\mathcal{L}}(\boldsymbol{v})} \bar{\partial}\tilde{\mathcal{L}}(\boldsymbol{v})\right)_{\perp} =e^{-\tilde{\mathcal{L}}(\boldsymbol{v})}(\bar{\partial} \tilde{\mathcal{L}}(\boldsymbol{v})-\langle \bar{\partial} \tilde{\mathcal{L}}(\boldsymbol{v}),\hat{\boldsymbol{v}}\rangle  \hat{\boldsymbol{v}}),
     \\
     \left(g'\left(\log \frac{1}{\tilde{\mathcal{L}}}\right)\frac{\bar{\partial} \tilde{\mathcal{L}}(\boldsymbol{v})}{\tilde{\mathcal{L}}(\boldsymbol{v})\Vert\boldsymbol{v}\Vert^L}+L\frac{g\left(\log \frac{1}{\tilde{\mathcal{L}}(\boldsymbol{v})}\right)}{\Vert \boldsymbol{v}\Vert^{L+1}}\hat{\boldsymbol{v}}\right)_{\backslash\backslash}=L\frac{g\left(\log \frac{1}{\mathcal{L}(\boldsymbol{v})}\right)}{\Vert \boldsymbol{v}\Vert^{L+1}}\hat{\boldsymbol{v}}+g'\left(\log \frac{1}{\tilde{\mathcal{L}}}\right)\frac{\langle\bar{\partial} \tilde{\mathcal{L}}(\boldsymbol{v}), \hat{\boldsymbol{v}}\rangle}{\tilde{\mathcal{L}}(\boldsymbol{v})\Vert\boldsymbol{v}\Vert^L}\hat{\boldsymbol{v}},
     \\
      \left(g'\left(\log \frac{1}{\tilde{\mathcal{L}}}\right)\frac{\bar{\partial} \tilde{\mathcal{L}}(\boldsymbol{v})}{\tilde{\mathcal{L}}(\boldsymbol{v})\Vert\boldsymbol{v}\Vert^L}+L\frac{g\left(\log \frac{1}{\tilde{\mathcal{L}}(\boldsymbol{v})}\right)}{\Vert \boldsymbol{v}\Vert^{L+1}}\hat{\boldsymbol{v}}\right)_{\perp}=
      g'\left(\log \frac{1}{\tilde{\mathcal{L}}}\right)\frac{\bar{\partial} \tilde{\mathcal{L}}(\boldsymbol{v})-\langle\bar{\partial} \tilde{\mathcal{L}}(\boldsymbol{v}), \hat{\boldsymbol{v}}\rangle\hat{\boldsymbol{v}}}{\tilde{\mathcal{L}}\Vert\boldsymbol{v}\Vert^L},
      \\
      \left(\bar{\partial} \tilde{\mathcal{L}}(\boldsymbol{v})\right)_{\backslash\backslash}=\langle \bar{\partial} \tilde{\mathcal{L}}(\boldsymbol{v}),\hat{\boldsymbol{v}}\rangle  \hat{\boldsymbol{v}}
      \\
       \left(\bar{\partial} \tilde{\mathcal{L}}(\boldsymbol{v})\right)_{\perp}=\bar{\partial} \tilde{\mathcal{L}}(\boldsymbol{v})-\langle \bar{\partial} \tilde{\mathcal{L}}(\boldsymbol{v}),\hat{\boldsymbol{v}}\rangle  \hat{\boldsymbol{v}}.
 \end{gather*}

Therefore,
\begin{align*}
    &\frac{1}{2} \left\langle  e^{-\tilde{\mathcal{L}}(\boldsymbol{v}(t))} \bar{\partial}\tilde{\mathcal{L}}(\boldsymbol{v}(t)), \bar{\partial} \tilde{\mathcal{L}}(\boldsymbol{v}(t))\right \rangle+\left\langle g'\left(\log \frac{1}{\tilde{\mathcal{L}}}\right)\frac{\bar{\partial} \tilde{\mathcal{L}}(\boldsymbol{v}(t))}{\tilde{\mathcal{L}}(\boldsymbol{v}(t))\Vert\boldsymbol{v}(t)\Vert^L}+L\frac{\log \frac{1}{\tilde{\mathcal{L}}(\boldsymbol{v}(t))}}{\Vert \boldsymbol{v}\Vert^{L+1}}\hat{\boldsymbol{v}}(t), \bar{\partial} \tilde{\mathcal{L}}(\boldsymbol{v}(t))\right \rangle
    \\
    =&\frac{1}{2} \left\langle  \left( e^{-\tilde{\mathcal{L}}(\boldsymbol{v}(t))} \bar{\partial}\tilde{\mathcal{L}}(\boldsymbol{v}(t))\right)_{\backslash\backslash}+\left( e^{-\tilde{\mathcal{L}}(\boldsymbol{v}(t))} \bar{\partial}\tilde{\mathcal{L}}(\boldsymbol{v}(t))\right)_{\perp} , \left(\bar{\partial} \tilde{\mathcal{L}}(\boldsymbol{v}(t))\right)_{\backslash\backslash}+\left(\bar{\partial} \tilde{\mathcal{L}}(\boldsymbol{v}(t))\right)_{\perp}\right \rangle
    \\
    +&\left\langle\left(g'\left(\log \frac{1}{\tilde{\mathcal{L}}}\right)\frac{\bar{\partial} \tilde{\mathcal{L}}(\boldsymbol{v}(t))}{\tilde{\mathcal{L}}(\boldsymbol{v}(t))\Vert\boldsymbol{v}(t)\Vert^L}+L\frac{g\left(\log \frac{1}{\tilde{\mathcal{L}}(\boldsymbol{v}(t))}\right)}{\Vert \boldsymbol{v}(t)\Vert^{L+1}}\hat{\boldsymbol{v}}(t)\right)_{\backslash\backslash}+ \left(g'\left(\log \frac{1}{\tilde{\mathcal{L}}}\right)\frac{\bar{\partial} \tilde{\mathcal{L}}(\boldsymbol{v}(t))}{\tilde{\mathcal{L}}(\boldsymbol{v}(t))\Vert\boldsymbol{v}(t)\Vert^L}+L\frac{g\left(\log \frac{1}{\tilde{\mathcal{L}}(\boldsymbol{v}(t))}\right)}{\Vert \boldsymbol{v}(t)\Vert^{L+1}}\hat{\boldsymbol{v}}(t)\right)_{\perp},\right.
    \\
    \quad&\left.\left(\bar{\partial} \tilde{\mathcal{L}}(\boldsymbol{v}(t))\right)_{\backslash\backslash}+\left(\bar{\partial} \tilde{\mathcal{L}}(\boldsymbol{v}(t))\right)_{\perp}\right \rangle
    \\
    \overset{(*)}{=}&\frac{1}{2} \left(\left\Vert \left(e^{-\tilde{\mathcal{L}}(\boldsymbol{v}(t))} \bar{\partial}\tilde{\mathcal{L}}(\boldsymbol{v}(t))\right)_{\backslash\backslash}\right\Vert \left\Vert \left(\bar{\partial} \tilde{\mathcal{L}}(\boldsymbol{v}(t))\right)_{\backslash\backslash}\right\Vert+\left\Vert \left(e^{-\tilde{\mathcal{L}}(\boldsymbol{v}(t))} \bar{\partial}\tilde{\mathcal{L}}(\boldsymbol{v}(t))\right)_{\perp}\right\Vert \left\Vert \left(\bar{\partial} \tilde{\mathcal{L}}(\boldsymbol{v}(t))\right)_{\perp}\right\Vert\right)
    \\
    +& \left\Vert\left(g'\left(\log \frac{1}{\tilde{\mathcal{L}}}\right)\frac{\bar{\partial} \tilde{\mathcal{L}}(\boldsymbol{v}(t))}{\tilde{\mathcal{L}}(\boldsymbol{v}(t))\Vert\boldsymbol{v}(t)\Vert^L}+L\frac{g\left(\log \frac{1}{\tilde{\mathcal{L}}(\boldsymbol{v}(t))}\right)}{\Vert \boldsymbol{v}(t)\Vert^{L+1}}\hat{\boldsymbol{v}}(t)\right)_{\backslash\backslash}\right\Vert \left\Vert \left(\bar{\partial} \tilde{\mathcal{L}}(\boldsymbol{v}(t))\right)_{\backslash\backslash}\right\Vert
    \\
    +&\left\Vert\left(g'\left(\log \frac{1}{\tilde{\mathcal{L}}}\right)\frac{\bar{\partial} \tilde{\mathcal{L}}(\boldsymbol{v}(t))}{\tilde{\mathcal{L}}(\boldsymbol{v}(t))\Vert\boldsymbol{v}(t)\Vert^L}+L\frac{g\left(\log \frac{1}{\tilde{\mathcal{L}}(\boldsymbol{v}(t))}\right)}{\Vert \boldsymbol{v}(t)\Vert^{L+1}}\hat{\boldsymbol{v}}(t)\right)_{\perp}\right\Vert \left\Vert \left(\bar{\partial} \tilde{\mathcal{L}}(\boldsymbol{v}(t))\right)_{\perp}\right\Vert
    \\
    \ge &\frac{1}{2}\left\Vert\left(e^{-\tilde{\mathcal{L}}(\boldsymbol{v}(t))} \bar{\partial}\tilde{\mathcal{L}}(\boldsymbol{v}(t))+g'\left(\log \frac{1}{\tilde{\mathcal{L}}}\right)\frac{\bar{\partial} \tilde{\mathcal{L}}(\boldsymbol{v}(t))}{\tilde{\mathcal{L}}(\boldsymbol{v}(t))\Vert\boldsymbol{v}(t)\Vert^L}+L\frac{g\left(\log \frac{1}{\tilde{\mathcal{L}}(\boldsymbol{v}(t))}\right)}{\Vert \boldsymbol{v}(t)\Vert^{L+1}}\hat{\boldsymbol{v}}(t)\right)_{\backslash\backslash}\right\Vert \left\Vert \left(\bar{\partial} \tilde{\mathcal{L}}(\boldsymbol{v}(t))\right)_{\backslash\backslash}\right\Vert
    \\
    +&\frac{1}{2}\left\Vert\left(e^{-\tilde{\mathcal{L}}(\boldsymbol{v}(t))} \bar{\partial}\tilde{\mathcal{L}}(\boldsymbol{v}(t))+g'\left(\log \frac{1}{\tilde{\mathcal{L}}}\right)\frac{\bar{\partial} \tilde{\mathcal{L}}(\boldsymbol{v}(t))}{\tilde{\mathcal{L}}(\boldsymbol{v}(t))\Vert\boldsymbol{v}(t)\Vert^L}+L\frac{g\left(\log \frac{1}{\tilde{\mathcal{L}}(\boldsymbol{v}(t))}\right)}{\Vert \boldsymbol{v}(t)\Vert^{L+1}}\hat{\boldsymbol{v}}(t)\right)_{\perp}\right\Vert \left\Vert \left(\bar{\partial} \tilde{\mathcal{L}}(\boldsymbol{v}(t))\right)_{\perp}\right\Vert,
\end{align*}
where eq. $(*)$ is because 
 \begin{gather*}
   \langle e^{-\tilde{\mathcal{L}}(\boldsymbol{v}(t))}\bar{\partial} \tilde{\mathcal{L}}(\boldsymbol{v}(t)),\hat{\boldsymbol{v}}(t)\rangle= -e^{-\tilde{\mathcal{L}}(\boldsymbol{v}(t))}\langle -\bar{\partial} \tilde{\mathcal{L}}(\boldsymbol{v}(t)),\hat{\boldsymbol{v}}(t)\rangle<0,
     \\
     e^{-\tilde{\mathcal{L}}(\boldsymbol{v}(t))}>0,
     \\
   L\frac{g\left(\log \frac{1}{\tilde{\mathcal{L}}(\boldsymbol{v}(t))}\right)}{\Vert \boldsymbol{v}(t)\Vert^{L+1}}+g'\left(\log \frac{1}{\tilde{\mathcal{L}}(\boldsymbol{v}(t))}\right)\frac{\langle\bar{\partial} \tilde{\mathcal{L}}(\boldsymbol{v}(t)), \hat{\boldsymbol{v}}(t)\rangle}{\tilde{\mathcal{L}}(\boldsymbol{v}(t))\Vert\boldsymbol{v}(t)\Vert^L}
   \\
   =-\frac{1}{\tilde{\mathcal{L}}(\boldsymbol{v}(t))\Vert \boldsymbol{v}(t)\Vert^{L+1}}\left(L\nu(t)g'\left(\log \frac{1}{\tilde{\mathcal{L}}(\boldsymbol{v}(t))}\right)-\tilde{\mathcal{L}}(\boldsymbol{v}(t))g\left(\log \frac{1}{\tilde{\mathcal{L}}(\boldsymbol{v}(t))}\right)\right)<0,
     \\
     \frac{1}{\tilde{\mathcal{L}}(\boldsymbol{v}(t))\Vert\boldsymbol{v}(t)\Vert^L}>0,
      \\
      \langle \bar{\partial} \tilde{\mathcal{L}}(\boldsymbol{v}(t)),\hat{\boldsymbol{v}}(t)\rangle =-\langle -\bar{\partial} \tilde{\mathcal{L}}(\boldsymbol{v}(t)),\hat{\boldsymbol{v}}(t)\rangle<0.
 \end{gather*}

The proof of the first claim is completed since $\bar{\partial} \bar{\gamma}(\boldsymbol{v})=-e^{-\tilde{\mathcal{L}}(\boldsymbol{v})} \bar{\partial}\tilde{\mathcal{L}}(\boldsymbol{v})-g'\left(\log \frac{1}{\tilde{\mathcal{L}}}\right)\frac{\bar{\partial} \tilde{\mathcal{L}}(\boldsymbol{v})}{\tilde{\mathcal{L}}(\boldsymbol{v})\Vert\boldsymbol{v}\Vert^L}-L\frac{g\left(\log \frac{1}{\tilde{\mathcal{L}}(\boldsymbol{v})}\right)}{\Vert \boldsymbol{v}\Vert^{L+1}}\hat{\boldsymbol{v}}$.

As for the second claim, we have a.e., 
\begin{equation*}
    \frac{\mathrm{d} \zeta}{\mathrm{d} t}=\left\Vert\frac{\mathrm{d} \hat{\bv} (t)}{\mathrm{d} t} \right\Vert = \frac{\left\|(\boldsymbol{\beta}(t)\odot\partial^s \tilde{\mathcal{L}}\left(\boldsymbol{v}(t)\right))_{\perp}\right\|}{\left\|\boldsymbol{v}(t)\right\|}. 
\end{equation*}
The proof is completed.

\end{proof}

The following lemma gives an equivalent proposition of that the curve length  $\zeta$ is finite.
\begin{lemma}
$\zeta$ is finite if
\begin{equation*}
    \int_{0}^{\infty}\frac{\left\|\bar{\partial}_{\perp} \tilde{\mathcal{L}}\left(\boldsymbol{v}(t)\right)\right\|}{\left\|\boldsymbol{v}(t)\right\|} dt<\infty.
\end{equation*}
\end{lemma}
\begin{proof}
By definition, 
\begin{align*}
   &\Vert(\boldsymbol{\beta}(t)\odot \partial^s \tilde{\mathcal{L}}\left(\boldsymbol{v}(t)\right))_{\perp}\Vert^2
   \\
   =&\Vert\boldsymbol{\beta}(t)\odot \partial^s \tilde{\mathcal{L}}\left(\boldsymbol{v}(t)\right)\Vert^2-\Vert(\boldsymbol{\beta}(t)\odot \partial^s \tilde{\mathcal{L}}\left(\boldsymbol{v}(t)\right))_{\backslash\backslash}\Vert^2
    \\
    =&\Vert\boldsymbol{\beta}^{\frac{1}{2}}(t)\odot \partial^s \tilde{\mathcal{L}}\left(\boldsymbol{v}(t)\right)\Vert^2+\left\Vert\left(\bbeta(t)-\boldsymbol{\beta}^{\frac{1}{2}}(t)\right)\odot \partial^s \tilde{\mathcal{L}}\left(\boldsymbol{v}(t)\right)\right\Vert^2+2\left\langle \boldsymbol{\beta}^{\frac{1}{2}}(t)\odot \partial^s \tilde{\mathcal{L}}\left(\boldsymbol{v}(t)\right),\left(\bbeta(t)-\boldsymbol{\beta}^{\frac{1}{2}}(t)\right)\odot \partial^s \tilde{\mathcal{L}}\left(\boldsymbol{v}(t)\right)\right\rangle
    \\
    &-\left\langle\boldsymbol{\beta}(t)\odot \partial^s \tilde{\mathcal{L}}\left(\boldsymbol{v}(t)\right), \hat{\bv}(t)\right\rangle^2
    \\
    =&\Vert\boldsymbol{\beta}^{\frac{1}{2}}(t)\odot \partial^s \tilde{\mathcal{L}}\left(\boldsymbol{v}(t)\right)\Vert^2+\left\Vert\left(\bbeta(t)-\boldsymbol{\beta}^{\frac{1}{2}}(t)\right)\odot \partial^s \tilde{\mathcal{L}}\left(\boldsymbol{v}(t)\right)\right\Vert^2+2\left\langle \boldsymbol{\beta}^{\frac{1}{2}}(t)\odot \partial^s \tilde{\mathcal{L}}\left(\boldsymbol{v}(t)\right),\left(\bbeta(t)-\boldsymbol{\beta}^{\frac{1}{2}}(t)\right)\odot \partial^s \tilde{\mathcal{L}}\left(\boldsymbol{v}(t)\right)\right\rangle
    \\
    -&\left\langle \partial^s \tilde{\mathcal{L}}\left(\boldsymbol{v}(t)\right), \hat{\bv}(t)\right\rangle^2
    -\left\langle \left(\boldsymbol{\beta}^{\frac{1}{2}}(t) 
    - \boldsymbol{1}\right)\odot\partial^s \tilde{\mathcal{L}}\left(\boldsymbol{v}(t)\right), \hat{\bv}(t)\right\rangle^2
    -2\left\langle \partial^s \tilde{\mathcal{L}}\left(\boldsymbol{v}(t)\right), \hat{\bv}(t)\right\rangle\left\langle \left(\boldsymbol{\beta}^{\frac{1}{2}}(t) 
    - \boldsymbol{1}\right)\odot\partial^s \tilde{\mathcal{L}}\left(\boldsymbol{v}(t)\right), \hat{\bv}(t)\right\rangle
    \\
    \overset{(*)}{\le} &\Vert\boldsymbol{\beta}^{\frac{1}{2}}(t)\odot \bar{\partial} \tilde{\mathcal{L}}\left(\boldsymbol{v}(t)\right)\Vert^2+\left\Vert\left(\bbeta(t)-\boldsymbol{\beta}^{\frac{1}{2}}(t)\right)\odot \partial^s \tilde{\mathcal{L}}\left(\boldsymbol{v}(t)\right)\right\Vert^2+2\left\langle \boldsymbol{\beta}^{\frac{1}{2}}(t)\odot \partial^s \tilde{\mathcal{L}}\left(\boldsymbol{v}(t)\right),\left(\bbeta(t)-\boldsymbol{\beta}^{\frac{1}{2}}(t)\right)\odot \partial^s \tilde{\mathcal{L}}\left(\boldsymbol{v}(t)\right)\right\rangle
    \\
    -&\left\langle \bar{\partial} \tilde{\mathcal{L}}\left(\boldsymbol{v}(t)\right), \hat{\bv}(t)\right\rangle^2
    -\left\langle \left(\boldsymbol{\beta}^{\frac{1}{2}}(t) 
    - \boldsymbol{1}\right)\odot\partial^s \tilde{\mathcal{L}}\left(\boldsymbol{v}(t)\right), \hat{\bv}(t)\right\rangle^2
    -2\left\langle \partial^s \tilde{\mathcal{L}}\left(\boldsymbol{v}(t)\right), \hat{\bv}(t)\right\rangle\left\langle \left(\boldsymbol{\beta}^{\frac{1}{2}}(t) 
    - \boldsymbol{1}\right)\odot\partial^s \tilde{\mathcal{L}}\left(\boldsymbol{v}(t)\right), \hat{\bv}(t)\right\rangle
    \end{align*}
    \begin{align*}
    \le &\Vert \bar{\partial} \tilde{\mathcal{L}}\left(\boldsymbol{v}(t)\right)\Vert^2+2\left\langle \bar{\partial} \tilde{\mathcal{L}}\left(\boldsymbol{v}(t)\right),\left(\bbeta(t)^{\frac{1}{2}}-\boldsymbol{1}\right)\odot \bar{\partial} \tilde{\mathcal{L}}\left(\boldsymbol{v}(t)\right)\right\rangle+\left\Vert\left(\boldsymbol{\beta}^{\frac{1}{2}}(t)-\boldsymbol{1}\right)\odot \bar{\partial} \tilde{\mathcal{L}}\left(\boldsymbol{v}(t)\right)\right\Vert^2
    \\
    +&\left\Vert\left(\bbeta(t)-\boldsymbol{\beta}^{\frac{1}{2}}(t)\right)\odot \partial^s \tilde{\mathcal{L}}\left(\boldsymbol{v}(t)\right)\right\Vert^2+2\left\langle \boldsymbol{\beta}^{\frac{1}{2}}(t)\odot \partial^s \tilde{\mathcal{L}}\left(\boldsymbol{v}(t)\right),\left(\bbeta(t)-\boldsymbol{\beta}^{\frac{1}{2}}(t)\right)\odot \partial^s \tilde{\mathcal{L}}\left(\boldsymbol{v}(t)\right)\right\rangle
    \\
    -&\left\langle \bar{\partial} \tilde{\mathcal{L}}\left(\boldsymbol{v}(t)\right), \hat{\bv}(t)\right\rangle^2
    -\left\langle \left(\boldsymbol{\beta}^{\frac{1}{2}}(t) 
    - \boldsymbol{1}\right)\odot\partial^s \tilde{\mathcal{L}}\left(\boldsymbol{v}(t)\right), \hat{\bv}(t)\right\rangle^2
    -2\left\langle \partial^s \tilde{\mathcal{L}}\left(\boldsymbol{v}(t)\right), \hat{\bv}(t)\right\rangle\left\langle \left(\boldsymbol{\beta}^{\frac{1}{2}}(t) 
    - \boldsymbol{1}\right)\odot\partial^s \tilde{\mathcal{L}}\left(\boldsymbol{v}(t)\right), \hat{\bv}(t)\right\rangle
    \\
    \overset{(**)}{\le } &\Vert \bar{\partial} \tilde{\mathcal{L}}\left(\boldsymbol{v}(t)\right)\Vert^2-\left\langle \bar{\partial} \tilde{\mathcal{L}}\left(\boldsymbol{v}(t)\right), \hat{\bv}(t)\right\rangle^2+\mathcal{O}\left(\frac{1}{t^{3}}\right).
\end{align*}
where Inequality $(*)$ is due to Lemma \ref{lem: optimal gradient} and the homogeneity of the neural network, and Inequality $(**)$ is due to $\Vert1-\boldsymbol{\beta}(t)\Vert_{\infty}=\mathcal{O}(\frac{1}{t(\log t)^{2-\frac{2}{L}}})$ by (eq. (\ref{eq:rate_of_beta})) and the convergent rate of $\bar{\partial} \tilde{\mathcal{L}}\left(\boldsymbol{v}(t)\right)$ and $\partial^s \tilde{\mathcal{L}}\left(\boldsymbol{v}(t)\right)$.
Therefore, 
\begin{align*}
    \int_{0}^{\infty}\frac{\mathrm{d} \zeta(t)}{\mathrm{d} t}\le \int_{0}^{\infty}\frac{\left\|\bar{\partial}_{\perp} \tilde{\mathcal{L}}\left(\boldsymbol{v}(t)\right)\right\|}{\left\|\boldsymbol{v}(t)\right\|} dt+\int_{0}^{\infty}\mathcal{O}\left(\frac{1}{t^{\frac{3}{2}}}\right) dt.
\end{align*}
The proof is completed.
\end{proof}

We then prove 
\begin{equation*}
     \int_{0}^{\infty}\frac{\left\|\bar{\partial}_{\perp} \tilde{\mathcal{L}}\left(\boldsymbol{v}(t)\right)\right\|}{\left\|\boldsymbol{v}(t)\right\|} dt<\infty.
\end{equation*}

\begin{theorem}
There exists $a,\gamma_0>0$  and a definable  desingularizing function $\Phi$ on $[0,a)$, such that, for large enough $t$, 
\begin{equation*}
    \frac{\left\|\bar{\partial}_{\perp} \tilde{\mathcal{L}}\left(\boldsymbol{v}(t)\right)\right\|}{\Vert \boldsymbol{v}(t)\Vert}\le -\frac{\mathrm{d} \Psi(\gamma_0-\bar{\gamma}(t))}{\mathrm{d}t}.
\end{equation*}
\end{theorem}

\begin{proof}
Since $\frac{d \bar{\gamma}(t)}{dt}\ge 0$, and both $\frac{\log \frac{1}{\tilde{\mathcal{L}(\boldsymbol{v}(t))}}}{\Vert \boldsymbol{v}(t)\Vert^L}$ and $e^{-\tilde{\mathcal{L}}(\boldsymbol{v}(t))}$ are upper bounded, $\bar{\gamma}(t)$ converges to a limit non-decreasingly. Define $\gamma_0=\lim_{t\rightarrow\infty} \bar{\gamma}(t)$. If $\bar{\gamma}(t)=\gamma_0$ for a finite time $t^0$, then  $\frac{\mathrm{d}\bar{\gamma}(t)}{dt}=0$ for any $t\ge t^0$, which further leads to $\Vert (\bar{\partial} \tilde{\mathcal{L}}(\boldsymbol{v})(t) )_{\perp}\Vert=0$, and the proof is then completed by letting $\Psi(x)=x$. Therefore, we only consider the case where $\bar{\gamma}(t)<\gamma_0$ for any finite time $t$. For any large enough $t$, we further divide the proof into two cases.

\textbf{Case I. $ \left\|\bar{\partial}_{\perp} \bar{\gamma}(\boldsymbol{v}(t))\right\| \geq \|\boldsymbol{v}(t)\|^{\frac{L}{4}}\left\|\bar{\partial}_{\backslash\backslash} \bar{\gamma}(\boldsymbol{v}(t))\right\|$.}

Applying Lemma \ref{lem:construction_of_psi_conditional} to $\gamma_0-\bar{\gamma}(\boldsymbol{v})|_{\Vert\boldsymbol{v}\Vert>1}$, there exists an $a_1>0$ and a definable desingularizing function $\Psi_1$, such that if $\Vert\boldsymbol{v}\Vert> 1$, $\bar{\gamma}(\boldsymbol{v})>\gamma_0-a_1$, and 
\begin{equation*}
    \left\|\bar{\partial}_{\perp} \bar{\gamma}(\boldsymbol{v})\right\| \geq \|\boldsymbol{v}\|^{\frac{L}{4}}\left\|\bar{\partial}_{\backslash\backslash} \bar{\gamma}(\boldsymbol{v})\right\|,
\end{equation*}
then
\begin{equation*}
    \Psi'_1(\gamma_0-\bar{\gamma}(\boldsymbol{v}))\Vert \boldsymbol{v}\Vert \Vert \bar{\partial} \bar{\gamma}(\boldsymbol{v}(t))\Vert\ge 1.
\end{equation*}

Since $\lim_{t\rightarrow\infty} \bar{\gamma}(t)=\gamma_0$, and $\lim_{t\rightarrow\infty} \Vert \boldsymbol{v}(t)\Vert=\infty$, there exists a large enough time $T_1$, such that, for every $t\ge T_1$, $\Vert\boldsymbol{v}(t)\Vert> 1$, and $\bar{\gamma}(t)>\gamma_0-a_1$.

Therefore, for any $t\ge T_1$ which satisfies $ \left\|\bar{\partial}_{\perp} \bar{\gamma}(\boldsymbol{v}(t))\right\| \geq \|\boldsymbol{v}(t)\|^{\frac{L}{4}}\left\|\bar{\partial}_{\backslash\backslash} \bar{\gamma}(\boldsymbol{v}(t))\right\|$, we have 
\begin{equation*}
    \left\|\bar{\partial}_{\perp} \bar{\gamma}(\boldsymbol{v}(t))\right\| \geq\left\|\bar{\partial}_{\backslash\backslash} \bar{\gamma}(\boldsymbol{v}(t))\right\|,
\end{equation*}
which further indicates
\begin{equation*}
   \left\|\bar{\partial}_{\perp} \bar{\gamma}(\boldsymbol{v}(t))\right\|\ge \frac{1}{2}  \left\|\bar{\partial} \bar{\gamma}(\boldsymbol{v}(t))\right\|.
\end{equation*}

Furthermore, by Lemma \ref{lem:lower_bound_bar_gamma}, 
\begin{align*}
     \frac{\mathrm{d} \bar{\gamma}(t)}{dt}\ge &\frac{1}{2} \left\|\bar{\partial}_{\perp} \bar{\gamma}(\boldsymbol{v}(t))\right\|\left\|\bar{\partial}_{\perp} \tilde{\mathcal{L}}(\boldsymbol{v}(t))\right\|
     \\
     \ge&\frac{1}{4}\Vert \boldsymbol{v}(t)\Vert \left\|\bar{\partial} \bar{\gamma}(\boldsymbol{v}(t))\right\|\frac{\left\|\bar{\partial}_{\perp} \tilde{\mathcal{L}}(\boldsymbol{v}(t))\right\|}{\Vert \boldsymbol{v}(t)\Vert}
     \\
     \ge& \frac{1}{4\Psi'_1(\gamma_0-\bar{\gamma}(t))}\frac{\left\|\bar{\partial}_{\perp} \tilde{\mathcal{L}}(\boldsymbol{v}(t))\right\|}{\Vert \boldsymbol{v}(t)\Vert}.
\end{align*}

\textbf{Case II. $ \left\|\bar{\partial}_{\perp} \bar{\gamma}(\boldsymbol{v}(t))\right\| < \|\boldsymbol{v}(t)\|^{\frac{L}{4}}\left\|\bar{\partial}_{\backslash\backslash} \bar{\gamma}(\boldsymbol{v}(t))\right\|$.}

Applying Lemma \ref{lem:construction_of_psi_no_condition} to $\gamma_0-\bar{\gamma}(\boldsymbol{v})|_{\Vert \boldsymbol{v}\Vert>1}$, we have that there exists an $a_2>0$ and  a desingularizing function on $[0,a_2)$, such that if $\boldsymbol{v}>1$, and $\bar{\gamma}(\boldsymbol{v})>\gamma_0-a_2$, then 
\begin{equation*}
    \max\left\{1,\frac{4}{L}\right\} \Psi_2'(\gamma_0-\bar{\gamma}(\boldsymbol{v}))\Vert \boldsymbol{v}\Vert^{\frac{L}{2}+1}\Vert\bar{\partial} \bar{\gamma}(\boldsymbol{v}) \Vert\ge 1.
\end{equation*}

Similar to Case I., there exists a large enough time $T_2$ and constants $C_1$, such that, for every $t\ge T_1$, $\Vert \boldsymbol{v}(t)\Vert>1$, $\bar{\gamma}(t)>\gamma_0-a_2$, $C_1\Vert \boldsymbol{v}(t)\Vert^L\le g\left(\log \frac{1}{\tilde{\mathcal{L}}(\boldsymbol{v}(t))}\right)$, and $\tilde{\mathcal{L}}(\boldsymbol{v}(t))\le \Vert \boldsymbol{v}(t)  \Vert^{-2L}$.

Therefore, 
\begin{equation}
\label{eq:paralel_g}
    \left\Vert \bar{\partial}_{\backslash\backslash} g\left( \log \frac{1}{\tilde{\mathcal{L}}(\boldsymbol{v}(t))}\right)\right\Vert = g'\left( \log \frac{1}{\tilde{\mathcal{L}}(\boldsymbol{v}(t))}\right)\frac{1}{\tilde{\mathcal{L}}(\boldsymbol{v}(t))}\frac{L\nu(t)}{\Vert \boldsymbol{v}(t)\Vert}\ge \frac{Lg\left(\log \frac{1}{\tilde{\mathcal{L}}(\boldsymbol{v}(t))}\right)}{\Vert \boldsymbol{v}(t)\Vert}\ge LC_1 \Vert\boldsymbol{v}(t)\Vert^{L-1},
\end{equation}
and 
\begin{align}
    \nonumber
    &\Vert\bar{\partial}_{\backslash\backslash} \bar{\gamma}(\boldsymbol{v}(t)) \Vert
    \\
    \nonumber
    =&\left\Vert L\frac{g\left(\log \frac{1}{\tilde{\mathcal{L}}(\boldsymbol{v}(t))}\right)}{\Vert \boldsymbol{v}(t)\Vert^{L+1}}\hat{\boldsymbol{v}}(t)+g'\left(\log \frac{1}{\tilde{\mathcal{L}}(\boldsymbol{v}(t))}\right)\frac{\langle\bar{\partial} \tilde{\mathcal{L}}(\boldsymbol{v}(t)), \hat{\boldsymbol{v}}(t)\rangle}{\tilde{\mathcal{L}}\Vert\boldsymbol{v}(t)\Vert^L}\hat{\boldsymbol{v}}(t)+e^{-\tilde{\mathcal{L}}(\boldsymbol{v}(t))}\langle \bar{\partial} \tilde{\mathcal{L}}(\boldsymbol{v}(t)),\hat{\boldsymbol{v}}(t)\rangle  \hat{\boldsymbol{v}}(t)\right\Vert
    \\
    \nonumber
    = & -L\frac{g\left(\log \frac{1}{\tilde{\mathcal{L}}(\boldsymbol{v}(t))}\right)}{\Vert \boldsymbol{v}(t)\Vert^{L+1}}-g'\left(\log \frac{1}{\tilde{\mathcal{L}}(\boldsymbol{v}(t))}\right)\frac{\langle\bar{\partial} \tilde{\mathcal{L}}(\boldsymbol{v}(t)), \hat{\boldsymbol{v}}(t)\rangle}{\tilde{\mathcal{L}}(\boldsymbol{v}(t))\Vert\boldsymbol{v}(t)\Vert^L}-e^{-\tilde{\mathcal{L}}(\boldsymbol{v}(t))}\langle \bar{\partial} \tilde{\mathcal{L}}(\boldsymbol{v}(t)),\hat{\boldsymbol{v}}(t)\rangle  
    \\
    \nonumber
    \overset{(*)}{\le } &\frac{L g\left(\log \frac{1}{Ne^{-f(0)}}\right)}{\Vert\boldsymbol{v}(t)\Vert^{L+1}}+\Vert \bar{\partial} \tilde{\mathcal{L}}(\boldsymbol{v}(t)) \Vert
    \\
    \nonumber
    \le &
    \frac{L g\left(\log \frac{1}{Ne^{-f(0)}}\right)} {\Vert\boldsymbol{v}(t)\Vert^{L+1}}+B_1\tilde{\mathcal{L}}(\boldsymbol{v}(t))\Vert \boldsymbol{v}(t)  \Vert^{L-1}
    \\
    \label{eq:paralel_gamma}
    \le &\frac{B_1+L g\left(\log \frac{1}{Ne^{-f(0)}}\right)} {\Vert\boldsymbol{v}(t)\Vert^{L+1}},
\end{align}
where inequality $(*)$ is due to 
\begin{align*}
    &-Lg\left(\log \frac{1}{\tilde{\mathcal{L}}(\boldsymbol{v}(t))}\right)-g'\left(\log \frac{1}{\tilde{\mathcal{L}}(\boldsymbol{v}(t))}\right)\frac{\langle\bar{\partial} \tilde{\mathcal{L}}(\boldsymbol{v}(t)), \boldsymbol{v}(t)\rangle}{\tilde{\mathcal{L}}(\boldsymbol{v}(t))}
    \\
    =&-Lg\left(\log \frac{1}{\tilde{\mathcal{L}}(\boldsymbol{v}(t))}\right)+g'\left(\log \frac{1}{\tilde{\mathcal{L}}(\boldsymbol{v}(t))}\right)\frac{L\nu(t)}{\tilde{\mathcal{L}}(\boldsymbol{v}(t))}
    \\
    =&-Lg\left(\log \frac{1}{\tilde{\mathcal{L}}(\boldsymbol{v}(t))}\right)+g'\left(\log \frac{1}{\tilde{\mathcal{L}}(\boldsymbol{v}(t))}\right)\frac{L}{\tilde{\mathcal{L}}(\boldsymbol{v}(t))}\left(\sum_{i=1}^N e^{-\tilde{q}_i(\boldsymbol{v}(t))}\tilde{q}_i(\boldsymbol{v}(t))f'\left(\tilde{q}_i(\boldsymbol{v}(t))\right)\right)
    \\
    =&-Lg\left(\log \frac{1}{\tilde{\mathcal{L}}(\boldsymbol{v}(t))}\right)+L\left\langle \nabla_{\tilde{q}} g\left(\log \frac{1}{\tilde{\mathcal{L}}(\boldsymbol{v}(t))}\right),\tilde{q} \right\rangle
    \\
    \overset{(*)}{\le}& -L g\left(\log \frac{1}{Ne^{-f(0)}}\right),
\end{align*}
where $\tilde{q}=(\tilde{q}_i)_{i=1}^N$ inequality $(*)$ comes from Jensen Inequality and $g\left(\log \frac{1}{\tilde{\mathcal{L}}(\boldsymbol{v}(t))}\right)$ is concave with respect to $\tilde{q}$.

Combining eqs. (\ref{eq:paralel_g}) and (\ref{eq:paralel_gamma}), we have
\begin{equation}
\label{eq:brige_parallel}
     \left\Vert  \bar{\partial}_{\backslash\backslash}g\left( \log \frac{1}{\tilde{\mathcal{L}}(\boldsymbol{v}(t))}\right)\right\Vert\ge \frac{LC_1}{B_1+L g\left(\log \frac{1}{Ne^{-f(0)}}\right)}\Vert \boldsymbol{v}(t)\Vert^{2L} \Vert\bar{\partial}_{\backslash\backslash} \bar{\gamma}(\boldsymbol{v}(t)) \Vert.
\end{equation}
On the other hand, 
\begin{equation*}
     \left\Vert \bar{\partial}_{\perp}g\left( \log \frac{1}{\tilde{\mathcal{L}}(\boldsymbol{v}(t))}\right)\right\Vert=\frac{g'\left( \log \frac{1}{\tilde{\mathcal{L}}(\boldsymbol{v}(t))}\right)}{\tilde{\mathcal{L}}(\boldsymbol{v}(t))}\Vert\bar{\partial} \tilde{\mathcal{L}}(\boldsymbol{v}(t))-\langle \bar{\partial} \tilde{\mathcal{L}}(\boldsymbol{v}(t)),\hat{\boldsymbol{v}}(t)\rangle  \hat{\boldsymbol{v}}(t)\Vert,
\end{equation*}
while 
\begin{align}
\nonumber
     &\Vert\bar{\partial}_{\perp} \bar{\gamma}(\boldsymbol{v}(t)) \Vert
     \\
     \nonumber
     =&\left\Vert e^{-\tilde{\mathcal{L}}(\boldsymbol{v}(t))}(\bar{\partial} \tilde{\mathcal{L}}(\boldsymbol{v}(t))-\langle \bar{\partial} \tilde{\mathcal{L}}(\boldsymbol{v}(t)),\hat{\boldsymbol{v}}(t)\rangle  \hat{\boldsymbol{v}}(t))+ g'\left( \log \frac{1}{\tilde{\mathcal{L}}(\boldsymbol{v}(t))}\right)\frac{\bar{\partial} \tilde{\mathcal{L}}(\boldsymbol{v}(t))-\langle\bar{\partial} \tilde{\mathcal{L}}(\boldsymbol{v}(t)), \hat{\boldsymbol{v}}(t)\rangle\hat{\boldsymbol{v}}(t)}{\tilde{\mathcal{L}}(\boldsymbol{v}(t))\Vert\boldsymbol{v}(t)\Vert^L}\right\Vert
     \\
     \nonumber
     =&\left(e^{-\tilde{\mathcal{L}}(\boldsymbol{v}(t))}+\frac{g'\left( \log \frac{1}{\tilde{\mathcal{L}}(\boldsymbol{v}(t))}\right)}{\tilde{\mathcal{L}}(\boldsymbol{v}(t))\Vert\boldsymbol{v}(t)\Vert^L}\right)\left\Vert \bar{\partial} \tilde{\mathcal{L}}(\boldsymbol{v}(t))-\langle \bar{\partial} \tilde{\mathcal{L}}(\boldsymbol{v}(t)),\hat{\boldsymbol{v}}(t)\rangle  \hat{\boldsymbol{v}}(t)\right\Vert
     \\
     \label{eq:brige_perp}
     \ge &\frac{1}{\Vert \boldsymbol{v}(t)\Vert^L}\left\Vert \bar{\partial}_{\perp} g\left( \log \frac{1}{\tilde{\mathcal{L}}(\boldsymbol{v}(t))}\right)\right\Vert.
\end{align}
Since 
\begin{equation*}
    \left\|\bar{\partial}_{\perp} \bar{\gamma}(\boldsymbol{v}(t))\right\| < \|\boldsymbol{v}\|^{\frac{L}{4}}\left\|\bar{\partial}_{\backslash\backslash} \bar{\gamma}(\boldsymbol{v}(t))\right\|,
\end{equation*}
combining eqs. (\ref{eq:brige_parallel}) and (\ref{eq:brige_perp}), we have that 
\begin{align*}
     \left\Vert \bar{\partial}_{\backslash\backslash} g\left( \log \frac{1}{\tilde{\mathcal{L}}(\boldsymbol{v}(t))}\right)\right\Vert\ge& \frac{LC_1}{B_1+L g\left(\log \frac{1}{Ne^{-f(0)}}\right)}\Vert \boldsymbol{v}(t)\Vert^{2L} \Vert\bar{\partial}_{\backslash\backslash} \bar{\gamma}(\boldsymbol{v}(t)) \Vert
     \\
     \ge& \frac{LC_1}{B_1+L g\left(\log \frac{1}{Ne^{-f(0)}}\right)}\Vert \boldsymbol{v}(t)\Vert^{\frac{7}{4}L} \Vert\bar{\partial}_{\perp} \bar{\gamma}(\boldsymbol{v}(t)) \Vert
     \\
     \ge &\frac{LC_1}{B_1+L g\left(\log \frac{1}{Ne^{-f(0)}}\right)}\Vert \boldsymbol{v}(t)\Vert^{\frac{3}{4}L} \left\Vert\bar{\partial}_{\perp} g\left( \log \frac{1}{\tilde{\mathcal{L}}(\boldsymbol{v}(t))}\right)\right\Vert.
\end{align*}
Since $\bar{\partial} \tilde{\mathcal{L}}(\boldsymbol{v}(t))$ is parallel to  $\bar{\partial}g\left( \log \frac{1}{\tilde{\mathcal{L}}(\boldsymbol{v}(t))}\right) $, we also have
\begin{equation*}
     \left\Vert \bar{\partial}_{\backslash\backslash}\tilde{\mathcal{L}}(\boldsymbol{v}(t))\right\Vert\ge\frac{LC_1}{B_1+L g\left(\log \frac{1}{Ne^{-f(0)}}\right)}\Vert \boldsymbol{v}(t)\Vert^{\frac{3}{4}L} \left\Vert\bar{\partial}_{\perp} \tilde{\mathcal{L}} (\boldsymbol{v}(t))\right\Vert.
\end{equation*}

Furthermore, 
\begin{equation*}
    \left\|\bar{\partial} \bar{\gamma}(\boldsymbol{v}(t))\right\| = \left\|\bar{\partial}_{\perp} \bar{\gamma}(\boldsymbol{v}(t))\right\|+ \left\|\bar{\partial}_{\backslash\backslash} \bar{\gamma}(\boldsymbol{v}(t))\right\|\le\|\boldsymbol{v}\|^{\frac{L}{4}}\left\|\bar{\partial}_{\backslash\backslash} \bar{\gamma}(\boldsymbol{v}(t))\right\|.
\end{equation*}

Thus, by Lemma \ref{lem:lower_bound_bar_gamma}, 
\begin{align*}
    \frac{\mathrm{d} \bar{\gamma}(t)}{\mathrm{d} t}\ge& \frac{1}{2}\left\|\bar{\partial}_{\backslash\backslash} \bar{\gamma}\left(\boldsymbol{v}(t)\right)\right\|\left\|\bar{\partial}_{\backslash\backslash} \tilde{\mathcal{L}}\left(\boldsymbol{v}(t)\right)\right\|
    \\
    \ge &\frac{LC_1}{2(B_1+L g\left(\log \frac{1}{Ne^{-f(0)}}\right))}\Vert \boldsymbol{v}(t)\Vert^{\frac{1}{2}L}\left\|\bar{\partial}\bar{\gamma}\left(\boldsymbol{v}(t)\right)\right\| \left\Vert\bar{\partial}_{\perp} \tilde{\mathcal{L}}(\boldsymbol{v}(t)) \right\Vert
    \\
    =&\frac{LC_1}{2(B_1+L g\left(\log \frac{1}{Ne^{-f(0)}}\right))}\Vert \boldsymbol{v}(t)\Vert^{\frac{1}{2}L+1}\left\|\bar{\partial}\bar{\gamma}\left(\boldsymbol{v}(t)\right)\right\|\frac{\left\|\bar{\partial}_{\perp} \tilde{\mathcal{L}}\left(\boldsymbol{v}(t)\right)\right\|}{\left\|\boldsymbol{v}(t)\right\|}
    \\
    \ge&\frac{LC_1}{2(B_1+L g\left(\log \frac{1}{Ne^{-f(0)}}\right)) \max\left\{1,\frac{4}{L}\right\} \Psi_2'(\gamma_0-\bar{\gamma}(\boldsymbol{v}(t)))}\frac{\left\|\bar{\partial}_{\perp} \tilde{\mathcal{L}}\left(\boldsymbol{v}(t)\right)\right\|}{\left\|\boldsymbol{v}(t)\right\|}.
\end{align*}

Concluding \textbf{Case I.} and \textbf{Case II.}, for any $t\ge\max\{ T_1,T_2\}$, and $\Psi(x)=\max\{4 \Psi_1(x),$ $\frac{2(B_1+L g\left(\log \frac{1}{Ne^{-f(0)}}\right)) \max\left\{1,\frac{4}{L}\right\} }{LC_1}\Psi_2(x)\}$, we have that
\begin{equation*}
    \Psi'(\gamma_0-\bar{\gamma}(\boldsymbol{v}(t))) \frac{\mathrm{d} \bar{\gamma}(t)}{\mathrm{d} t}\ge \frac{\left\|\bar{\partial}_{\perp} \tilde{\mathcal{L}}\left(\boldsymbol{v}(t)\right)\right\|}{\Vert \boldsymbol{v}(t)\Vert}.
\end{equation*}

The proof is completed.
\end{proof}

%% file: max-margin-icml.bbl
\begin{thebibliography}{43}
\providecommand{\natexlab}[1]{#1}
\providecommand{\url}[1]{\texttt{#1}}
\expandafter\ifx\csname urlstyle\endcsname\relax
  \providecommand{\doi}[1]{doi: #1}\else
  \providecommand{\doi}{doi: \begingroup \urlstyle{rm}\Url}\fi

\bibitem[Arora et~al.(2019)Arora, Du, Hu, Li, Salakhutdinov, and
  Wang]{arora2019exact}
Arora, S., Du, S.~S., Hu, W., Li, Z., Salakhutdinov, R.~R., and Wang, R.
\newblock On exact computation with an infinitely wide neural net.
\newblock In \emph{Advances in Neural Information Processing Systems}, pp.\
  8141--8150, 2019.

\bibitem[Bartlett \& Shawe-Taylor(1999)Bartlett and
  Shawe-Taylor]{bartlett1999generalization}
Bartlett, P. and Shawe-Taylor, J.
\newblock Generalization performance of support vector machines and other
  pattern classifiers.
\newblock \emph{Advances in Kernel methods—support vector learning}, pp.\
  43--54, 1999.

\bibitem[Bartlett et~al.(2017)Bartlett, Foster, and
  Telgarsky]{bartlett2017spectrally}
Bartlett, P.~L., Foster, D.~J., and Telgarsky, M.
\newblock Spectrally-normalized margin bounds for neural networks.
\newblock In \emph{Proceedings of the 31st International Conference on Neural
  Information Processing Systems}, pp.\  6241--6250, 2017.

\bibitem[Brutzkus et~al.(2018)Brutzkus, Globerson, Malach, and
  Shalev-Shwartz]{brutzkus2017sgd}
Brutzkus, A., Globerson, A., Malach, E., and Shalev-Shwartz, S.
\newblock Sgd learns over-parameterized networks that provably generalize on
  linearly separable data.
\newblock In \emph{International Conference on Learning Representations}, 2018.

\bibitem[Chen et~al.(2018)Chen, Zhou, Tang, Yang, Cao, and Gu]{chen2018closing}
Chen, J., Zhou, D., Tang, Y., Yang, Z., Cao, Y., and Gu, Q.
\newblock Closing the generalization gap of adaptive gradient methods in
  training deep neural networks.
\newblock \emph{arXiv preprint arXiv:1806.06763}, 2018.

\bibitem[Choromanska et~al.(2015)Choromanska, Henaff, Mathieu, Arous, and
  LeCun]{choromanska2015loss}
Choromanska, A., Henaff, M., Mathieu, M., Arous, G.~B., and LeCun, Y.
\newblock The loss surfaces of multilayer networks.
\newblock In \emph{Artificial intelligence and statistics}, pp.\  192--204,
  2015.

\bibitem[Clarke(1975)]{clarke1975generalized}
Clarke, F.~H.
\newblock Generalized gradients and applications.
\newblock \emph{Transactions of the American Mathematical Society},
  205:\penalty0 247--262, 1975.

\bibitem[Davis et~al.(2020)Davis, Drusvyatskiy, Kakade, and
  Lee]{davis2020stochastic}
Davis, D., Drusvyatskiy, D., Kakade, S., and Lee, J.~D.
\newblock Stochastic subgradient method converges on tame functions.
\newblock \emph{Foundations of computational mathematics}, 20\penalty0
  (1):\penalty0 119--154, 2020.

\bibitem[Deng et~al.(2013)Deng, Hinton, and Kingsbury]{deng2013new}
Deng, L., Hinton, G., and Kingsbury, B.
\newblock New types of deep neural network learning for speech recognition and
  related applications: An overview.
\newblock In \emph{2013 IEEE international conference on acoustics, speech and
  signal processing}, pp.\  8599--8603. IEEE, 2013.

\bibitem[Duchi et~al.(2011)Duchi, Hazan, and Singer]{duchi2011adaptive}
Duchi, J., Hazan, E., and Singer, Y.
\newblock Adaptive subgradient methods for online learning and stochastic
  optimization.
\newblock \emph{Journal of machine learning research}, 12\penalty0 (7), 2011.

\bibitem[Gunasekar et~al.(2018{\natexlab{a}})Gunasekar, Lee, Soudry, and
  Srebro]{gunasekar2018characterizing}
Gunasekar, S., Lee, J., Soudry, D., and Srebro, N.
\newblock Characterizing implicit bias in terms of optimization geometry.
\newblock In \emph{ICML}, 2018{\natexlab{a}}.

\bibitem[Gunasekar et~al.(2018{\natexlab{b}})Gunasekar, Lee, Soudry, and
  Srebro]{gunasekar2018implicit}
Gunasekar, S., Lee, J.~D., Soudry, D., and Srebro, N.
\newblock Implicit bias of gradient descent on linear convolutional networks.
\newblock In \emph{Advances in Neural Information Processing Systems}, pp.\
  9461--9471, 2018{\natexlab{b}}.

\bibitem[Hinton et~al.(2012)Hinton, Srivastava, and Swersky]{hinton2012neural}
Hinton, G., Srivastava, N., and Swersky, K.
\newblock Neural networks for machine learning lecture 6a overview of
  mini--batch gradient descent.
\newblock 2012.

\bibitem[Ji \& Telgarsky(2018)Ji and Telgarsky]{ji2018gradient}
Ji, Z. and Telgarsky, M.
\newblock Gradient descent aligns the layers of deep linear networks.
\newblock In \emph{International Conference on Learning Representations}, 2018.

\bibitem[Ji \& Telgarsky(2019)Ji and Telgarsky]{ji2019implicit}
Ji, Z. and Telgarsky, M.
\newblock The implicit bias of gradient descent on nonseparable data.
\newblock In \emph{Conference on Learning Theory}, pp.\  1772--1798, 2019.

\bibitem[Ji \& Telgarsky(2020)Ji and Telgarsky]{ji2020directional}
Ji, Z. and Telgarsky, M.
\newblock Directional convergence and alignment in deep learning.
\newblock \emph{Advances in Neural Information Processing Systems}, 33, 2020.

\bibitem[Jiang et~al.(2019)Jiang, Neyshabur, Mobahi, Krishnan, and
  Bengio]{jiang2019fantastic}
Jiang, Y., Neyshabur, B., Mobahi, H., Krishnan, D., and Bengio, S.
\newblock Fantastic generalization measures and where to find them.
\newblock In \emph{International Conference on Learning Representations}, 2019.

\bibitem[Keskar \& Socher(2017)Keskar and Socher]{keskar2017improving}
Keskar, N.~S. and Socher, R.
\newblock Improving generalization performance by switching from adam to sgd.
\newblock \emph{arXiv preprint arXiv:1712.07628}, 2017.

\bibitem[Kingma \& Ba(2015)Kingma and Ba]{kingma2014adam}
Kingma, D.~P. and Ba, J.
\newblock Adam: A method for stochastic optimization.
\newblock In \emph{ICLR}, 2015.

\bibitem[Kurdyka(1998)]{kurdyka1998gradients}
Kurdyka, K.
\newblock On gradients of functions definable in o-minimal structures.
\newblock In \emph{Annales de l'institut Fourier}, volume~48, pp.\  769--783,
  1998.

\bibitem[LeCun(1998)]{lecun1998mnist}
LeCun, Y.
\newblock The mnist database of handwritten digits.
\newblock \emph{http://yann. lecun. com/exdb/mnist/}, 1998.

\bibitem[Li et~al.(2019)Li, Fang, Xu, and Zhao]{li2019implicit}
Li, Y., Fang, E.~X., Xu, H., and Zhao, T.
\newblock Implicit bias of gradient descent based adversarial training on
  separable data.
\newblock In \emph{International Conference on Learning Representations}, 2019.

\bibitem[Loshchilov \& Hutter(2018)Loshchilov and
  Hutter]{loshchilov2017decoupled}
Loshchilov, I. and Hutter, F.
\newblock Decoupled weight decay regularization.
\newblock In \emph{International Conference on Learning Representations}, 2018.

\bibitem[Luo et~al.(2018)Luo, Xiong, Liu, and Sun]{luo2019adaptive}
Luo, L., Xiong, Y., Liu, Y., and Sun, X.
\newblock Adaptive gradient methods with dynamic bound of learning rate.
\newblock In \emph{International Conference on Learning Representations}, 2018.

\bibitem[Lyu \& Li(2019)Lyu and Li]{lyu2019gradient}
Lyu, K. and Li, J.
\newblock Gradient descent maximizes the margin of homogeneous neural networks.
\newblock In \emph{International Conference on Learning Representations}, 2019.

\bibitem[Madry et~al.(2018)Madry, Makelov, Schmidt, Tsipras, and
  Vladu]{madry2017towards}
Madry, A., Makelov, A., Schmidt, L., Tsipras, D., and Vladu, A.
\newblock Towards deep learning models resistant to adversarial attacks.
\newblock In \emph{International Conference on Learning Representations}, 2018.

\bibitem[Nacson et~al.(2019{\natexlab{a}})Nacson, Gunasekar, Lee, Srebro, and
  Soudry]{nacson2019lexicographic}
Nacson, M.~S., Gunasekar, S., Lee, J., Srebro, N., and Soudry, D.
\newblock Lexicographic and depth-sensitive margins in homogeneous and
  non-homogeneous deep models.
\newblock In \emph{International Conference on Machine Learning}, pp.\
  4683--4692. PMLR, 2019{\natexlab{a}}.

\bibitem[Nacson et~al.(2019{\natexlab{b}})Nacson, Lee, Gunasekar, Savarese,
  Srebro, and Soudry]{nacson2019convergence}
Nacson, M.~S., Lee, J., Gunasekar, S., Savarese, P. H.~P., Srebro, N., and
  Soudry, D.
\newblock Convergence of gradient descent on separable data.
\newblock In \emph{The 22nd International Conference on Artificial Intelligence
  and Statistics}, pp.\  3420--3428. PMLR, 2019{\natexlab{b}}.

\bibitem[Nacson et~al.(2019{\natexlab{c}})Nacson, Srebro, and
  Soudry]{nacson2019stochastic}
Nacson, M.~S., Srebro, N., and Soudry, D.
\newblock Stochastic gradient descent on separable data: Exact convergence with
  a fixed learning rate.
\newblock In \emph{The 22nd International Conference on Artificial Intelligence
  and Statistics}, pp.\  3051--3059. PMLR, 2019{\natexlab{c}}.

\bibitem[Neyshabur et~al.(2015)Neyshabur, Salakhutdinov, and
  Srebro]{neyshabur2015path}
Neyshabur, B., Salakhutdinov, R.~R., and Srebro, N.
\newblock Path-sgd: Path-normalized optimization in deep neural networks.
\newblock In \emph{Advances in Neural Information Processing Systems}, pp.\
  2422--2430, 2015.

\bibitem[Neyshabur et~al.(2018)Neyshabur, Bhojanapalli, and
  Srebro]{neyshabur2018pac}
Neyshabur, B., Bhojanapalli, S., and Srebro, N.
\newblock A pac-bayesian approach to spectrally-normalized margin bounds for
  neural networks.
\newblock In \emph{International Conference on Learning Representations}, 2018.

\bibitem[Qian \& Qian(2019)Qian and Qian]{qian2019implicit}
Qian, Q. and Qian, X.
\newblock The implicit bias of adagrad on separable data.
\newblock In \emph{Advances in Neural Information Processing Systems}, pp.\
  7761--7769, 2019.

\bibitem[Reddi et~al.(2018)Reddi, Zaheer, Sachan, Kale, and
  Kumar]{reddi2018adaptive}
Reddi, S., Zaheer, M., Sachan, D., Kale, S., and Kumar, S.
\newblock Adaptive methods for nonconvex optimization.
\newblock In \emph{Proceeding of 32nd Conference on Neural Information
  Processing Systems (NIPS 2018)}, 2018.

\bibitem[Ruder(2016)]{ruder2016overview}
Ruder, S.
\newblock An overview of gradient descent optimization algorithms.
\newblock \emph{arXiv preprint arXiv:1609.04747}, 2016.

\bibitem[Soudry et~al.(2018)Soudry, Hoffer, Nacson, Gunasekar, and
  Srebro]{soudry2018implicit}
Soudry, D., Hoffer, E., Nacson, M.~S., Gunasekar, S., and Srebro, N.
\newblock The implicit bias of gradient descent on separable data.
\newblock \emph{The Journal of Machine Learning Research}, 19\penalty0
  (1):\penalty0 2822--2878, 2018.

\bibitem[Stein \& Shakarchi(2009)Stein and Shakarchi]{stein2009real}
Stein, E.~M. and Shakarchi, R.
\newblock \emph{Real analysis: measure theory, integration, and Hilbert
  spaces}.
\newblock Princeton University Press, 2009.

\bibitem[Voulodimos et~al.(2018)Voulodimos, Doulamis, Doulamis, and
  Protopapadakis]{voulodimos2018deep}
Voulodimos, A., Doulamis, N., Doulamis, A., and Protopapadakis, E.
\newblock Deep learning for computer vision: A brief review.
\newblock \emph{Computational intelligence and neuroscience}, 2018, 2018.

\bibitem[Wilson et~al.(2017)Wilson, Roelofs, Stern, Srebro, and
  Recht]{wilson2017marginal}
Wilson, A.~C., Roelofs, R., Stern, M., Srebro, N., and Recht, B.
\newblock The marginal value of adaptive gradient methods in machine learning.
\newblock In \emph{Advances in neural information processing systems}, pp.\
  4148--4158, 2017.

\bibitem[Witten \& Frank(2005)Witten and Frank]{witten2016data}
Witten, I.~H. and Frank, E.
\newblock \emph{Data Mining: Practical Machine Learning Tools and Techniques,
  (Morgan Kaufmann Series in Data Management Systems)}.
\newblock Morgan Kaufmann Publishers Inc., 2005.

\bibitem[Xu et~al.(2018)Xu, Zhou, Ji, and Liang]{xu2018will}
Xu, T., Zhou, Y., Ji, K., and Liang, Y.
\newblock When will gradient methods converge to max-margin classifier under
  relu models?
\newblock \emph{arXiv preprint arXiv:1806.04339}, 2018.

\bibitem[Young et~al.(2018)Young, Hazarika, Poria, and
  Cambria]{young2018recent}
Young, T., Hazarika, D., Poria, S., and Cambria, E.
\newblock Recent trends in deep learning based natural language processing.
\newblock \emph{ieee Computational intelligenCe magazine}, 13\penalty0
  (3):\penalty0 55--75, 2018.

\bibitem[Zhou et~al.(2020)Zhou, Feng, Ma, Xiong, Hoi, et~al.]{zhou2020towards}
Zhou, P., Feng, J., Ma, C., Xiong, C., Hoi, S. C.~H., et~al.
\newblock Towards theoretically understanding why sgd generalizes better than
  adam in deep learning.
\newblock \emph{Advances in Neural Information Processing Systems}, 33, 2020.

\bibitem[Zhuang et~al.(2020)Zhuang, Tang, Ding, Tatikonda, Dvornek,
  Papademetris, and Duncan]{zhuang2020adabelief}
Zhuang, J., Tang, T., Ding, Y., Tatikonda, S.~C., Dvornek, N., Papademetris,
  X., and Duncan, J.
\newblock Adabelief optimizer: Adapting stepsizes by the belief in observed
  gradients.
\newblock \emph{Advances in Neural Information Processing Systems}, 33, 2020.

\end{thebibliography}
